\documentclass[11pt,letterpaper]{article}

\usepackage[margin=1in]{geometry}
\usepackage{float}

\usepackage[T1]{fontenc}
\usepackage{lmodern}
\usepackage{microtype}

\usepackage{amsmath, amsthm, amssymb}
\usepackage{mathtools}
\usepackage{bbm}

\usepackage{graphicx}
\usepackage{booktabs}
\usepackage{tabularx}
\usepackage[table]{xcolor}
\usepackage{longtable}
\usepackage{subcaption}

\usepackage[numbers,sort&compress]{natbib}
\usepackage{aliascnt}
\usepackage[hidelinks]{hyperref}
\usepackage[nameinlink]{cleveref}
\newcommand{\authorcite}[1]{\citeauthor{#1}~\cite{#1}}

\usepackage[printonlyused]{acronym}
\acrodef{bmnd}[BMND]{n-dimensional block matching}
\acrodef{bm3d}[BM3D]{3-dimensional block matching}
\acrodef{bm4d}[BM4D]{4-dimensional block matching}
\acrodef{nlm}[NLM]{nonlocal means}
\acrodef{sscf}[SSCF]{similar-segment cooperative filtering}
\acrodef{psd}[PSD]{power spectral density}
\acrodef{ecg}[ECG]{electrocardiography}
\acrodef{dct}[DCT]{discrete cosine transform}
\acrodef{vst}[VST]{variance-stabilizing transform}
\acrodef{dwt}[DWT]{discrete wavelet transform}
\acrodef{bsd68}[BSD68]{Berkeley Segmentation Dataset}
\acrodef{set12}[Set12]{standard 12-image grayscale denoising test set}
\acrodef{kodak24}[Kodak24]{Kodak Lossless True Color Image Suite}
\acrodef{fmd}[FMD]{Fluorescence Microscopy Denoising}
\acrodef{xcat}[XCAT]{eXtended CArdiac-Torso}
\acrodef{ssd}[SSD]{sum of squared distances}
\acrodef{psnr}[PSNR]{peak signal-to-noise ratio}
\acrodef{ssim}[SSIM]{structural similarity index measure}
\acrodef{snr}[SNR]{signal-to-noise ratio}
\acrodef{rmse}[RMSE]{root-mean-square error}
\acrodef{cv}[CV]{coefficient of variation}

\usepackage[refpage,intoc]{nomencl}
\makenomenclature
\usepackage{etoolbox}

\theoremstyle{plain}

\newtheorem{theorem}{Theorem}[section]

\newaliascnt{lemma}{theorem}
\newtheorem{lemma}[lemma]{Lemma}
\aliascntresetthe{lemma}

\newaliascnt{proposition}{theorem}
\newtheorem{proposition}[proposition]{Proposition}
\aliascntresetthe{proposition}

\newaliascnt{corollary}{theorem}
\newtheorem{corollary}[corollary]{Corollary}
\aliascntresetthe{corollary}

\newaliascnt{conjecture}{theorem}

\aliascntresetthe{conjecture}

\theoremstyle{definition}

\newaliascnt{definition}{theorem}
\newtheorem{definition}[definition]{Definition}
\aliascntresetthe{definition}

\newaliascnt{example}{theorem}
\newtheorem{example}[example]{Example}
\aliascntresetthe{example}

\newaliascnt{exercise}{theorem}

\aliascntresetthe{exercise}

\newaliascnt{problem}{theorem}

\aliascntresetthe{problem}

\theoremstyle{remark}

\newaliascnt{remark}{theorem}
\newtheorem{remark}[remark]{Remark}
\aliascntresetthe{remark}

\newaliascnt{notation}{theorem}

\aliascntresetthe{notation}

\newaliascnt{claim}{theorem}

\aliascntresetthe{claim}

\crefname{theorem}{theorem}{theorems}
\Crefname{theorem}{Theorem}{Theorems}

\crefname{lemma}{lemma}{lemmas}
\Crefname{lemma}{Lemma}{Lemmas}

\crefname{proposition}{proposition}{propositions}
\Crefname{proposition}{Proposition}{Propositions}

\crefname{corollary}{corollary}{corollaries}
\Crefname{corollary}{Corollary}{Corollaries}

\crefname{conjecture}{conjecture}{conjectures}
\Crefname{conjecture}{Conjecture}{Conjectures}

\crefname{definition}{definition}{definitions}
\Crefname{definition}{Definition}{Definitions}

\crefname{example}{example}{examples}
\Crefname{example}{Example}{Examples}

\crefname{exercise}{exercise}{exercises}
\Crefname{exercise}{Exercise}{Exercises}

\crefname{problem}{problem}{problems}
\Crefname{problem}{Problem}{Problems}

\crefname{remark}{remark}{remarks}
\Crefname{remark}{Remark}{Remarks}

\crefname{notation}{notation}{notations}
\Crefname{notation}{Notation}{Notations}

\crefname{claim}{claim}{claims}
\Crefname{claim}{Claim}{Claims}

\usepackage{customcoms}

\title{BMND: Direct Poisson Denoising by N-Dimensional Block Matching and Collaborative Filtering}
\author{Christof Duhme, Lars Schiefelbein, Florian Büther, Xiaoyi Jiang}
\date{University of Münster}

\begin{document}

\maketitle

\begin{abstract}
Poisson denoising of scientific data requires methods that account for signal-dependent noise while accommodating different data dimensionalities and preserving quantitative intensity information.
We present BMND, a dimension-independent extension of block matching and collaborative filtering for Gaussian and Poisson observations.
Building on the two-stage structure of BM3D and BM4D, BMND processes Poisson data directly, without a variance-stabilizing transform, by combining noise-aware patch matching with propagation of signal-dependent noise variances through collaborative filtering and aggregation.
A dimension-independent reference-patch traversal scheme supports arrays with an arbitrary number of axes.
An optional aggregation-aware mass conservation preserves the observed total intensity after weighted overlap-add.
We evaluate the framework on one-dimensional physiological signals, two-dimensional images, and three-dimensional volumes, using controlled noise experiments and measured fluorescence microscopy acquisitions.
The experiments demonstrate improved reconstruction quality from noise-aware matching and Wiener filtering, while low-count phantom experiments show reduced denoising-induced intensity loss through mass conservation.
The framework provides a unified, non-learning-based approach to denoising across arbitrary data dimensions and is released as an open-source library.
\end{abstract}

\section{Introduction}
\label{sec:introduction}
Image and signal denoising is a fundamental ill-posed inverse problem that aims to recover a clean signal $x$ from a noisy observation $y$. 
The noise model is application-dependent: while additive white Gaussian noise with variance $\sigma^2$ is a common assumption, many imaging modalities - such as fluorescence microscopy and positron emission tomography - are dominated by Poisson noise, where the observed intensity follows a Poisson distribution with mean equal to the true signal~\cite{bertero2021introduction}.
Unlike additive white Gaussian noise, Poisson noise is signal-dependent, making it more challenging to remove~\cite{luisier2010image,makitalo2010optimal}.
Classical approaches to this problem include spatial filtering, such as Gaussian filtering and anisotropic diffusion~\cite{weickert1998anisotropic,perona1990scale}, as well as transform-domain methods~\cite{rudin1992nonlinear} like wavelet shrinkage~\cite{donoho1995denoising, coifman1995translation, chang2000adaptive}. 
Although these methods are computationally efficient, they operate locally or with a fixed basis and inherently struggle to preserve fine structural details, often introducing significant smoothing artifacts~\cite{milanfar2012tour, buades2005review}.
Furthermore, most classical methods are designed for Gaussian noise and require \acp{vst} (e.g., the Anscombe transform~\cite{anscombe1948transformation, makitalo2010optimal}) to handle Poisson data, which struggles in low-count regimes.

These classical methods operate exclusively locally or in a fixed transform basis and neglect the often pronounced self-similarity of natural signals.
This oversight then motivated the creation of nonlocal approaches that take advantage of the redundancy of similar structures across the entire data volume.
One of the forerunners of this technique was the \ac{nlm} algorithm~\cite{buades2005review}, which computes for each pixel a weighted average of all other pixels, with weight depending on the similarity of the surrounding image patches.
The \ac{nlm} algorithm showed that the exploitation of self-similarity can improve denoising quality substantially, thus laying the foundation for an entire family of patch-based denoising methods.
A multitude of refinements improved both robustness and efficiency of \ac{nlm}, like optimized blockwise processing~\cite{coupe2008optimized}, iterative weight estimation~\cite{deledalle2009iterative} and adaptive parameter selection~\cite{kervrann2006optimal}.
With these developments nonlocal self-similarity established itself as a core principle in modern image denoising, paving the way for more sophisticated patch-based approaches.

Building upon the success of nonlocal methods, \Ac{bm3d}~\cite{dabov2007image} represents a landmark advancement in image denoising.
\Ac{bm3d} combines block matching with collaborative filtering in a 3D transform domain: first 2D patches are grouped into 3D stacks, which then are jointly transformed using a separable 3D transform (e.g. wavelet or \ac{dct}), afterwards they are shrunk via hard thresholding or Wiener filtering, and finally inverse-transformed to produce the denoised image.
The synergy of nonlocal self-similarity and sparse transform-domain representation established \Ac{bm3d} as a powerful tool for removing white Gaussian noise in 2D images.
The capabilities of \Ac{bm3d} have been further improved by numerous extensions, including shaped adaptive variants~\cite{katkovnik2010local}, integration of PCA-based dictionaries~\cite{danielyan2011bm3d}, incorporation of color information~\cite{dabov2007color} and combinations with nonlocal sparse models~\cite{mairal2009non}.
The success of \Ac{bm3d} motivated its extension to higher-dimensional data.
\Ac{bm4d}~\cite{maggioni2012nonlocal} generalizes the collaborative filtering paradigm to volumetric data (3D) and video sequences by grouping 3D patches into 4D arrays, achieving positive results in these domains.
However, the transfer to arbitrary dimensions has not yet been done.
The block-matching step suffers from higher dimensionality, as both computational complexity as well as memory footprint grow exponentially with the number of dimensions~\cite{kolda2009tensor}.

In recent years, deep learning-based methods have achieved remarkable success in image denoising, often outperforming classical nonlocal approaches in terms of \ac{psnr} and \ac{ssim} on standard benchmarks.
Discriminative models such as DnCNN~\cite{zhang2017beyond} and FFDNet~\cite{zhang2018ffdnet} learn a direct mapping from noisy to clean images, while encoder-decoder architectures like U-Net~\cite{ronneberger2015u} have proven highly effective for various restoration tasks.
Self-supervised approaches, including Noise2Noise~\cite{lehtinen2018noise2noise}, Noise2Void~\cite{krull2019noise2void} and Noise2Self~\cite{batson2019noise2self}, have further pushed the boundaries by eliminating the need for clean ground-truth data.
However, these models are typically trained for a specific dimensionality and specific noise model (usually Gaussian).
Adapting them to n-dimensional data or to Poisson noise requires architectural modifications and retraining, which is often impractical for high-dimensional scientific data where large annotated training datasets are scarce~\cite{zhang2017learning}.
A further point of limitation is the black-box nature of deep networks, limiting interpretability, which is crucial in medical and scientific imaging~\cite{rudin2019stop}.
While hybrid approaches such as Plug-and-Play priors~\cite{venkatakrishnan2013plug, burger2012image} integrate classical denoisers into iterative schemes, they still rely on pre-trained models that are not inherently dimension-agnostic. 
In conclusion, despite the success, these methods do not provide a direct solution to the challenges of n-dimensional, Poisson-corrupted data, highlighting the need for a dedicated framework like \Ac{bmnd}.

Rather than simply adapting existing block-matching filters, our method is a nontrivial extension of \Ac{bm3d} and \Ac{bm4d}.
Our key contributions are fourfold:

\begin{enumerate}
    \item \textbf{Native Poisson handling.} Directly processes Poisson-corrupted data without requiring a \ac{vst}, preserving accuracy even at low photon counts.
    \item \textbf{Aggregation-aware mass conservation.} Conserves the total intensity within each patch group under Poisson noise.
    \item \textbf{Arbitrary-dimensional support.} Enables denoising of data with any number of dimensions without manual adaptation.
    \item \textbf{Unified reference-patch traversal.}
    Replaces dimension-specific lookup tables with a unified scheme for arbitrary-dimensional arrays, enabling seamless application to new data types.
\end{enumerate}

In addition, we release our implementation as an open-source library to ensure easy use and reproducibility. Together, these contributions go beyond incremental modifications and establish a self-contained denoising framework for a wide range of data types.

The \texttt{bmnd} package can be downloaded from PyPI (\url{https://pypi.org/project/bmnd/}).
The source code of the library can be found at \url{https://github.com/cduhme/bmnd}.

\section{Method Overview}
\label{sec:method-overview}
This section formalizes the observation model that underlies the proposed denoising framework.
We consider both additive Gaussian and Poisson-distributed noise, which require different treatment in the filtering process.
The subsequent description of the algorithm is structured along the two processing stages illustrated in \Cref{fig:method-overview}.
This provides a high-level overview of the proposed \Ac{bmnd} framework, which builds upon the collaborative filtering scheme of \Ac{bm3d}~\cite{dabov2007image} and \Ac{bm4d}~\cite{maggioni2012nonlocal}.
The algorithm consists of two stages over the entire input volume: a hard-thresholding stage followed by a Wiener filtering stage.
Both stages share the same core structure - block matching, collaborative transform, coefficient shrinkage, inverse transform, and aggregation - but differ in the shrinkage rule and the use of reference data.
The first stage produces a complete intermediate estimate, which is then used as a reference for the second stage to guide block matching and compute Wiener shrinkage coefficients.

The first stage, the hard-thresholding stage, operates directly on the noisy observation.
Candidate patches are identified using block matching (\Cref{sec:block-matching}), grouped into stacks, and transformed using an $N$-dimensional transform.
The transformed coefficients are then hard-thresholded (\Cref{sec:hard-thresholding}).
After applying the inverse transform and a Kaiser window, the processed patches are aggregated (\Cref{sec:aggregation-weights}) to form the hard-thresholding estimate.

The second stage, the Wiener filtering stage, refines this estimate.
It again performs block matching, but now uses the hard-thresholding estimate to find candidate patches, while the actual values used for filtering are taken from the original noisy observation (\Cref{sec:block-matching}).
The grouped patches are transformed with another, different $N$-dimensional transform, and Wiener shrinkage is applied (\Cref{sec:wiener}).
An optional mass-conservation constraint (\Cref{sec:mass-conservation}) is then imposed before the inverse transform and Kaiser window.
Finally, aggregation (\Cref{sec:aggregation-weights}) yields the Wiener shrinkage estimate, which is the output of the algorithm.

\begin{figure}[htbp]
    \centering
    \includegraphics[width=\textwidth]{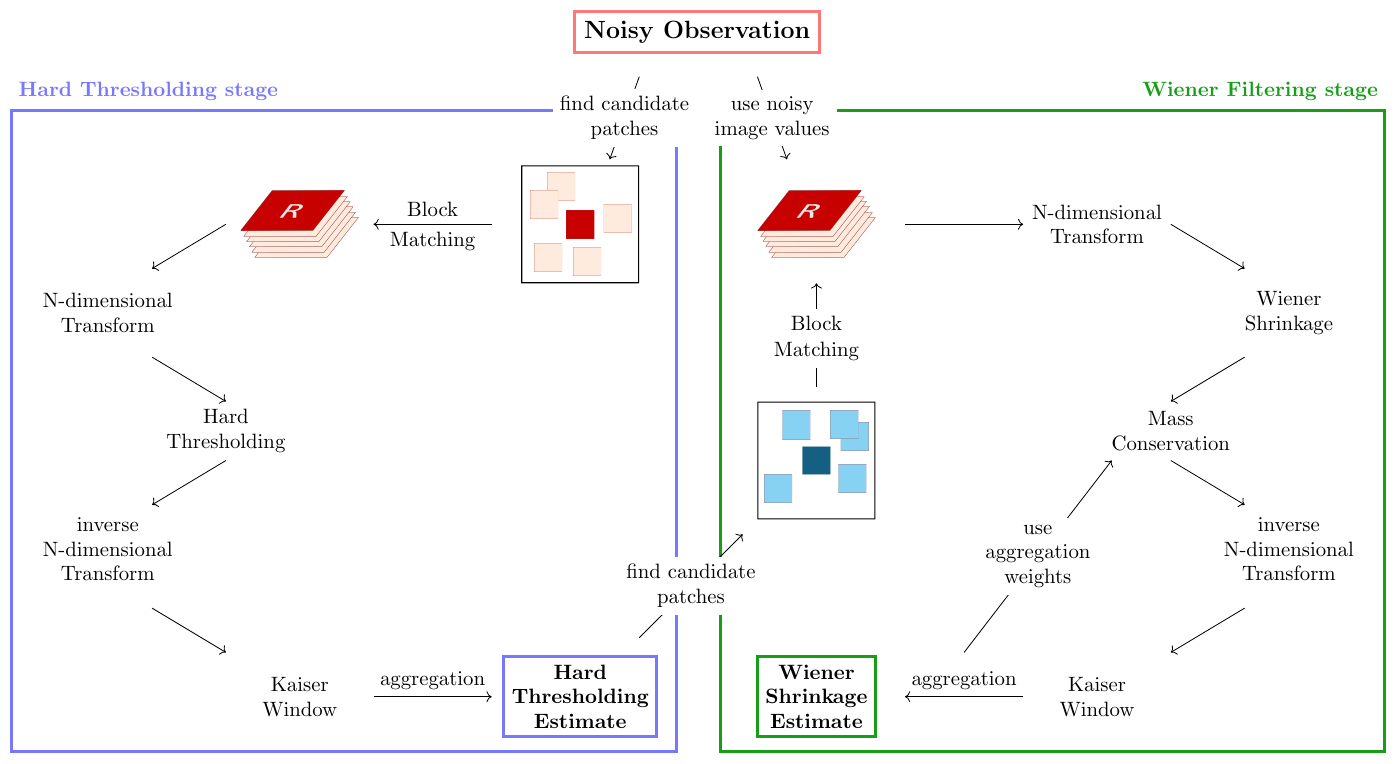}
    \caption{Overview of the algorithm.}
    \label{fig:method-overview}
\end{figure}

\begingroup
\footnotesize
\renewcommand{\arraystretch}{1.5}
\begin{longtable}{@{}p{0.57\textwidth} p{0.38\textwidth}@{}}
    \caption{Top-down overview of the problems addressed by \Ac{bmnd}, the limitations of existing approaches, and the corresponding methodological solutions.}
    \label{tab:bmnd-problem-overview} \\
        \toprule
        \textbf{Problem and existing limitations} &
        \textbf{\Ac{bmnd} approach} \\
        \midrule
        \endfirsthead

        \toprule
        \textbf{Problem and existing limitations} &
        \textbf{\Ac{bmnd} approach} \\
        \midrule
        \endhead

        \midrule
        \multicolumn{2}{r}{Continued on next page} \\
        \endfoot

        \bottomrule
        \endlastfoot

        \textbf{Different data modalities and noise models.}
        Parameters, dimensions, and noise assumptions are tightly coupled in existing implementations, which limits their reuse across imaging modalities and acquisition settings. &
        A common formulation with configurable profiles for geometry, noise, matching, shrinkage, and aggregation. \\
        \midrule

        \textbf{Signal-dependent Poisson noise.}
        Standard squared Euclidean block matching and fixed-variance shrinkage are designed for data-independent, stationary Gaussian noise and do not account for the signal dependence of Poisson noise. &
        Direct scaled-Poisson processing with a plug-in variance map and Poisson-aware matching metrics. \\
        \midrule

        \textbf{Correlated or heteroscedastic noise.}
        A single scalar variance per coefficient is inadequate for nonstationary or correlated noise. &
        A common covariance interface that propagates noise through transforms, thresholds, Wiener gains, and aggregation weights. \\
        \midrule

        \textbf{Arbitrary-dimensional data.}
        \Ac{bm3d} and \Ac{bm4d} are formulated for specific data dimensionalities and therefore require dimension-dependent implementations of patch extraction, grouping, transforms, and traversal. &
        Coupled reference-patch schedules and group transforms for arrays with an arbitrary number of data axes. \\
        \midrule

        \textbf{Low-count measurements.}
        High-count asymptotics for discrepancy moments become inaccurate for small pooled counts. &
        Exact finite-count conditional moments for calibrating acceptance rules and matching scales. \\
        \midrule

        \textbf{Noise-aware block matching.}
        Raw \ac{ssd} includes the expected contribution of noise and does not distinguish structural mismatch from noise fluctuations. &
        Poisson deviance, Pearson and Anscombe discrepancies, and noise-bias-corrected \ac{ssd} under Gaussian or Poisson models. \\
        \midrule

        \textbf{Hard thresholding with correlated Gaussian noise.}
        Standard coefficientwise thresholding relies on a scalar variance per coefficient and does not propagate noise covariance through the transform. &
        Exact transform-domain variance propagation and partial covariance propagation. \\
        \midrule
        
        \textbf{Hard thresholding with signal-dependent Poisson noise.}
        Thresholding rules based on fixed-variance assumptions do not account for the signal dependence of Poisson noise. &
        Direct-Poisson variance maps and shared-source covariance model for repeated voxels. \\
        \midrule
        
        \textbf{Scalar Wiener risk and pilot signal power.}
        The classical Wiener gain uses a scalar minimum-risk rule and raw pilot powers that include a noise floor, which is suboptimal for coefficient-dependent noise. &
        Coefficient-dependent gains, bias-corrected pilot powers, and covariance-aware risk. \\
        \midrule
        
        \textbf{Group-level aggregation weights.}
        Standard aggregation assigns the same weight to all patches in a group and does not exploit patch-level reliability. &
        Patch-level weighting based on propagated variance and predicted marginal risk. \\
        \midrule
        
        \textbf{Aggregation weights ignoring bias and overlap covariance.}
        Existing reliability weights do not account for bias in the pilot or covariance between overlapping estimates. &
        Bias-aware risk and overlap-aware covariance. \\
        \midrule
        
        \textbf{Loss of total intensity under shrinkage.}
        Aggregation weights alone do not ensure that the reconstructed array preserves the observed total intensity, which is proportional to the number of detected events under the scaled-Poisson model. &
        Optional aggregation-aware constraint enforcing each filtered group to reproduce the normalized overlap-add contribution of its noisy source group via a transform-domain correction. \\

\end{longtable}
\endgroup

In the following chapters we systematically walk through the problems and limitations of existing nonlocal denoising methods and present the corresponding \Ac{bmnd} solutions.
The order of presentation follows \Cref{tab:bmnd-problem-overview} and, at the same time, the order in which the algorithm encounters individual issues.
We therefore use the table as a top-down roadmap: each row is a problem-solution pair that is discussed in the corresponding chapter. 

Existing implementations such as \Ac{bm3d} and \Ac{bm4d} are designed for specific data dimensionalities and noise models.
As a result, their parameters, dimensions, and noise assumptions are tightly coupled, which limits their reuse across imaging modalities and acquisition settings.
\Ac{bmnd} addresses this limitation by introducing a common formulation with configurable profiles for geometry, noise, matching, shrinkage, and aggregation.
In \Cref{sec:method-overview}, we establish this general framework: we define the notation, the patch-group observation model, and the covariance propagation rules that are used throughout the manuscript.
The covariance serves as the common interface between the noise model and the algorithm.
All thresholds and shrinkage factors are expressed in terms of the transform-domain variances.

Because the standard \ac{ssd} block matching and fixed-variance shrinkage are specifically designed for data-independent, stationary Gaussian noise, the preexisting implementations are limited in what types of data they may handle.
This is additionally visible in the used variance map, which originated as a single scalar variance per coefficient. 
These limitations do not represent nonstationary or correlated noise and thus in particular not the signal dependence of Poisson noise.
In \Cref{sec:observation-models}, we introduce both the stationary Gaussian model as well as the scaled Poisson model and discuss the handling of the variance in both noise modalities.
This allows us to establish direct scaled-Poisson processing with a plug-in variance map and Poisson-aware matching metrics in the upcoming sections.

The algorithms \Ac{bm3d} and \Ac{bm4d} use a precomputed lookup table to select reference-patch shifts. 
These are therefore limited to the provided patch and step sizes. 
A generalized reference-patch schedule, to overcome these challenges, has thus been added in \Cref{subsec:reference-patch-schedule} to allow for block matching in arbitrary dimensions.
The resulting construction supports an arbitrary number of data axes while balancing spatial coverage, boundary validity and computational cost.
Shifted traversal improves the coverage of coarse lattice without incurring the full cost of refining every axis simultaneously. 

Following the definition of the reference locations, \Cref{subsec:metrics} deals with comparing candidate patches. 
The commonly used method is \ac{ssd}, which is particularly suitable for data-independent stationary Gaussian noise.
Under signal-dependent Poisson noise, however, the variability of the observation changes with the underlying signal and can therefore reflect noise fluctuations rather than structural differences.
We consequently consider four matching metrics: \ac{ssd}, Poisson deviance, Pearson discrepancy, and Anscombe \ac{ssd}.

The scale of the matching metric is particularly important in the low-count regime.
The usual $\chi_1^2-$based calibration can be inaccurate for small pooled counts because the conditional moments of the discrepancy depend on the total count.
\Cref{subsec:ex-fin-count-mom-cali} therefore derives exact finite-count moments under an equal-rate Poisson null and uses them to calibrate low-count acceptance rules.
These are moment-calibrated criteria rather than exact conditional tests.

\Ac{ssd} does not distinguish structural differences between patches from discrepancies caused by noise.
The expected noise contribution depends on the spatial covariance under Gaussian noise and on the underlying signal under Poisson noise.
In \Cref{subsec:ssd-noise-bias}, we therefore derive the noise contribution for both observation models and subtract it from the raw \ac{ssd}.
For Poisson noise, the unknown contribution is estimated using the observed patch values.

Following block matching, the grouped patches are transformed and their coefficients are thresholded.
Thresholding rules based on a common noise variance do not account for spatial correlations or the signal dependence of Poisson noise.
In \Cref{sec:hard-thresholding}, the noise covariance is therefore propagated through the group transform to determine the individual coefficient variances.
For Poisson noise, a plug-in variance map and a shared-source covariance model additionally account for signal-dependent variances and repeated voxels in overlapping patches.
Partial covariance propagation allows the accuracy of this model to be balanced against computational cost.
The resulting variances are used to scale the coefficientwise thresholds.

The subsequent Wiener stage uses the hard-thresholding estimate to determine the signal power of each coefficient.
This estimate still contains residual noise, so the raw squared coefficients can overestimate the underlying signal power.
In \Cref{sec:wiener}, we derive the scalar minimum-risk gain using coefficient-dependent noise variances and introduce noise-floor corrections to the pilot-power estimates.
The resulting gains account for the estimated signal and noise contributions when filtering the original noisy coefficients.

Following inverse transformation, the filtered patches are aggregated to form the reconstructed array.
Standard group-level aggregation assigns the same weight to all patches in a group and therefore does not account for differences in their reliability.
\Cref{sec:aggregation-weights} introduces patch-level weighting alongside group-level weighting, using propagated variance and predicted marginal risk to determine the individual patch weights.
This allows more reliable patches to contribute more strongly to the reconstructed estimate, even when they belong to the same group.

Aggregation weights based only on the transmitted noise do not account for the signal attenuation introduced by shrinkage.
Overlapping patches additionally share source voxels, which introduces dependence between the estimates within a group.
In \Cref{subsec:wiener-weighting-models}, the predicted Wiener risk therefore includes both the residual noise and the estimated signal loss.
The covariance-aware models in \Cref{subsec:windowed-weighting-models} further account for shared-source dependence through inverse transformation and aggregation windowing.
This allows the reliability weights to reflect both shrinkage-induced bias and overlap-induced covariance.

Finally, shrinkage can alter the total intensity of the reconstructed array, and aggregation weights alone do not ensure its preservation.
Under the scaled-Poisson model, the observed total intensity is proportional to the number of detected events.
In \Cref{sec:mass-conservation}, we therefore introduce an optional aggregation-aware constraint that makes each filtered group reproduce the normalized overlap-add contribution of its noisy source group.
The constraint is enforced through a transform-domain correction that accounts for the aggregation weights, window, and overlap normalization.
With the aggregation weights held fixed, this preserves the observed total intensity after aggregation.

To describe these steps within a common mathematical framework, we first introduce the notation, the patch-group observation model, and the covariance propagation rule used throughout the manuscript.
The following sections then develop the individual processing steps in the order outlined above.

\paragraph{Notation.} 
Spatial-domain group vectors are written in bold lowercase, such as $\mathbf y_g$, and their transform-domain counterparts in bold uppercase, such as $\mathbf Y_g = T_g \mathbf y_g$.
Individual entries are written without boldface, for example $y_{g,i}$ and $Y_{g,k}$.
Matrices and linear operators generally use non-bold uppercase letters, such as $T_g$ and $U_g$.
Calligraphic letters generally denote sets or collections, with locally defined exceptions for distributions, likelihoods, and selected operators.
Hats denote estimates.
The superscripts ${}^{\top}$ and ${}^{\mathrm H}$ denote transpose and conjugate transpose, respectively.\\

Let $D = N - 1$ denote the number of data axes.
Stacking $D$-dimensional patches along an additional grouping axis produces the $N$-dimensional groups processed by \ac{bmnd}.
Let $\Omega \subset \mathbb{Z}^D$ denote the finite sampling domain, and let $u \in \Omega$ denote a sampling location.
Let $y \colon \Omega \to \R$ denote the observed array and $x \colon \Omega \to \R$ its unknown clean mean.
For each reference origin, block matching searches a local window, scores the candidate patches, and retains the best matches.
The resulting group is the basic object processed by both collaborative-filtering stages.
 
\nomenclature[AAA]{\(\Omega\)}{discrete sampling domain}
\nomenclature[AAB]{\(u\)}{sampling location}
\nomenclature[AAC]{\(y\)}{observed array}
\nomenclature[AAD]{\(x\)}{unknown clean mean}
\nomenclature[AAE]{\(n\)}{zero-mean noise array}

\begin{definition}[Patch-group observation model]
    \label{def:patch-group-model}
    After vectorization, a noisy patch group is written as
    \begin{equation}
        \yy_g = \xx_g + \nn_g,
    \end{equation}
    where $\xx_g\in\R^d$ is the unknown clean group, $\nn_g\in\R^d$ is the corresponding noise vector with $\E[\nn_g]=0$ and $d$ is the group dimension.
    A possibly complex-valued combined patch and nonlocal group transform $T\in\C^{d\times d}$ produces
    \begin{equation}
        \begin{aligned}
            \mathbf{Y}_g &= T\mathbf{y}_g = \XX_g + \NN_g,\\
            \XX_g &= T\xx_g,
            &\qquad
            \NN_g &= T\nn_g.
        \end{aligned}
    \end{equation}
    Spatial-domain entries are denoted by $y_{g,i}$, $x_{g,i}$, and $n_{g,i}$, whereas transformed entries are denoted by $Y_{g,k}$, $X_{g,k}$, and $N_{g,k}$.
    When one group is fixed, the group index is suppressed, giving symbols such as $n_i$, $Y_k$, $X_k$, and $N_k$.

    For complex-valued random vectors, the covariance of $\nn_g$ is defined as 
    \begin{equation}
        \Sigma_g=\Cov(\nn_g)=\E\!\left[\nn_g\nn_g^{\mathrm{H}}\right],
    \end{equation}
    where the zero-mean assumption is used. The transformed covariance is denoted by 
    \begin{equation}
        \Sigma_{N,g}= \Cov(\NN_g).
    \end{equation}
\end{definition}

\nomenclature[AAF]{\(g\)}{patch-group index}
\nomenclature[AAG]{\(\yy_g\)}{observation group}
\nomenclature[AAH]{\(\xx_g\)}{unknown clean group}
\nomenclature[AAI]{\(\nn_g\)}{zero-mean noise group}
\nomenclature[AAJ]{\(d\)}{number of entries in a patch group}
\nomenclature[AAK]{\(T\)}{group transform operator}
\nomenclature[AAL]{\(\YY_g\)}{transformed observation group}
\nomenclature[AAM]{\(\XX_g\)}{transformed unknown clean group}
\nomenclature[AAN]{\(\NN_g\)}{transformed noise group}
\nomenclature[AAO]{\(Y_{g,k}\)}{observation coefficient $k$ of group $g$}
\nomenclature[AAP]{\(X_{g,k}\)}{clean coefficient $k$ of group $g$}
\nomenclature[AAQ]{\(N_{g,k}\)}{noise coefficient $k$ of group $g$}
\nomenclature[AAR]{\(\Sigma_g\)}{covariance of noise group $g$}
\nomenclature[AAS]{\(\Sigma_{N,g}\)}{covariance of transformed noise group $g$}

\begin{remark}[Complex-valued transforms]
    \label{rem:complex-valued-transforms}
    Real-valued source groups may have complex transform coefficients, e.g. when $T$ contains a Fourier transform.
    When $T$ is complex-valued, it and its inverse $U$ are understood as complex-linear operators on the complex spaces.
    A real group $\xx_g\in\mathbb{R}^d$ is embedded in $\mathbb{C}^d$, its coefficients lie in the real-valued subspace $\mathcal{S}_g = T(\mathbb{R}^d) = \{T \xx: \xx \in \R^d\}$, and the restriction of $U$ to $\mathcal{S}_g$ maps back to $\mathbb{R}^d$.
    For coefficients outside $\mathcal{S}_g$, the synthesized group need not be real-valued.
\end{remark}

\nomenclature[AAT]{\(U\)}{inverse group transform operator}
\nomenclature[AAU]{\(\mathcal{S}_g\)}{coefficient subspace corresponding to real-valued groups}

For real- or complex-valued random variables, covariance is understood in the Hermitian sense,
\begin{equation}
    \Cov(z,w)
    = \E \left[(z - \E[z]) \overline{(w - \E[w])}\right],
    \qquad
    \Cov(\mathbf{z})
    = \E \left[(\mathbf{z} - \E[\mathbf{z}])(\mathbf{z} - \E[\mathbf{z}])^{\H} \right].
    \label{eq:hermitian-covariance}
\end{equation}
The corresponding variance is $\Var(z)=\Cov(z,z)=\E[|z-\E[z]|^2]$.
For real random variables, these expressions reduce to the usual covariance and variance definitions.

The algorithm applies the following sequence twice.
In the first pass, transformed coefficients that are small relative to their own noise standard deviations are removed.
Inverse transformation and weighted overlap-add produce a pilot estimate.
The second pass repeats block matching using that pilot, transforms the pilot and noisy groups together, estimates signal power from the pilot, and applies a Wiener gain to the noisy coefficients.
A second overlap-add produces the output.
Thus, the covariance is the common interface between the noise model and the algorithm.
In what follows, all thresholds and shrinkage factors are defined in terms of the diagonal entries of $\Sigma_{N,g}$.

\begin{lemma}[Covariance propagation]
    \label{lem:covariance-propagation}
    Let $\nn_g$ have finite second moments and covariance $\Sigma_g$. For a deterministic linear transform $T$, the transformed noise covariance is
    \begin{equation}
        \Sigma_{N,g} = \Cov(\NN_g) = T \Sigma_g T^\H.
    \end{equation}
    In particular, the noise variance of the transformed coefficient $k$ is
    \begin{equation}
        \sigma_k^2 = (\Sigma_{N,g})_{kk}
        = \sum_{i,j} T_{ki} \overline{T_{kj}} \Cov(n_{g,i}, n_{g,j}). 
        \label{eq:transformed-coefficient-variance}
    \end{equation}
\end{lemma}

\nomenclature[AAV]{\(T^\H\)}{conjugate transpose of $T$}
\nomenclature[AAW]{\(\sigma_k^2\)}{noise variance of transformed coefficient $k$}

\begin{proof}
    Since $T$ is deterministic and linear, 
    \begin{equation*}
        \NN_g - \E[\NN_g]
        = T\left(\nn_g-\E[\nn_g]\right).
    \end{equation*}
    Consequently,
    \begin{equation*}
        \begin{split}
            \Cov(\NN_g)
            &= \E\left[T\left(\nn_g - \E[\nn_g]\right) \left(\nn_g - \E[\nn_g]\right)^\H T^\H \right]\\
            & = T \E\left[(\nn_g-\E[\nn_g])
                        (\nn_g-\E[\nn_g])^\H\right] T^\H\\
            & = T\Sigma_gT^\H.
        \end{split}
    \end{equation*}
    \Cref{eq:transformed-coefficient-variance} follows by taking the $k$th diagonal entry.
\end{proof}

\begin{remark}[Effect of data-dependent group selection]
    \label{rem:data-dependent-selection}
    \Cref{lem:covariance-propagation} is stated for both a fixed group and a fixed deterministic transform.
    In the algorithm, however, the group is selected by a data-dependent block-matching procedure.
    A fully conditional analysis would therefore require the distribution of the noise given the matching outcome.
    This issue applies to both Gaussian and Poisson observations.
    In the latter case, the noise covariance is additionally signal-dependent.
    Unless stated otherwise, the selected group is treated as fixed when computing the local covariance used by the filtering operations.     
\end{remark}
\section{Observation Models}
\label{sec:observation-models}
Observation models describe how the measured data relate to the underlying clean signal.
Two observation models are considered here: additive Gaussian noise and scaled-Poisson noise.
The additive white-Gaussian model is inherited from \Ac{bm3d}/\Ac{bm4d}~\cite{dabov2007image, maggioni2012nonlocal}, while stationary correlated-Gaussian modeling follows generalized collaborative filtering~\cite{makinen2020collaborative}.
For such stationary Gaussian fields, the autocovariance and power spectral density are linked via the discrete Wiener--Khinchin theorem~\cite{oppenheim2010discrete}.
On the other hand, Poisson noise arises in photon-limited imaging, where the observed counts follow a Poisson distribution.
Stored intensities may differ from these counts by a fixed multiplicative conversion factor.
This motivates the scaled-Poisson model, meaning that mean and variance formulas follow from standard Poisson moment identities~\cite{kingman1993poisson}.
Direct scaled-Poisson processing and its plug-in variance map extend these noise models in \Ac{bmnd}.

Despite their different origins, Gaussian and Poisson noise models share a common structure: both define a covariance matrix $\Sigma_g$ that characterizes the noise in the observation domain.
Once this covariance is known, the subsequent linear propagation through the filtering pipeline is identical.
In the following, these two noise models are formalized, starting with the Gaussian case, and the properties used throughout this work are derived.

\subsection{Stationary Gaussian Noise}
\begin{definition}[Additive Gaussian observation model]
    \label{def:additive-gaussian-model}
    An observation follows the additive Gaussian model if
    \begin{equation}
        y = x + n. \qquad n\sim\mathcal{N}(0, \Sigma_n),
    \end{equation}
    where $\Sigma_n$ is the covariance matrix of the vectorized noise array.
\end{definition}

\begin{definition}[Wide-sense stationary Gaussian noise]
    \label{def:stationary-gaussian}
    The Gaussian noise field is wide-sense stationary if its mean is constant and its covariance depends only on the displacement $\tau$ between sampling locations $u$ and $u + \tau \in \Omega$:
    \begin{equation}
        \Cov (n(u),n(u + \tau)) = \mathfrak{c}(\tau).
    \end{equation}
    The function $\mathfrak{c}$ is called the autocovariance.
\end{definition}

\nomenclature[BAA]{\(\mathfrak{c}(\tau)\)}{noise autocovariance at displacement $\tau$}

Under discrete Fourier normalization, the \ac{psd} samples $S(\omega)$ and the autocovariance satisfy
\begin{equation}
    \mathfrak{c}(\tau) = \frac{1}{|\Omega|}\operatorname{ifftn}(S)(\tau),
\end{equation}
where $|\Omega|$ is the number of samples in the array.

For white noise with \ac{psd} $S(\omega) = |\Omega|\sigma^2$ this simplifies.
\begin{corollary}[White-noise transform variance]
    \label{cor:white-noise-variance}
    If $\Sigma_g = \sigma^2 I$ and $T$ is orthonormal, then 
    \begin{equation}
        \Sigma_{N,g}=\sigma^2 I,
    \end{equation}
    and every transformed coefficient has noise variance $\sigma^2$.
\end{corollary}

\begin{proof}
    \Cref{lem:covariance-propagation} gives
    \begin{equation*}
        \Sigma_{N,g} = T (\sigma^2 I) T^\H = \sigma^2 T T^\H = \sigma^2 I.
    \end{equation*}
\end{proof}
Colored noise instead produces non-zero off-diagonal patch covariance and generally coefficient-dependent variances.

\subsection{Scaled-Poisson Noise}
For Poisson noise, there is not only the normal Poisson noise but also its scaled variant.
Choosing the Poisson scale $s_P = 1$, all subsequent formulas reduce to the standard Poisson case.
\begin{definition}[Scaled-Poisson observation model]
    \label{def:scaled-poisson-model}
    Let $K(u)$ denote the number of detected events at location $u$.
    Under the scaled-Poisson model,
    \begin{equation}
        K(u)\sim\Poi(\lambda(u)), \qquad Y(u)=s_P K(u),~ s_P >0.
    \end{equation}
    The corresponding clean mean in stored-intensity units is
    \begin{equation}
        X(u)= \E[Y(u)] 
        = s_P \lambda(u). 
    \end{equation}
\end{definition}

\nomenclature[BAB]{\(s_P\)}{stored intensity per detected event}
\nomenclature[BAC]{\(K(u)\)}{number of detected events at sampling location $u$}

The standard Poisson model additionally assumes that the random variables $K(u)$ are independent across sampling locations.

\begin{lemma}[Scaled-Poisson moments]
    \label{lem:scaled-poisson-moments}
    Under \Cref{def:scaled-poisson-model},
    \begin{equation}
        \E[Y(u)] = X(u),
        \qquad
        \Var[Y(u)] = s_P X(u).
    \end{equation}
    If the raw counts are independent, then the spatial-domain noise covariance is diagonal, i.e. 
    \begin{equation}
        \Cov (Y(u),Y(v)) 
        = \begin{cases}
            s_PX(u), & u=v \\
            0, & u\neq v.
        \end{cases}
        \label{eq:poisson-covariance}
    \end{equation}
\end{lemma}

\begin{proof}
    Since a Poisson random variable with parameter $\lambda(u)$ has mean and variance $\lambda(u)$.
    \begin{equation*}
        \begin{split}
            \E[Y(u)]
            & = s_P\E[K(u)]
            = s_P\lambda(u)
            = X(u), \\
            \Var(Y(u)) 
            & = s_P^2 \Var(K(u))
            = s_P^2\lambda(u)
            = s_PX(u).
        \end{split}
    \end{equation*}
    Independence of the raw counts yields \Cref{eq:poisson-covariance}.
\end{proof}
The factor $s_P$ is a change of units, not an additional free noise parameter.
An intensity $X$ represents $X/s_P$ expected events.
After each event is scaled by $s_P$ , the variance in stored-intensity units is $s_P X$.

As a Poisson process does not have a spectral component, a variance map $m_\mathrm{v}$ is defined instead of the \Ac{psd}.
\begin{definition}[Poisson variance map]
    \label{def:poisson-variance-map}
    Let $\widehat X$ be a nonnegative surrogate of the clean stored-intensity mean.
    Define the variance map
    \begin{equation}
        m_\mathrm{v}(u) = s_P [\widehat X(u)]_+.
    \end{equation}
\end{definition}

\nomenclature[BAD]{\(m_{\mathrm{v}}(u)\)}{estimated Poisson noise variance at sampling location $u$}

The variance identity in \Cref{lem:scaled-poisson-moments} motivates this map.
The positive part respects the support of Poisson intensities.

Let $c > 0$ be a conversion factor between stored-intensity units.
If the observations, clean means, and Poisson scale are rescaled consistently as
\begin{equation}
    Y'(u) = c Y(u),
    \qquad
    X'(u) = c X(u),
    \qquad
    s'_P = c s_P,
\end{equation}
then
\begin{equation}
    \frac{Y'(u)}{s'_P}
    = \frac{Y(u)}{s_P}.
\end{equation}
Thus, any matching statistic based only on equivalent raw counts is invariant to the chosen storage units.
\section{Block Matching}
\label{sec:block-matching}
Block matching identifies patches that are likely to contain similar underlying signal content and groups them for subsequent collaborative filtering.
For each reference patch, the algorithm selects a valid reference origin, searches a local neighborhood, evaluates a matching metric against candidate patches, and retains the best matches.
In \Ac{bmnd}, the reference origins are generated along an arbitrary number of data axes, so the traversal must balance spatial coverage, boundary validity, and computational cost.
The matching metric is selected according to the observation model and may require finite-count calibration or correction for the expected contribution of noise.
This section first defines the reference patch schedules and then develops the matching metrics and their noise-aware extensions.

Local window search, \ac{ssd} comparison, and retention of the best matching patches are inherited from \Ac{bm3d}/\Ac{bm4d}~\cite{dabov2007image,maggioni2012nonlocal}.
The Gaussian SSD noise bias correction follows generalized collaborative filtering~\cite{makinen2020collaborative}.
Poisson deviance, Pearson discrepancy, the Anscombe transform, and conditional Poisson splitting are standard statistical results~\cite{mccullagh1989generalized,pearson1900criterion,
anscombe1948transformation,kingman1993poisson}.
Coupled arbitrary-dimensional schedules, exact finite-count calibration, and noise bias-corrected matching extend that construction in \Ac{bmnd}.

\subsection{Reference-Patch Schedule}
\label{subsec:reference-patch-schedule}
Using zero-based indexing, a patch origin is the coordinate of its first sample.
Along an axis of length $L_d$, a patch of size $j_d \leq L_d$ has valid origins $0, \ldots, L_d - j_d$.
Therefore, the last valid patch origin is $\ell_d = L_d - j_d$.
Evaluating every valid patch origin gives the densest traversal, but its cost grows approximately as $\prod_d (\ell_d + 1)$.
A coarse traversal with step sizes $h_d$ requires $\prod_d \left(\left\lceil\frac{\ell_d}{h_d}\right\rceil + 1\right)
$ evaluations.
Reducing the step sizes $h_d$ therefore improves coverage at a rapidly increasing cost, particularly when several axes are refined simultaneously.
A coarse unshifted lattice is cheaper, but it always samples the same positions and may underrepresent structures lying between them.

Shifted traversal provides an intermediate solution.
It retains the coarse step sizes but displaces the lattice between a prescribed number of schedule passes.
This samples different patch alignments while increasing the number of reference origins by the number of passes rather than by the Cartesian product of finer axiswise lattices.

The \Ac{bm3d}/\Ac{bm4d} reference implementations use a precomputed lookup table to select such shifts for supported patch and step sizes.
Unsupported patch and step sizes, as well as dimensions outside the supported two- and three-dimensional cases, fall back to zero shifts.
Because this lookup table does not define a general arbitrary-dimensional traversal, \Ac{bmnd} replaces it with schedules whose shifts are generated directly from the patch sizes, step sizes, and density parameters.

\begin{definition}[Clipped reference origin]
    \label{def:clipped-reference-origin}
    Let $D$ be the number of data axes, indexed by $d \in \{0, \ldots, D - 1\}$.
    Along axis $d$, let $j_d$ be the patch size, $h_d$ the step size, and $\ell_d$ the last valid patch origin.
    The coarse lattice contains
    \begin{equation}
        Z_d = \left\lceil \frac{\ell_d}{h_d} \right\rceil + 1
    \end{equation}
    slots, indexed by $q_d \in \{0, \ldots, Z_d - 1\}$.
    The nominal origin associated with slot $q_d$ is
    \begin{equation}
        s_d^{\mathrm{slot}} = \min\{q_d h_d, \ell_d\}.
    \end{equation}
    For a nonnegative offset $o_d$, the corresponding reference origin is
    \begin{equation}
        r_d = \min\{\max\{s_d^{\mathrm{slot}} - o_d, 0\}, \ell_d\}.
        \label{eq:clipped-origin}
    \end{equation}
    For the last slot $q_d = Z_d - 1$, one sets $r_d = \ell_d$ regardless of $o_d$.
    Thus every reference patch is valid, and the last slot represents the far boundary independently of the selected offset.
\end{definition}

\nomenclature[CAA]{\(D\)}{number of data axes}
\nomenclature[CAB]{\(j_d\)}{patch size along axis $d$}
\nomenclature[CAC]{\(h_d\)}{reference-lattice step size along axis $d$}
\nomenclature[CAD]{\(\ell_d\)}{last valid patch origin along axis $d$}
\nomenclature[CAE]{\(Z_d\)}{number of coarse lattice slots along axis $d$}
\nomenclature[CAF]{\(s_d^\text{slot}\)}{slot origin along axis $d$}
\nomenclature[CAG]{\(o_d\)}{slot offset along axis $d$}
\nomenclature[CAH]{\(q_d\)}{slot index along axis $d$}
\nomenclature[CAI]{\(r_d\)}{reference patch origin along axis $d$}

Clipping determines whether an origin is valid but not how offsets are selected.
\Ac{bmnd} constructs a short candidate list for each axis.
Let
\begin{equation}
    n_d = \min\left\{j_d,
        \max\left(1, \left\lceil \rho_s \frac{j_d - 1}{h_d} \right\rceil\right)
    \right\}.
    \label{eq:number-shifts}
\end{equation}
Define the ordered candidate shift list by
\begin{equation}
    \mathcal{S}_d = (s_{d, 0}, \ldots, s_{d, n_d - 1}),
    \qquad
    s_{d, j} =
    \begin{cases}
        0, & n_d = 1,\\
        \displaystyle
        \left\lfloor
            \frac{j (j_d - 1)}{n_d - 1} + \frac{1}{2}
        \right\rfloor,
        & n_d > 1,
    \end{cases}
    \label{eq:candidate-shift-list}
\end{equation}
for $j = 0, \ldots, n_d - 1$.
Thus $\mathcal{S}_d$ consists of rounded uniformly spaced samples of $[0, j_d - 1]$.
The shift-density parameter $\rho_s$ controls the resolution of these candidate lists.

\nomenclature[CAJ]{\(\mathcal{S}_d\)}{ordered candidate shift list along axis $d$}
\nomenclature[CAK]{\(\rho_s\)}{candidate-list resolution}

Let $K_s \geq 1$ denote the number of schedule passes.
For each slot tuple $\mathbf{q} = (q_0, \ldots, q_{D - 1})$, \Ac{bmnd} visits that tuple once in every pass $k \in \{0, \ldots, K_s - 1\}$.
Each pass selects one offset vector $\mathbf{o}$ and obtains $\mathbf{r}$ from \Cref{eq:clipped-origin} coordinatewise.
Consequently, the $K_s$ passes construct $K_s$ candidate origins for each slot tuple, of which at most $K_s$ are distinct.
Choosing the offset on every axis independently would instead construct $\prod_d |\mathcal{S}_d|$ candidates.
Repeated origins created by clipping or by different passes are retained only once.

\nomenclature[CAL]{\(K_s\)}{number of schedule passes}
\nomenclature[CAM]{\(\mathbf{q}\)}{slot-index tuple}
\nomenclature[CAN]{\(\mathbf{o}\)}{slot-offset tuple}
\nomenclature[CAO]{\(\mathbf{r}\)}{reference patch origin tuple}

\begin{definition}[Generated reference schedule]
    \label{def:generated-schedule}
    The generated schedule uses the next-axis slot index to select the offset of the current axis:
    \begin{equation}
        o_d =
        \begin{cases}
            \mathcal{S}_d[(q_{d + 1} + k) \bmod |\mathcal{S}_d|], & d < D - 1,\\
            0, & d = D - 1.
        \end{cases}
        \label{eq:generated-schedule}
    \end{equation}
    The last axis remains unshifted and anchors the resulting asymmetric hierarchy.
\end{definition}

\begin{definition}[Balanced reference schedule]
    \label{def:balanced-reference-schedule}
    Let $Q = \sum_e q_e$.
    The balanced schedule sets
    \begin{equation}
        o_d = \mathcal{S}_d[(k + Q - q_d) \bmod |\mathcal{S}_d|].
        \label{eq:balanced-schedule}
    \end{equation}
    Because $Q - q_d$ is the sum of the slot indices on all other axes, no axis is given a distinguished role.
\end{definition}

\begin{definition}[Sparse reference schedule]
    \label{def:sparse-reference-schedule}
    Define the active axes by
    \begin{equation}
        \mathcal{A}
        = \{d : Z_d > 1 \text{ or } j_d > 1\}.
    \end{equation}
    For each $d \in \mathcal{A}$ that has a next active axis $e = \min\{a \in \mathcal{A} : a > d\}$, the sparse schedule sets
    \begin{equation}
        o_d = \mathcal{S}_d[(q_e + k) \bmod |\mathcal{S}_d|].
    \end{equation}
    All remaining axes use $o_d = 0$.
    Hence singleton axes do not interrupt the hierarchy, while its last active axis remains an unshifted anchor.
\end{definition}

\nomenclature[CAP]{\(\mathcal{A}\)}{set of active axes}

\begin{definition}[Off reference schedule]
    \label{def:off-reference-schedule}
    The off schedule sets $o_d = 0$ and traverses
    \begin{equation}
        0, h_d, 2 h_d, \ldots, \ell_d
    \end{equation}
    along axis $d$, appending $\ell_d$ when the stepped sequence would otherwise omit it.
\end{definition}

The four schedules differ only in how they couple the axiswise offsets.
The off schedule is the baseline coarse lattice.
The generated schedule is an asymmetric hierarchy over consecutive axes, and the sparse schedule removes inactive axes from that hierarchy.
The balanced schedule instead couples every axis symmetrically.
In all three shifted schedules, increasing the pass index $k$ advances each selected list index cyclically.

\subsection{Metrics}
\label{subsec:metrics}
The original \Ac{bm3d}/\Ac{bm4d} algorithms use the \acf{ssd} as their block-matching metric.
For white Gaussian noise, which is data-independent and stationary, \ac{ssd} is a natural choice.
As the \Ac{bmnd} algorithm further wants to deal with Poisson as an additional noise modality one must verify if \ac{ssd} remains the correct choice or if perhaps another metric should be used.
Poisson noise naturally is data dependent and could thus distort the correctness of correct matches using \ac{ssd}.
Therefore, a total of four metrics are currently usable for denoising using \Ac{bmnd}.
These are \ac{ssd}, Poisson Deviance, Pearson, and Anscombe SSD. \\

Throughout this section, let $P$ denote the number of entries per patch, and let $p = (p_1, \dots, p_P)$ and $q = (q_1, \dots, q_P)$ denote the reference and candidate patches, respectively.
For Poisson metrics, their equivalent raw counts are defined later in \Cref{eq:raw-counts}.

\nomenclature[CAQ]{\(p\)}{reference patch}
\nomenclature[CAR]{\(q\)}{candidate patch}
\nomenclature[CAS]{\(P\)}{number of entries per patch}

\begin{definition}[Gaussian SSD matching statistic]
    \label{def:gaussian-ssd-matching}
    For patches $p$ and $q$ containing $P$ entries, the Gaussian matching statistic is the sum of squared differences
    \begin{equation}
        d_{\SSD}(p, q)
        = \sum_{i=1}^{P} (q_i - p_i)^2.
        \label{eq:ssd}
    \end{equation}
    This score is measured in square stored-intensity units.
\end{definition}

The \ac{ssd} is directly tied to the Gaussian noise model. The following proposition formalizes these ties. 

\begin{proposition}[SSD as Gaussian negative log-likelihood]
    \label{prop:ssd-gaussian-nll}
    Assume that the entries of two patches $p$ and $q$ are generated by
    \begin{equation}
        p_i = \mu_i + \varepsilon_i^{(p)},
        \qquad
        q_i = \mu_i + \varepsilon_i^{(q)},
    \end{equation}
    with i.i.d. Gaussian noise $\varepsilon_i^{(p)},\, \varepsilon_i^{(q)}\sim\mathcal{N}(0, \sigma^2)$.
    Then minimizing the \ac{ssd} between $p$ and $q$ is equivalent to maximizing the likelihood that both patches share the same underlying mean $\mu$.
\end{proposition}

\begin{proof}
    The log-likelihood of the observations $(d_i)_{i=1}^{P}$ is
    \begin{equation*}
        \log \P (d_1,\dots,d_P \mid \mu) 
        = \sum_{i=1}^{P} \log f(d_i\mid\mu),
    \end{equation*}
    where $f$ is the Gaussian density.
    Under the model above, the difference $d_i= q_i - p_i$ satisfies
    \begin{equation*}
        d_i\sim \mathcal{N}(0, 2\sigma^2).
    \end{equation*}
    This is given by distribution of $\varepsilon_i^{(p)}$ and $\varepsilon_i^{(q)}$.
    Up to additive constants, the log-likelihood is
    \begin{equation*}
        \begin{split}
            \log \P(d_1,\dots,d_P \mid\mu)
            & = \sum_{i=1}^{P}\log\left(\frac{1}{\sqrt{4\pi\sigma^2}}\exp\!\left( -\frac{d_i^2}{4\sigma^2} \right) \right) \\
            & = -\frac{P}{2} \log (4\pi\sigma^2) - \frac{1}{4\sigma^2}\sum_{i=1}^{P}(q_i - p_i)^2.
        \end{split}
    \end{equation*}
    For fixed $\sigma^2$, the only data-dependent term is proportional to $-\sum_i(q_i-p_i)^2$. Hence, maximizing this log-likelihood is equivalent to minimizing
    \begin{equation*}
        \sum_{i=1}^{P}(q_i-p_i)^2
        = d_{\ssd}(q,p).
    \end{equation*}
\end{proof}

For Poisson noise, this does not hold statistically. This is because the variance depends on the mean. Therefore, different metrics which derive from the Poisson likelihood have to be considered.

For direct Poisson observation with scale $s_P>0$, one defines the equivalent raw counts associated with a patch entry $p_i$ as
\begin{equation}
    r_i 
    = \left[\frac{p_i}{s_P}\right]_+,
    \qquad c_i
    = \left[\frac{q_i}{s_P}\right]_+.
    \label{eq:raw-counts}
\end{equation}
The positive part enforces non-negativity, and $r_i,c_i$ are interpreted as expected event counts in stored-intensity units.

\nomenclature[CAT]{\(r_i\)}{raw count at entry $i$ of the reference patch}
\nomenclature[CAU]{\(c_i\)}{raw count at entry $i$ of the candidate patch}

\begin{definition}[Poisson deviance matching statistic]
    \label{def:poisson-deviance-matching}
    Let $r_i$ and $c_i$ be the equivalent raw counts from \Cref{eq:raw-counts}. The symmetric Poisson deviance between patches $p$ and $q$ is defined as
    \begin{equation}
        d_{\dev}(p,q)
        = \frac{2}{P}\sum_{i=1}^{P}\left(r_i \log \frac{2 r_i}{r_i+c_i} + c_i \log \frac{2 c_i}{r_i + c_i}\right)
        \label{eq:deviance}
    \end{equation}
\end{definition}

\begin{remark}
    Here the following convention is adopted: $0\log 0 = 0$, and any term with $r_i + c_i = 0$ contributes zero to the sum.
\end{remark}

\begin{definition}[Poisson Pearson matching statistic]
    \label{def:poisson-pearson-matching}
    With $r_i$ and $c_i$ as in \Cref{eq:raw-counts}, the pooled Pearson statistic between patches $p$ and $q$ is 
    \begin{equation}
        d_{\Pearson}(p,q)
        = \frac{1}{P}\sum_{i=1}^{P} \frac{(c_i - r_i)^2}{r_i + c_i}.
        \label{eq:pearson}
    \end{equation}
    Terms with $r_i + c_i = 0$ are defined to contribute zero.
\end{definition}

\begin{definition}[Anscombe-SSD matching statistic]
    \label{def:anscombe-ssd-matching}
    Let $A(n) = 2\sqrt{n + 3/8}$ be the Anscombe transform~\cite{anscombe1948transformation}. 
    For equivalent raw counts $r_i, c_i$ from \Cref{eq:raw-counts}, define
    \begin{equation}
        d_{\mathrm{A}}(p, q)
        = \frac{1}{2 P}\sum_{i=1}^{P} \bigl(A(c_i) - A(r_i)\bigr)^2.
        \label{eq:anscombe-distance}
    \end{equation}
    The factor $1/2$ accounts for the approximately unit variance of each transformed observation and hence the approximately twofold variance of their difference.
\end{definition}

\nomenclature[CAV]{\(A(n)\)}{variance-stabilizing Anscombe transform of count $n$}

Both $d_{\dev}$ and $d_{\Pearson}$ arise from the Poisson likelihood model. The deviance corresponds to a likelihood-ratio statistic comparing separate rates to a pooled equal-rate model~\cite{mccullagh1989generalized}, whereas the Pearson statistic is the classical $\chi^2$-type goodness-of-fit measure~\cite{pearson1900criterion}.
The next Proposition shows that the deviance is exactly the likelihood-ratio statistic for comparing two rates.

\begin{proposition}[Poisson deviance as symmetric likelihood-ratio distance]
    \label{prop:deviance-likelihood-ratio}
    Let $r_i$ and $c_i$ be as in \Cref{eq:raw-counts}.
    Define the entrywise likelihood-ratio statistic
    \begin{equation}
        \Lambda_i
        = -2\log\frac{\mathcal{L}_0(\widehat{\lambda}_{0,i})}{\mathcal{L}_1(\widehat{\lambda}_{r,i},\widehat{\lambda}_{c,i})},
        \label{eq:lambda-i-def}
    \end{equation}
    where $\mathcal{L}_0$ and $\mathcal{L}_1$ are the likelihoods under the null hypothesis of a common rate and under the alternative of separate rates, respectively. Then 
    \begin{equation}
        \Lambda_i 
        = 2\left(
            r_i\log\frac{2r_i}{r_i+c_i}
            + c_i\log\frac{2c_i}{r_i+c_i} \right).
        \label{eq:lambda-i-explicit}
    \end{equation}
    The symmetric deviance $d_{\dev}(p,q)$ in \Cref{def:poisson-deviance-matching} is the average of these entrywise statistics:
    \begin{equation}
        d_{\dev}(p,q) = \frac{1}{P}\sum_{i=1}^{P}\Lambda_i
    \end{equation}
\end{proposition}

\nomenclature[CAW]{\(\mathcal{L}_0\)}{likelihood under the null hypothesis}
\nomenclature[CAX]{\(\mathcal{L}_1\)}{likelihood under the alternative hypothesis}
\nomenclature[CAY]{\(\Lambda_i\)}{likelihood-ratio statistic at entry $i$}

\begin{proof}
    Under the null hypothesis of a common rate $\lambda_i$, the joint likelihood is
    \begin{equation*}
        \mathcal{L}_0(\lambda_i)
        = \frac{e^{-\lambda_i}\lambda_i^{r_i}}{r_i!} \cdot \frac{e^{-\lambda_i}\lambda_i^{c_i}}{c_i!}
        = \frac{e^{-2\lambda_i}\lambda_i^{r_i+c_i}}{r_i!\,c_i!}.
    \end{equation*}
    The maximum-likelihood estimate under $H_0$ is
    \begin{equation*}
        \widehat{\lambda}_{0,i}
        =
        \frac{r_i+c_i}{2}.
    \end{equation*}
    Under the alternative of separate rates $\lambda_i^{(r)}$ and
    $\lambda_i^{(c)}$, the likelihood factorizes and the MLEs are
    \begin{equation*}
        \widehat{\lambda}_{r,i}=r_i,
        \qquad
        \widehat{\lambda}_{c,i}=c_i.
    \end{equation*}
    The likelihood-ratio statistic for entry $i$ is
    \begin{equation*}
        \Lambda_i = -2\log\frac{\mathcal{L}_0(\widehat{\lambda}_{0,i})}{\mathcal{L}_1(\widehat{\lambda}_{r,i},\widehat{\lambda}_{c,i})}.
    \end{equation*}
    Substituting the likelihoods gives
    \begin{equation*}
        \begin{split}
            \Lambda_i
            & = -2\Bigl[-2\widehat{\lambda}_{0,i} + (r_i+c_i)\log\widehat{\lambda}_{0,i} - \log(r_i!\,c_i!)
            \Bigr] \\
            &\quad
            +2\Bigl[
                -(\widehat{\lambda}_{r,i}+\widehat{\lambda}_{c,i})
                +r_i\log\widehat{\lambda}_{r,i}
                +c_i\log\widehat{\lambda}_{c,i}
                -\log(r_i!\,c_i!)
            \Bigr] \\
            & =
            2\Bigl[
                r_i\log\frac{r_i}{\widehat{\lambda}_{0,i}}
                +
                c_i\log\frac{c_i}{\widehat{\lambda}_{0,i}}
            \Bigr].
        \end{split}
    \end{equation*}
    Using $\widehat{\lambda}_{0,i}=(r_i+c_i)/2$, the following is obtained
    \begin{equation*}
        \Lambda_i
        = 2\left( r_i\log\frac{2r_i}{r_i+c_i} + c_i\log\frac{2c_i}{r_i+c_i} \right).
    \end{equation*}
    The quantity defined in \Cref{eq:deviance} is exactly the average of these entrywise likelihood-ratio statistics over all $P$ entries:
    \begin{equation*}
        d_{\mathrm{dev}}(p,q)
        = \frac{1}{P}\sum_{i=1}^{P}\Lambda_i.
    \end{equation*}    
\end{proof}

The Pearson statistic is the classical $\chi^2$ approximation to the same likelihood-ratio test~\cite{pearson1900criterion}.

\begin{remark}[Pearson statistic as classical chi-square test for two Poisson counts]
    \label{rem:pearson-chi-square}
    The Pearson matching statistic $d_{\Pearson}(p,q)$ in \Cref{def:poisson-pearson-matching} is the average of the entrywise classical Pearson $\chi^2$-statistics for testing equality of two Poisson rates.
    This follows directly by substituting the maximum-likelihood estimate $\widehat{\lambda}_{0,i}= (r_i+c_i)/2$ under the null hypothesis into the classical Pearson $\chi^2$ formula.
    This identity explains why the convergence to $\chi^2_1$ in \Cref{thm:high-count-limit} is exact for the Pearson metric in the high-count limit, rather than merely an approximation. 
\end{remark}

After the introduction of the four metrics an examination of their statistical properties is needed. The optimality of the \ac{ssd} for Gaussian noise, as claimed above, will be discussed first.

\begin{proposition}[Optimality of SSD under Gaussian noise]
    \label{prop:optimality-ssd-gaussian}
    Let $p$ and $q$ be two patches generated under the i.i.d. Gaussian model 
    \begin{equation}
        p_i = \mu_i + \varepsilon_i^{(p)}
        \qquad
        q_i = \mu_i + \varepsilon_i^{(q)},
    \end{equation}
    with $\varepsilon_i^{(p)},\,\varepsilon_i^{(q)}\stackrel{\text{i.i.d.}}{\sim}\mathcal{N}(0,\sigma^2)$. 
    Then, among all metrics that depend only on the difference $q-p$, the \ac{ssd} is, up to a positive scaling factor, the unique metric that satisfies $d(p,p)=0$ and that is proportional to the negative log-likelihood of the hypothesis that $p$ and $q$ share the same underlying mean $\mu$.
\end{proposition}

\begin{proof}
    Define $d_i$ as in \Cref{prop:ssd-gaussian-nll}. Under the model, $d_i \sim \mathcal{N}(0, 2\sigma^2)$ independently.
    The joint density of $\widetilde{d}=(d_1,\dots,d_P)$ is
    \begin{equation*}
        f(\widetilde{d}\mid \mu)
        = \prod_{i=1}^{P} \frac{1}{\sqrt{4\pi\sigma^2}} \exp\!\left(-\frac{d_i^2}{4\sigma^2}\right).
    \end{equation*}
    The log-likelihood is thus
    \begin{equation*}
        \log \Pr(d_1,\dots,d_P \vert \mu)
        = -\frac{P}{2}\log(4\pi\sigma^2) -\frac{1}{4\sigma^2} \sum_{i=1}^{P}d_i^2.
    \end{equation*}
    Any metric that depends only on $q-p$ and is proportional to the negative log-likelihood must be of the form
    \begin{equation*}
        d(p,q) = c_1 \sum_{i=1}^{P}(q_i-p_i)^2 + c_0
    \end{equation*}
    with some constants $c_1>0$ and $c_0\in\R$.
    The requirement of $d(p,p) = 0$ requires $c_0 = 0$, such that  
    \begin{equation*}
        d(p,q) = c_1 \sum_{i=1}^{P}(q_i-p_i)^2.
    \end{equation*}
    This is precisely the \ac{ssd} up to scaling.
\end{proof}

Consequently, for stationary Gaussian noise with constant variance, the \ac{ssd} is the statistically canonical matching metric. 
Therefore the $d_{\ssd}$ is used as the default metric for Gaussian observations in \Ac{bmnd}.

\begin{corollary}[Deviance as exact likelihood-based metric for Poisson]
    \label{cor:deviance-exact-poisson}
    Let $p$ and $q$ be two patches with equivalent raw counts $r_i, c_i$ defined in \Cref{eq:raw-counts}, 
    and assume that $r_i, c_i$ are independent Poisson realizations with rates $\lambda_i^{(r)}$ and $\lambda_i^{(c)}$. Then the symmetric deviance $d_{\dev}(p,q)$ in \Cref{def:poisson-deviance-matching} is, up to the factor $1/P$, the likelihood-ratio statistic for testing equality of rates $\lambda_i^{(r)} = \lambda_i^{(c)}$ against separate rates.
\end{corollary}

\begin{proof}
    This follows directly from \Cref{prop:deviance-likelihood-ratio}: for each entry $i$, $\Lambda_i$ is the likelihood-ratio statistic for testing $\lambda_i^{(r)}=\lambda_i^{(c)}$. The total deviance is the average over all entries, hence $d_{\dev}$ is the normalized likelihood-ratio statistic.
\end{proof}

The above may be concluded as follows:
For Poisson data $d_{\dev}$ is the exact metric, but it comes with a more expensive computational cost than $d_{\Pearson}$.
The Pearson statistic is an approximation that becomes exact as the counts grow.
Furthermore $d_{\mathrm{A}}$ offers a Gaussian approximation that is convenient but not exact.
The following remark and proposition quantify these relationships.

\begin{remark}[Pearson as asymptotic approximation of Poisson deviance]
    \label{rem:pearson-approx-deviance}
    Under the setting of \Cref{cor:deviance-exact-poisson}, if all counts are sufficiently large, the Pearson contribution $X^2_i = (r_i - c_i)^2/(r_i+c_i)$ converges in probability to the deviance contribution $\Lambda_i$.
    
    This follows from a second-order Taylor expansion of the log terms in $\Lambda_i$ around the pooled mean.
    Consequently, $d_{\Pearson}$ is a computationally cheaper approximation of the exact likelihood-based deviance in the high-count regime.
\end{remark}

\begin{proposition}[Anscombe-SSD as variance-stabilized Gaussian approximation]
    \label{prop:anscombe-ssd-approx}
    Let $K\sim\Poi(\lambda)$ with $\lambda$ sufficiently large, and let $A(n)=2\sqrt{n+3/8}$ be the Anscombe transform~\cite{anscombe1948transformation}.
    Then 
    \begin{equation}
        \Var(A(K)) \to 1 \quad \text{as } \lambda \to \infty,
    \end{equation}
    and the distribution of $A(K)$ is approximately Gaussian.
    For two independent Poisson counts $r_i, c_i$ with large rates, the squared difference $(A(c_i)-A(r_i))^2/2$ is an approximate likelihood-based matching statistic under a Gaussian model with unit variance.
\end{proposition}

\begin{proof}
    The derivative of $A$ is
    \begin{equation*}
        A'(n)
        = \frac{\partial}{\partial n}\,2\sqrt{n+3/8}
        = \frac{1}{\sqrt{n+3/8}}.
    \end{equation*}
    By the delta method (cf.~\cite[Chapter~3]{vaart1998asymptotic}), for $K\sim\Poi(\lambda)$,
    \begin{equation*}
        \Var(A(K))
        \approx (A'(\lambda))^2 \Var(K)
        = \frac{\lambda}{\lambda+3/8}
        = 1 - \frac{3}{8\lambda+3}.
    \end{equation*}
    The last expression tends to $1$ as $\lambda\to\infty$, so the variance of $A(K)$ converges to $1$.
    By the central limit theorem, $K$ is approximately Gaussian for large $\lambda$, and the smooth transform $A$ preserves approximate Gaussianity. 
    For two independent such variables $A(r_i), A(c_i)$, the difference has approximate variance $2$, so that 
    \begin{equation*}
        \frac{(A(c_i)-A(r_i))^2}{2}
    \end{equation*}
    is an approximate squared standardized Gaussian difference, i.e., an approximate likelihood-based statistic under a unit-variance Gaussian model.
\end{proof}

Let $m \in \{\dev, \Pearson, \mathrm{A}\}$ label the matching metric, and let $\varphi_m(r_i, c_i)$ denote its per-entry contribution:
\begin{equation}
    \begin{aligned}
        \varphi_{\dev}(r_i, c_i)
        &= \Lambda_i
        && \text{for the Poisson deviance,}\\
        \varphi_{\Pearson}(r_i, c_i)
        &= \frac{(r_i - c_i)^2}{r_i + c_i}
        && \text{for the Pearson statistic,}\\
        \varphi_{\mathrm{A}}(r_i, c_i)
        &= \frac{1}{2} \bigl(A(c_i) - A(r_i)\bigr)^2
        && \text{for the Anscombe-SSD.}
    \end{aligned}
    \label{eq:varphi}
\end{equation}
The corresponding patch-level distance is
\begin{equation}
    d_m(p, q)
    = \frac{1}{P} \sum_{i = 1}^{P} \varphi_m(r_i, c_i).
    \label{eq:metric-from-contributions}
\end{equation}

\nomenclature[CAZ]{\(\varphi_m\)}{per-entry contribution of matching metric $m$}

All three Poisson-based metrics $d_{\dev},\, d_{\Pearson}\text{ and }d_{\mathrm{A}}$ are asymptotically equivalent in the high-count regime.
This can be specified, as done in the below theorem, where their common $\chi^2$ limit, under the null hypothesis of equal rates, is shown.

\begin{theorem}[Common high-count limit]
    \label{thm:high-count-limit}
    For each entry $i$, write
    \begin{equation}
        m_i
        = \frac{r_i + c_i}{2},
        \qquad \delta_i = r_i - c_i.
    \end{equation}
    If $m_i \to \infty$ while $|\delta_i| / m_i \to 0$, the deviance, Pearson, and normalized Anscombe contributions
    each satisfy
    \begin{equation}
        \varphi_m(r_i, c_i)
        = \frac{\delta_i^2}{2 m_i}
          + o \left(\frac{\delta_i^2}{m_i}\right).
        \label{eq:common-high-count}
    \end{equation}
    For Pearson's statistic the leading term is exact:
    \begin{equation}
        \varphi_{\Pearson}(r_i,c_i)
        = \frac{(r_i-c_i)^2}{r_i+c_i}
        = \frac{\delta_i^2}{2m_i}.
    \end{equation}
    If, in addition, $r_i$ and $c_i$ are independent $\Poi(\mu_i)$ variables under the equal-rate null and $\mu_i \to \infty$, then
    \begin{equation}
        \varphi_m(r_i, c_i) \xrightarrow{d} \chi^2_1.
        \label{eq:common-chi-square-limit}
    \end{equation}
\end{theorem}

\begin{proof}
    Fix an entry and suppress its index.
    Write
    \begin{equation*}
        r = m + \frac{\delta}{2} = m (1 + t),
        \qquad
        c = m - \frac{\delta}{2} = m (1 - t),
        \qquad
        t = \frac{\delta}{2 m}.
    \end{equation*}
    The assumptions imply $t \to 0$.
    
    For the symmetric deviance, the single-entry contribution is obtained from \Cref{eq:deviance} by omitting the patch-average factor $1 / P$:
    \begin{equation*}
        \varphi_{\dev}(r, c)
        = 2 \left(
            r \log \frac{2 r}{r + c}
            + c \log \frac{2 c}{r + c}
        \right).
    \end{equation*}
    Since $r + c = 2 m$, the two ratios in the logarithms satisfy $2 r / (r + c) = r / m$ and $2 c / (r + c) = c / m$.
    The definition $t = \delta / (2 m)$ and the parameterization of $r$ and $c$ give
    \begin{equation*}
        r = m (1 + t),
        \qquad
        c = m (1 - t),
        \qquad
        \frac{r}{m} = 1 + t,
        \qquad
        \frac{c}{m} = 1 - t.
    \end{equation*}
    Substituting these four identities gives
    \begin{align*}
        \varphi_{\dev}(r, c)
        &= 2 \left(
            r \log \frac{r}{m}
            + c \log \frac{c}{m}
        \right) \\
        &= 2 m \left(
            (1 + t) \log(1 + t)
            + (1 - t) \log(1 - t)
        \right).
    \end{align*}
    Expanding the two terms separately around $t = 0$ gives
    \begin{align*}
        (1 + t) \log(1 + t)
        &= t + \frac{t^2}{2} - \frac{t^3}{6}
           + \frac{t^4}{12} - \frac{t^5}{20} + O(t^6), \\
        (1 - t) \log(1 - t)
        &= -t + \frac{t^2}{2} + \frac{t^3}{6}
           + \frac{t^4}{12} + \frac{t^5}{20} + O(t^6).
    \end{align*}
    Adding these expansions cancels every displayed odd-power term and yields
    \begin{equation*}
        (1 + t) \log(1 + t)
        + (1 - t) \log(1 - t)
        = t^2 + \frac{t^4}{6} + O(t^6).
    \end{equation*}
    Substituting $t = \delta / (2 m)$ therefore gives
    \begin{align*}
        \varphi_{\dev}(r, c)
        &= \frac{\delta^2}{2 m}
           + \frac{\delta^4}{48 m^3}
           + O\left(\frac{\delta^6}{m^5}\right) \\
        &= \frac{\delta^2}{2 m}
           + o\left(\frac{\delta^2}{m}\right),
    \end{align*}
    because $|\delta| / m \to 0$.
    
    For the pooled Pearson statistic, no expansion is required:
    \begin{equation*}
        \varphi_{\Pearson}(r, c)
        = \frac{(c - r)^2}{r + c}
        = \frac{\delta^2}{2 m}.
    \end{equation*}
    
    For the normalized Anscombe statistic, the mean-value theorem gives a point $\xi$ between $r$ and $c$ such that
    \begin{equation*}
        A(c) - A(r)
        = A'(\xi)(c - r)
        = -\frac{\delta}{\sqrt{\xi + 3 / 8}}.
    \end{equation*}
    Since $|\xi - m| \le |\delta| / 2$,
    \begin{equation*}
        \frac{\xi + 3 / 8}{m}
        = 1 + O\left(\frac{|\delta|}{m}\right) + O\left(\frac{1}{m}\right)
        \longrightarrow 1.
    \end{equation*}
    Consequently,
    \begin{align*}
        \varphi_{\mathrm{A}}(r, c)
        &= \frac{1}{2} (A(c) - A(r))^2 \\
        &= \frac{\delta^2}{2 (\xi + 3 / 8)} \\
        &= \frac{\delta^2}{2 m}
           + o\left(\frac{\delta^2}{m}\right).
    \end{align*}
    Thus, all three single-entry contributions have the stated common high-count limit.
    
    Under the equal-rate Poisson null, the Poisson central limit theorem and independence give
    \begin{equation*}
        \frac{r-c}{\sqrt{2 \mu}} \xrightarrow{d} \Norm(0, 1),
    \end{equation*}
    while
    \begin{equation*}
        \frac{r+c}{2 \mu} \longrightarrow 1
    \end{equation*}
    in probability.
    Because the square-root function is continuous at $1$, this also implies
    \begin{equation*}
        \sqrt{\frac{r+c}{2 \mu}} \longrightarrow 1
    \end{equation*}
    in probability.
    Define
    \begin{equation*}
        X_{\mu} = \frac{r-c}{\sqrt{2 \mu}},
        \qquad
        Y_{\mu} = \sqrt{\frac{r+c}{2 \mu}}.
    \end{equation*}
    Then $X_{\mu} \xrightarrow{d} \Norm(0, 1)$ and $Y_{\mu} \to 1$ in probability, and
    \begin{equation*}
        \frac{r-c}{\sqrt{r+c}}
        = \frac{X_{\mu}}{Y_{\mu}}.
    \end{equation*}
    Slutsky's theorem~\cite[Lemma~2.8]{vaart1998asymptotic} permits division by a random sequence converging in probability to the nonzero constant $1$, and therefore gives
    \begin{equation*}
        \frac{X_{\mu}}{Y_{\mu}} \xrightarrow{d} \Norm(0, 1).
    \end{equation*}
    Since
    \begin{equation*}
        \frac{\delta^2}{2 m}
        = \frac{(r-c)^2}{r+c}
        = \left(\frac{r-c}{\sqrt{r+c}}\right)^2,
    \end{equation*}
    the continuous mapping theorem~\cite[Theorem~2.3]{vaart1998asymptotic} shows that the common leading term converges to $\chi^2_1$.
    The remainder in \Cref{eq:common-high-count} vanishes in probability under the same null, which proves \Cref{eq:common-chi-square-limit}.
\end{proof}

The common high-count limit established in \Cref{thm:high-count-limit} shows that deviance, Pearson and Anscombe-SSD all share the same leading term and converge to the same $\chi^2_1$ distribution under the equal-rate null. 
This justifies their use as matching metrics in the high-count regime.
However, for moderate or low counts, the asymptotic approximation becomes inaccurate.

\begin{example}[Breakdown of Anscombe approximation for small counts]
    \label{exa:anscombe-break}
    The variance-stabilization property of the Anscombe transform relies on the asymptotic regime $\lambda\to\infty$. For small counts, the approximation $\Var(A(K))\approx 1$ deteriorates. 
    For example, if $K\sim\Poi(0)$, then $A(K)=2\sqrt{3/8}$ is deterministic, so $\Var(A(K))=0$, whereas the asymptotic formula predicts variance close to $1$.
    For $N\sim\Poi(1)$, direct computation gives
    \begin{equation*}
        \E[A(K)^2]
        = 4\left( 1 + \frac{3}{8} \right) 
        = 5.5,
    \end{equation*}
    while numerical evaluation yields $\E[A(K)]\approx 2.18691$, hence
    \begin{equation*}
        \Var(A(K))
        \approx 5.5 - (2.18691)^2
        \approx 0.717443,
    \end{equation*}
    which deviates substantially from $1$. 
    Hence, for low-count Poisson data, the Anscombe approximation in general, and thus also the Anscombe-SSD in particular, no longer provides an accurate approximation to a likelihood-based metric.
\end{example}

The high-count limit \Cref{thm:high-count-limit} proves the convergence of all three Poisson metrics to the same $\chi^2$ distribution.
This approximation can become rather poor for low counts.
Thus the next subsection introduces an exact finite-count calibration that corrects for these deviations at small counts.

\subsection{Exact Finite-Count Moment Calibration}
\label{subsec:ex-fin-count-mom-cali}
The common high-count limit replaces each entrywise discrepancy by an asymptotic $\chi^2_1$ variable.
At small pooled counts however, only finitely many allocations between the reference and the candidate are possible, and the conditional mean and variance of the discrepancy can depend on the pooled count.
Using the high-count asymptotic mean $1$ and variance $2$ in the low-count regime therefore assigns different effective scales to otherwise comparable low- and high-count patches.

This effect can be corrected by computing exact conditional moments under an equal-rate Poisson null.
The moments are exact finite-count quantities.
The acceptance rules constructed from them are moment-calibrated rules, not exact conditional tests.

\begin{lemma}[Conditional Poisson splitting]
    \label{lem:poisson-splitting}
    If $r$ and $c$ are independent $\Poi(\mu)$ variables with $\mu > 0$, then
    \begin{equation}
        r \mid (r + c = n) \sim \Bin(n, 1 / 2).
        \label{eq:poisson-splitting}
    \end{equation}
\end{lemma}

\nomenclature[CBA]{\(r\)}{reference count in the conditional splitting model}
\nomenclature[CBB]{\(c\)}{candidate count in the conditional splitting model}

\begin{proof}
    The sum of the independent variables satisfies $r + c \sim \Poi(2 \mu)$.
    For $x \in \{0, \ldots, n\}$,
    \begin{align*}
        \Pr(r = x \mid r + c = n)
        &= \frac{\Pr(r = x, c = n - x)}{\Pr(r + c = n)}\\
        &= \frac{e^{-2 \mu} \mu^n / [x! (n - x)!]}{e^{-2 \mu} (2 \mu)^n / n!}\\
        &= \frac{n!}{x! (n - x)!} \frac{\mu^n}{(2 \mu)^n}\\
        &= \binom{n}{x} 2^{-n}.
    \end{align*}
    This is the probability mass function of $\Bin(n, 1/2)$.
\end{proof}

Conditioning therefore replaces the unknown-rate problem by the finite allocation of $n$ observed events between two measurements~\cite{kingman1993poisson}.

\begin{proposition}[Universal finite-count moments]
    \label{prop:finite-count-moments}
    For each metric contribution $\varphi_m$ defined in \Cref{eq:varphi} and each pooled count $n \in \N_0$, the exact conditional moments are
    \begin{align}
        \mu_m(n) &= \sum_{x = 0}^n \binom{n}{x} 2^{-n} \varphi_m(x, n - x),\\
        v_m(n) &= \sum_{x = 0}^n \binom{n}{x} 2^{-n} \varphi_m(x, n - x)^2 - \mu_m(n)^2.
        \label{eq:finite-count-moments}
    \end{align}
    These moments depend only on the metric and $n$, not on the image or the unknown rate.
\end{proposition}

\nomenclature[CBC]{\(\mu_m\)}{conditional finite-count mean of metric $m$}
\nomenclature[CBD]{\(v_m\)}{conditional finite-count variance of metric $m$}

\begin{proof}
    Using \Cref{lem:poisson-splitting}, conditional on $r + c = n$ every possible pair is of the form $(r, c) = (x, n - x)$ with probability $\binom{n}{x} 2^{-n}$.
    The first expression of \Cref{eq:finite-count-moments} is the conditional expectation of $\varphi_m(r, c)$.
    The second expression subtracts the square of this expectation.
    Neither expression contains $\mu$.
\end{proof}

At $n = 0$, the only allocation is $(0, 0)$, whose contribution is zero for all three metrics.
Hence $\mu_m(0) = v_m(0) = 0$.
Pearson's statistic also admits a closed form for every positive count.
\begin{proposition}[Exact Pearson moments]
    \label{prop:pearson-moments}
    For $n \geq 1$, the conditional Pearson moments are
    \begin{equation}
        \mu_{\mathrm{Pearson}}(n) = 1,
        \qquad
        v_{\mathrm{Pearson}}(n) = 2 - \frac{2}{n}.
        \label{eq:pearson-moments}
    \end{equation}
    At $n = 0$, both moments are zero.
\end{proposition}

\begin{proof}
    For $n = 0$, the only possible allocation is $(r, c) = (0, 0)$.
    The zero-total convention therefore gives $\mu_\Pearson(0) = v_\Pearson(0) = 0$.

    Now let $n \geq 1$.
    Using \Cref{lem:poisson-splitting}, conditional on $r + c = n$ gives $r \sim \Bin(n, 1/2)$.
    Equivalently, each of the $n$ pooled events is assigned independently to either $r$ or $c$ with probability $1/2$.
    Let $Z_j = 1$ when event $j$ is assigned to $r$ and $Z_j = -1$ when it is assigned to $c$.
    Then the $Z_j$ are independent variables and
    \begin{equation*}
        r - c = 2r - n = \sum_{j = 1}^n Z_j.
    \end{equation*}
    So the Pearson contribution becomes
    \begin{equation*}
        \varphi_{\Pearson}(r, c)
        = \frac{(c - r)^2}{r + c}
        = \frac{1}{n} \left(\sum_{j = 1}^n Z_j\right)^2.
    \end{equation*}
    Each $Z_j$ has mean zero and satisfies $Z_j^2 = 1$.
    Independence therefore gives
    \begin{equation*}
        \E\left[\left(\sum_{j = 1}^n Z_j\right)^2\right]
        = \sum_{j = 1}^n \E[Z_j^2]
        = n,
    \end{equation*}
    because all cross terms have zero expectation.
    It follows that
    \begin{equation*}
        \mu_{\Pearson}(n)
        = \E[\varphi_{\Pearson} \mid r + c = n]
        = 1.
    \end{equation*}

    To obtain the variance, expand the fourth power.
    Writing out the product gives
    \begin{equation*}
        \left(\sum_{j = 1}^n Z_j\right)^4
        = \sum_{j, k, \ell, m = 1}^n Z_j Z_k Z_\ell Z_m.
    \end{equation*}
    Because the variables are independent and have mean zero, a product has zero expectation whenever any index occurs an odd number of times.
    Only two index patterns remain.
    First, all four indices can be equal, producing the $n$ terms $Z_j^4$.
    Second, two distinct indices can each occur twice, producing terms $Z_j^2 Z_k^2$.
    There are $\binom{n}{2}$ choices of the pair $\{j, k\}$ and
    \begin{equation*}
        \frac{4!}{2! 2!} = 6
    \end{equation*}
    orders of its four factors.
    Since $Z_j^4 = 1$ and $Z_j^2 Z_k^2 = 1$, every surviving term has expectation one.
    Hence
    \begin{equation*}
        \E\left[\left(\sum_{j = 1}^n Z_j\right)^4\right]
        = n + 6 \binom{n}{2}
        = 3 n^2 - 2 n.
    \end{equation*}
    Consequently,
    \begin{align*}
        v_{\mathrm{Pearson}}(n)
        &= \E[\varphi_{\mathrm{Pearson}}^2 \mid r + c = n]
           - \mu_{\mathrm{Pearson}}(n)^2\\
        &= \frac{3 n^2 - 2 n}{n^2} - 1
         = 2 - \frac{2}{n}.
    \end{align*}
\end{proof}

\begin{corollary}[Patch-level null moments]
    \label{cor:patch-null-moments}
    Let $\boldsymbol{n} = (n_1, \ldots, n_P)$ be the vector of pooled counts.
    If the entry pairs are conditionally independent given $\boldsymbol{n}$, then
    \begin{align}
        \E[d_m(p, q) \mid \boldsymbol{n}]
        &= \frac{1}{P} \sum_i \mu_m(n_i),\\
        \Var[d_m(p, q) \mid \boldsymbol{n}]
        &= \frac{1}{P^2} \sum_i v_m(n_i).
        \label{eq:patch-null-moments}
    \end{align}
\end{corollary}

\nomenclature[CBE]{\(\boldsymbol{n}\)}{vector of pooled counts}

\begin{proof}
    Let
    \begin{equation*}
        W_i = \varphi_m(r_i, c_i)
    \end{equation*}
    denote the contribution of entry $i$.
    By \Cref{prop:finite-count-moments}, conditioning on the pooled count $n_i$ gives
    \begin{equation*}
        \E[W_i \mid n_i] = \mu_m(n_i),
        \qquad
        \Var(W_i \mid n_i) = v_m(n_i).
    \end{equation*}
    Since by \Cref{eq:metric-from-contributions} the patch distance is the average of these contributions,
    \begin{equation*}
        d_m(p, q) = \frac{1}{P} \sum_{i = 1}^P W_i,
    \end{equation*}
    linearity of conditional expectation gives
    \begin{align*}
        \E[d_m(p, q) \mid \boldsymbol{n}]
        &= \E\left[
            \frac{1}{P} \sum_i W_i
            \,\middle|\,
            \boldsymbol{n}
        \right]\\
        &= \frac{1}{P} \sum_i \E[W_i \mid n_i]\\
        &= \frac{1}{P} \sum_i \mu_m(n_i).
    \end{align*}
    This calculation does not require independence.
    
    For the variance, conditional independence makes all conditional covariance terms vanish, so
    \begin{equation*}
        \Var\left(\sum_i W_i \,\middle|\, \boldsymbol{n}\right)
        = \sum_i \Var(W_i \mid n_i).
    \end{equation*}
    Moreover, $\Var(a X) = a^2 \Var(X)$.
    Therefore,
    \begin{align*}
        \Var[d_m(p, q) \mid \boldsymbol{n}]
        &= \frac{1}{P^2}
           \Var\left(\sum_i W_i \,\middle|\, \boldsymbol{n}\right)\\
        &= \frac{1}{P^2} \sum_i v_m(n_i).
    \end{align*}
\end{proof}

The entrywise moments remain exact regardless of the patch size.
The patch-level variance relies on conditional independence and is only an approximation when overlapping patches cause the same observed voxel to occur in different entry pairs.

There are two ways to use the conditional moments without evaluating the complete finite-count distribution of a patch.

\begin{definition}[Reference finite-count threshold]
    \label{def:fin-count-threshold}
    Let $p$ be the reference patch and $q$ a candidate patch.
    Let $t \geq 0$ be the threshold multiplier, measured in estimated null standard deviations.
    Let $\beta_s \geq 0$ control an optional allowance for structural differences, and let $\kappa_s > 0$ control how this allowance scales with the mean reference intensity.
    Before inspecting $q$, the reference-only rule estimates the pooled count under the equal-mean null by
    \begin{equation}
        \widehat{n}_i^{\mathrm{ref}} = 2 r_i.
    \end{equation}
    The factor $2$ follows from the equal-mean null, under which the expected pooled count is twice the expected reference count.
    Define
    \begin{align}
        \overline{\mu}_{m, p}
        &= \frac{1}{P} \sum_i \mu_m(\widehat{n}_i^{\mathrm{ref}}), &
        V_{m, p}
        &= \frac{1}{P^2} \sum_i v_m(\widehat{n}_i^{\mathrm{ref}}),\\
        \overline{r}_p
        &= \frac{1}{P} \sum_i r_i, &
        T_{\mathrm{ref}}(p; t)
        &= \overline{\mu}_{m, p} + t \sqrt{V_{m, p}}
           + \beta_s \overline{r}_p^{\kappa_s}.
        \label{eq:reference-threshold}
    \end{align}
    The candidate is accepted when $d_m(p, q) < T_{\mathrm{ref}}(p; t)$.
    Setting $\beta_s = 0$ removes the structural allowance and leaves only the finite-count moment calibration.
\end{definition}

\nomenclature[CBF]{\(t\)}{match-threshold multiplier}
\nomenclature[CBG]{\(\beta_s\)}{allowance for structural differences}
\nomenclature[CBH]{\(\kappa_s\)}{intensity exponent of the structural allowance}
\nomenclature[CBI]{\(T_{\mathrm{ref}}(p; t)\)}{reference-adaptive match threshold for patch $p$}

The first two terms of \Cref{eq:reference-threshold} use exact conditional moment formulas evaluated at the reference-only pooled-count estimate $\widehat{n}_i^{\mathrm{ref}} = 2 r_i$.
Consequently, $T_{\mathrm{ref}}$ is a plug-in, moment-calibrated threshold rather than an exact conditional quantile.
The structural term is heuristic and is not part of the Poisson null but admits controlled structural variation.
Because every term depends only on the reference, a candidate cannot alter its own acceptance threshold.

\begin{definition}[Candidate-standardized distance]
    \label{def:standardized-distance}
    Let $\varphi_m$ be the selected per-entry contribution from \Cref{eq:varphi}.
    Let $p$ be the reference patch, $q$ a candidate patch, and $t \geq 0$ the standardized-distance cutoff.
    For each entry, define the observed pooled count
    \begin{equation}
        n_i^{pq} = r_i + c_i.
    \end{equation}
    If
    \begin{equation}
        \sum_{i = 1}^P v_m(n_i^{pq}) > 0,
    \end{equation}
    define
    \begin{equation}
        Z_m(p, q) =
        \frac{
            \sum_{i = 1}^P \varphi_m(r_i, c_i)
            - \sum_{i = 1}^P \mu_m(n_i^{pq})
        }{
            \sqrt{\sum_{i = 1}^P v_m(n_i^{pq})}
        }.
        \label{eq:standardized-distance}
    \end{equation}
    If the summed variance is zero, set $Z_m(p, q) = 0$.
    The candidate is accepted when $Z_m(p, q) < t$.
\end{definition}

\nomenclature[CBJ]{\(Z_m\)}{candidate-standardized distance for metric $m$}

The factors $P^{-1}$ cancel between the centered patch distance and its conditional standard deviation, which explains their absence from \Cref{eq:standardized-distance}.

\begin{corollary}[Standardized null moments]
    \label{cor:standardized-null-moments}
    Under the exact count model, condition on the pooled-count vector $\boldsymbol{n} = (n_1, \ldots, n_P)$.
    If the entry pairs are conditionally independent and
    \begin{equation}
        \sum_{i = 1}^P v_m(n_i) > 0,
    \end{equation}
    then
    \begin{equation}
        \E[Z_m(p, q) \mid \boldsymbol{n}] = 0,
        \qquad
        \Var[Z_m(p, q) \mid \boldsymbol{n}] = 1.
    \end{equation}
\end{corollary}

\begin{proof}
    Let
    \begin{equation*}
        A_m(p, q)
        = \sum_{i = 1}^P \varphi_m(r_i, c_i).
    \end{equation*}
    By \Cref{prop:finite-count-moments} and conditional independence,
    \begin{equation*}
        \E[A_m(p, q) \mid \boldsymbol{n}]
        = \sum_{i = 1}^P \mu_m(n_i),
    \end{equation*}
    and
    \begin{equation*}
        \Var[A_m(p, q) \mid \boldsymbol{n}]
        = \sum_{i = 1}^P v_m(n_i).
    \end{equation*}
    Therefore, \Cref{eq:standardized-distance} can be written as
    \begin{equation*}
        Z_m(p, q)
        =
        \frac{
            A_m(p, q)
            - \E[A_m(p, q) \mid \boldsymbol{n}]
        }{
            \sqrt{\Var[A_m(p, q) \mid \boldsymbol{n}]}
        }.
    \end{equation*}
    Taking the conditional expectation gives
    \begin{equation*}
        \E[Z_m(p, q) \mid \boldsymbol{n}] = 0.
    \end{equation*}
    Since the denominator is positive and fixed after conditioning on
    $\boldsymbol{n}$,
    \begin{equation*}
        \Var[Z_m(p, q) \mid \boldsymbol{n}]
        =
        \frac{
            \Var[A_m(p, q) \mid \boldsymbol{n}]
        }{
            \Var[A_m(p, q) \mid \boldsymbol{n}]
        }
        = 1.
    \end{equation*}
\end{proof}

Candidate standardization expresses discrepancy in conditional null standard deviations and permits comparisons across count levels.
Each candidate influences its own center and scale, so standardized and raw distance rankings need not agree.

\subsection{Removing the Noise Bias from SSD}
\label{subsec:ssd-noise-bias}
Raw \ac{ssd} contains structural mismatch and the discrepancy between two noise realizations.
The expected noise contribution has a simple form under either observation model.

\begin{proposition}[Gaussian SSD noise bias]
    \label{prop:gaussian-ssd-noise-bias}
    For stationary zero-mean Gaussian noise and a displacement $\Delta$ between patch
    origins, the expected noise-only contribution is
    \begin{equation}
        B_G(\Delta) = 2 P (\mathfrak{c}(0) - \mathfrak{c}(\Delta)).
        \label{eq:gaussian-ssd-bias}
    \end{equation}
\end{proposition}

\nomenclature[CBK]{\(B_G(\Delta)\)}{expected Gaussian noise contribution to SSD at displacement $\Delta$}

\begin{proof}
    Let $u_i$ be the location of entry $i$ in the reference patch.
    The corresponding candidate entry is at $u_i + \Delta$.
    The noise-only \ac{ssd} is
    \begin{equation*}
        \sum_{i = 1}^P
        \left[n(u_i + \Delta) - n(u_i)\right]^2.
    \end{equation*}
    Since the noise is zero mean, each difference has mean zero.
    Expanding the variance and using the stationary covariance function $\mathfrak{c}$ gives
    \begin{align*}
        \E\left[
            \left(n(u_i + \Delta) - n(u_i)\right)^2
        \right]
        &= \Var[n(u_i + \Delta) - n(u_i)]\\
        &= \Var[n(u_i + \Delta)] + \Var[n(u_i)]\\
        &\quad - 2 \Cov(n(u_i + \Delta), n(u_i))\\
        &= 2 \mathfrak{c}(0) - 2 \mathfrak{c}(\Delta).
    \end{align*}
    The last equality uses the symmetry $\mathfrak{c}(-\Delta) = \mathfrak{c}(\Delta)$ of a real-valued stationary covariance function.
    Linearity of expectation therefore yields
    \begin{equation*}
        \E\left[
            \sum_{i = 1}^P
            \left(n(u_i + \Delta) - n(u_i)\right)^2
        \right]
        = 2 P (\mathfrak{c}(0) - \mathfrak{c}(\Delta)),
    \end{equation*}
    which proves \Cref{eq:gaussian-ssd-bias}.
    Independence between different patch entries is not required for this expectation.
\end{proof}

\begin{proposition}[Scaled-Poisson SSD noise bias]
    \label{prop:poisson-ssd-noise-bias}
    Let $P$ be the number of entries in each patch.
    For independent scaled-Poisson observations with clean stored means $X_{p, i}$ and $X_{q, i}$, the expected noise-only contribution to the patch \ac{ssd} is
    \begin{equation}
        s_P \sum_{i = 1}^P (X_{p, i} + X_{q, i}).
        \label{eq:exact-poisson-ssd-bias}
    \end{equation}
    Replacing the unknown means by the observed patch entries gives the plug-in bias
    \begin{equation}
        B_P(p, q) = s_P \sum_{i = 1}^P [p_i + q_i]_+.
        \label{eq:poisson-ssd-bias}
    \end{equation}
\end{proposition}

\nomenclature[CBL]{\(B_P(p, q)\)}{plug-in scaled-Poisson noise contribution to SSD between patches $p$ and $q$}

\begin{proof}
    Write the centered errors as
    \begin{equation*}
        \varepsilon_{p, i} = Y_{p, i} - X_{p, i},
        \qquad
        \varepsilon_{q, i} = Y_{q, i} - X_{q, i}.
    \end{equation*}
    By \Cref{lem:scaled-poisson-moments},
    \begin{equation*}
        \E[\varepsilon_{p, i}] = \E[\varepsilon_{q, i}] = 0,
        \qquad
        \Var(\varepsilon_{p, i}) = s_P X_{p, i},
        \qquad
        \Var(\varepsilon_{q, i}) = s_P X_{q, i}.
    \end{equation*}
    Independence of the two observations gives
    \begin{align*}
        \E\left[
            (\varepsilon_{q, i} - \varepsilon_{p, i})^2
        \right]
        &= \Var(\varepsilon_{q, i} - \varepsilon_{p, i})\\
        &= s_P (X_{q, i} + X_{p, i}).
    \end{align*}
    Thus the exact expected noise contribution to the patch \ac{ssd} is
    \begin{equation*}
        s_P \sum_{i = 1}^P (X_{p, i} + X_{q, i}).
    \end{equation*}
    Replacing the unknown clean means by the observed patch entries and enforcing nonnegativity gives
    \begin{equation*}
        B_P(p, q)
        = s_P \sum_{i = 1}^P [p_i + q_i]_+,
    \end{equation*}
    which is \Cref{eq:poisson-ssd-bias}.
\end{proof}

\begin{definition}[Noise-bias-corrected SSD]
    \label{def:corrected-ssd}
    For a reference patch $p$ and a non-reference candidate patch $q$, let
    \begin{equation}
        B_*(p, q)
        =
        \begin{cases}
            B_G(\Delta), & \text{under the Gaussian model},\\
            B_P(p, q), & \text{under the scaled-Poisson model},
        \end{cases}
    \end{equation}
    where $\Delta$ is the displacement between their origins.
    For a correction strength $\gamma \geq 0$, define
    \begin{equation}
        \widetilde{d}_{\SSD}(p, q)
        = d_{\SSD}(p, q) - \gamma B_*(p, q),
        \label{eq:corrected-ssd}
    \end{equation}
    The choice $\gamma = 0$ disables the correction, while $\gamma = 1$ subtracts one expected or estimated noise contribution.
    Values $\gamma > 1$ apply a stronger heuristic correction.
    Because the subtracted term is an expected or estimated noise contribution, $\widetilde{d}_{\SSD}$ may be negative and is a matching score rather than a metric.
    Unlike the count-aware distances, it retains squared stored-intensity units.
\end{definition}

\nomenclature[CBM]{\(\gamma\)}{SSD noise-correction strength}
\section{Hard Thresholding}
\label{sec:hard-thresholding}
After block matching, each group of similar patches is transformed into a coefficient domain in which much of the signal energy is concentrated in a relatively small number of coefficients.
Hard thresholding suppresses coefficients whose magnitudes are not large relative to their estimated noise standard deviations, producing a pilot group for the subsequent Wiener stage.
This section defines the thresholding operation and develops the corresponding Gaussian and Poisson variance models.

The coefficientwise hard-thresholding rule and the construction of a transform-domain pilot group follow \Ac{bm3d}/\Ac{bm4d}~\cite{dabov2007image,maggioni2012nonlocal}.
Exact transform-domain variance propagation for correlated Gaussian noise and its use in collaborative shrinkage follow generalized collaborative filtering~\cite{makinen2020collaborative}.
Direct-Poisson variance maps, the shared-source covariance model for repeated voxels, and partial covariance propagation extend that framework in \Ac{bmnd}.
An optional soft-thresholding variant is also included in \Ac{bmnd}.

\begin{definition}[Coefficient thresholding]
    \label{def:coefficient-thresholding}
    For a transform-domain coefficient $Y_k$ and noise standard deviation $\sigma_k$, define
    \begin{equation}
        C_k^{\HT} =
        \begin{cases}
            Y_k, & \mid Y_k\mid > \lambda_\HT\sigma_k,\\
            0, & \text{otherwise},
        \end{cases}
        \label{eq:hard-threshold}
    \end{equation}
    where $\lambda_{\HT}$ is a dimensionless threshold parameter. 
    The corresponding soft-thresholded coefficient is
    \begin{equation}
        C_k^{\mathrm{soft}}
        = \sgn(Y_k)[\mid Y_k\mid  - \lambda_\HT \sigma_k]_+,
        \label{eq:soft-threshold}
    \end{equation}
    with $\sgn(z) = z/\mid z\mid $ for $z \in \C$.
\end{definition}

\nomenclature[DAA]{\(C_k^\HT\)}{hard-thresholded coefficient $k$}

For a group $g$, let
\begin{equation}
    \YY_g=(Y_{g,1},\ldots,Y_{g,d})^\top
\end{equation}
be its transform-domain coefficients.
The hard-thresholded coefficient
vector is defined componentwise by \Cref{def:coefficient-thresholding}:
\begin{equation}
    \CC^{\HT}_g
    = \bigl(C_{g,1}^{\HT},\ldots,C_{g,d}^{\HT}\bigr)^\top.
\end{equation}

\begin{definition}[Hard-threshold group filter]
    \label{def:hard-threshold-group-filter}
    Define the coefficient mask and corresponding diagonal group filter by
    \begin{equation}
        a_{g,k}
        = \mathbbm{1}\!\left\{|Y_{g,k}| > \lambda_\HT \sigma_k\right\},
        \qquad
        A_g = \operatorname{diag}(a_{g,1}, \ldots, a_{g,d}).
        \label{eq:hard-threshold-group-filter}
    \end{equation}
    The hard-thresholded coefficient vector can therefore be written as
    \begin{equation}
        \CC_g^{\HT} = A_g \YY_g.
    \end{equation}
\end{definition}

\nomenclature[DAB]{\(\CC_g^\HT\)}{hard-thresholded coefficient group $g$}

Scaling by $\sigma_k$ makes the decision dimensionless.
A coefficient is retained because it is large relative to its uncertainty, not merely large in stored units.

The threshold in \Cref{eq:hard-threshold} depends on the coefficient variance $\sigma_k^2$, which is obtained from the spatial-domain covariance via \Cref{lem:covariance-propagation}. 
For stationary Gaussian noise, this covariance is built from the autocovariance $\mathfrak{c}(\tau)$ as described in \Cref{sec:observation-models}. 
For direct Poisson observations, two approximations of the spatial-domain covariance are being considered.
For each group entry $i \in \{1, \ldots, d\}$, let $u_{g,i} \in \Omega$ denote its source sampling location.

\begin{proposition}[Diagonal Poisson variance propagation]
    \label{prop:diagonal-poisson-variance}
    Suppose that the entries of the group are treated as independent and that the variance associated with entry $i$ is estimated by $m_{\mathrm{v}}(u_{g,i})$.
    Then
    \begin{equation}
        \Sigma_g \approx M_{\mathrm{v},g}
        = \operatorname{diag}\bigl(
            m_{\mathrm{v}}(u_{g,1}), \ldots, m_{\mathrm{v}}(u_{g,d})
          \bigr),
    \end{equation}
    and 
    \begin{equation}
        \sigma_k^2
        \approx \sum_i \mid T_{ki}\mid ^2 m_{\mathrm{v}}(u_{g,i}).
        \label{eq:diagonal-poisson-variance}
    \end{equation}
\end{proposition}

\nomenclature[DAC]{\(M_{\mathrm{v},g}\)}{diagonal group variance-map matrix}
\nomenclature[DAD]{\(u_{g,i}\)}{source sampling location of group entry $i$}

\begin{proof}
    Treating the grouped entries as independent is equivalent to the covariance approximation
    \begin{equation*}
        \Sigma_g \approx M_{\mathrm{v},g}.
    \end{equation*}
    Since $T$ is deterministic and linear, covariance propagation (\Cref{lem:covariance-propagation}) gives
    \begin{equation*}
        \Sigma_{N,g}
        = T \Sigma_g T^\H
        \approx T M_{\mathrm{v},g} T^\H.
    \end{equation*}
    The variance of the $k$th transformed coefficient is the corresponding diagonal entry
    \begin{equation*}
        \begin{split}
            \sigma_k^2
            &\approx (T M_{\mathrm{v},g} T^\H)_{kk}\\
            &= \sum_{i,j}T_{ki}(M_{\mathrm{v},g})_{ij} \overline{T_{kj}}\\
            &= \sum_i \mid T_{ki}\mid ^2m_{\mathrm{v}}(u_{g,i}),
        \end{split}
    \end{equation*}
    because $(M_{\mathrm{v},g})_{ij}=0$ whenever $i\ne j$.
    This is \Cref{eq:diagonal-poisson-variance}.
\end{proof}

This diagonal model is inexpensive, but it forgets where each group entry originated.
That loss matters when patches overlap.
The same source voxel can appear at several positions in one stacked group, and those occurrences share exactly the same noise realization.

Let
\begin{equation}
    \mathcal{U}_g = \{u_{g,i} : i = 1, \ldots, d\}
\end{equation}
be the set of distinct source sampling locations represented in group $g$, and define
\begin{equation}
    \mathcal{I}_g(u) = \{i : u_{g,i} = u\}
\end{equation}
as the set of group entries originating from $u \in \mathcal{U}_g$.

Under the independent-count assumption stated in \Cref{lem:scaled-poisson-moments}, distinct source voxels have zero covariance. Consequently, the group covariance satisfies
\begin{equation}
    (\Sigma_g)_{ij} = 
    \begin{cases}
        m_{\mathrm{v}}(u_{g,i}),
        & u_{g,i} = u_{g,j},\\
        0, 
        & u_{g,i} \neq u_{g,j}.
    \end{cases}
    \label{eq:shared-voxel-group-covariance}
\end{equation}
In the direct-Poisson setting, the unknown variance $\Var(y(u))$ is replaced by the variance-map value $m_{\mathrm{v}}(u)$.

The transformed covariance and the coefficient variance are then obtained from \Cref{lem:covariance-propagation}.
This propagation combines all weights that multiply one underlying random variable before computing covariance, which preserves overlap-induced dependence.
A partial approximation retains this exact structure for only the first selected nonlocal transform planes and uses \Cref{eq:diagonal-poisson-variance} for the remainder

\nomenclature[DAE]{\(\mathcal{U}_g\)}{distinct source sampling locations in group $g$}
\nomenclature[DAF]{\(\mathcal{I}_g(u)\)}{group entries originating from sampling location $u$}

\begin{definition}[Effective transform weights]
    \label{def:effective-weights}
    For each coefficient index $k$ and source voxel $u \in \mathcal{U}_g$, define the effective weight
    \begin{equation}
        t_{ku} = \sum_{i \in \mathcal{I}_g(u)} T_{ki},
        \label{eq:effective-weights}
    \end{equation}
\end{definition}

\nomenclature[DAG]{\(t_{ku}\)}{effective transform weight for coefficient $k$ and source location $u$}

\begin{definition}[Partial shared-source Poisson variance model]
    \label{def:partial-shared-source-variance}
    Let $k_{\mathrm{ex}} \in \{0, \ldots, d\}$ denote the number of leading transform indices for which shared-source covariance propagation is evaluated explicitly, and define $\mathcal{K}_{\mathrm{ex}} = \{j \in \{1, \ldots, d\} : j \leq k_{\mathrm{ex}}\}$.
    The variance used in the threshold \Cref{eq:hard-threshold} is defined as
    \begin{equation}
        \sigma_k^2
        = \begin{cases}
            \displaystyle
            \sum_{u\in\mathcal{U}_g}
            \mid t_{ku}\mid^2\,m_{\mathrm{v}}(u),
            & k\in\mathcal{K}_{\mathrm{ex}},\\[2.5ex]
            \displaystyle
            \sum_{i=1}^{d}
            \mid T_{ki}\mid^2\,m_{\mathrm{v}}(u_{g,i}),
            & k\notin\mathcal{K}_{\mathrm{ex}}.            
        \end{cases}
        \label{eq:partial-poisson-variance}        
    \end{equation}
\end{definition}

\nomenclature[DAH]{\(\mathcal{K}_\mathrm{ex}\)}{set of transform indices for exact variance}

\begin{remark}[Limiting cases]
    If $\mathcal{K}_{\mathrm{ex}}=\emptyset$, \Cref{def:partial-shared-source-variance} reduces to the diagonal approximation \Cref{eq:diagonal-poisson-variance}.
    If $\mathcal{K}_{\mathrm{ex}}=\{1,\dots,d\}$, it coincides with the shared-source model \Cref{eq:shared-voxel-group-covariance} for all coefficients. 
\end{remark}

\begin{definition}[Pilot Group]
    \label{def:pilot-group}
    For a group $g$ with hard-thresholded coefficient vector
    $\CC^{\HT}_g$, the corresponding spatial-domain pilot group is
    \begin{equation}
        \widehat{\xx}^{\HT}_g
        = T^{-1}\CC^{\HT}_g.
        \label{eq:pilot-group}
    \end{equation}
\end{definition}

\nomenclature[DAI]{\(\widehat{\xx}^\HT_g\)}{pilot group}

\begin{definition}[Retained coefficient count]
\label{def:retained-count}
    The number of retained coefficients in a group $g$ after hard thresholding is
    \begin{equation}
        N_g^{\HT}
        = \sum_{k=1}^{d} \mathbbm{1}\!\left\{C_{g,k}^{\HT}\neq 0 \right\}
        = \sum_{k=1}^{d} a_{g,k}
        \label{eq:number-retained-coefficients}
    \end{equation}
\end{definition}

\nomenclature[DAJ]{\(N_g^\HT\)}{retained coefficient count}

This quantity is used in \Cref{sec:aggregation-weights} to define the group-dependent aggregation weight.
\section{Wiener Shrinkage}
\label{sec:wiener}
Hard thresholding makes a binary decision.
Wiener shrinkage instead uses a continuous gain, retaining most of a coefficient when estimated signal power dominates and attenuating it when noise dominates.

Pilot-guided grouping, coefficientwise Wiener gains, and gain-energy aggregation are inherited from \Ac{bm3d}/\Ac{bm4d}~\cite{dabov2007image,maggioni2012nonlocal}.
The scalar minimum-risk Wiener rule is a standard estimation result~\cite{wiener1949extrapolation}.
Coefficient-dependent gains, bias-corrected pilot powers, and covariance-aware risk extend that construction in \Ac{bmnd}.

\begin{definition}[Paired pilot and observation groups]
    \label{def:paired-wiener-groups}
    Let $\widehat{x}^{\HT}$ be the aggregated output of the first stage.
    For a reference origin, block matching on this pilot selects an ordered set of patch origins $\mathcal{I}_g$.
    At exactly those origins, define the paired pilot and observation groups by
    \begin{equation}
        \boldsymbol{\gamma}_g
        = \operatorname{stack}\{\mathfrak{P}_i \widehat{x}^{\HT}: i \in \mathcal{I}_g\},
        \qquad
        \yy_g
        = \operatorname{stack}\{\mathfrak{P}_i y: i \in \mathcal{I}_g\},
        \label{eq:paired-wiener-groups}
    \end{equation}
    where $\mathfrak{P}_i$ is the linear operator that extracts and vectorizes the patch at origin $i$.
    The same ordering and group transform $T_g$ are used for both stacks:
    \begin{equation}
        \boldsymbol{\Gamma}_g = T_g\boldsymbol{\gamma}_g,
        \qquad
        \YY_g = T_g\yy_g.
        \label{eq:paired-wiener-transforms}
    \end{equation}
\end{definition}

\nomenclature[EAA]{\(\widehat{x}^\HT\)}{aggregated pilot array}
\nomenclature[EAB]{\(\mathcal{I}_g\)}{ordered patch origins selected by pilot block matching}
\nomenclature[EAC]{\(\mathfrak{P}_i\)}{linear patch extraction and vectorization operator at origin $i$}
\nomenclature[EAD]{\(\boldsymbol{\gamma}_g\)}{paired pilot group}
\nomenclature[EAE]{\(\boldsymbol{\Gamma}_g\)}{transformed paired pilot group}

\subsection{Scalar Wiener Risk and Gain}
\begin{theorem}[Optimal scalar Wiener gain]
    \label{thm:wiener-gain}
    Let $Y_k = X_k + N_k$, where signal and noise are uncorrelated, with powers
    \begin{equation}
        S_k = \E[|X_k|^2],
        \qquad
        \E[|N_k|^2] = \sigma_k^2.
    \end{equation}
    Among real scalar linear estimates $\widehat{X}_k = G_k Y_k$, the mean-squared error is the risk
    \begin{equation}
        \mathfrak{R}_k(G_k)
        = (1 - G_k)^2 S_k + G_k^2 \sigma_k^2
        \label{eq:scalar-risk}
    \end{equation}
    and is minimized by
    \begin{equation}
        G_k = \frac{S_k}{S_k + \sigma_k^2}.
        \label{eq:scalar-wiener-gain}
    \end{equation}
\end{theorem}

\nomenclature[EAF]{\(S_k\)}{signal power of coefficient $k$}
\nomenclature[EAG]{\(\widehat{X}_k\)}{linear estimate of clean coefficient $k$}
\nomenclature[EAH]{\(G_k\)}{scalar Wiener gain}
\nomenclature[EAI]{\(\mathfrak{R}_k(G_k)\)}{mean-squared error of coefficient $k$ at gain $G_k$}

\begin{proof}
    The estimation error is
    \begin{equation*}
        G_k Y_k - X_k = (G_k - 1) X_k + G_k N_k.
    \end{equation*}
    Using the mean-squared error criterion, the squared magnitude of this estimation error is taken.
    Because $G_k$ is real, expanding it gives
    \begin{align*}
        |G_k Y_k - X_k|^2
        &= \bigl((G_k - 1) X_k + G_k N_k\bigr)
           \bigl((G_k - 1) \overline{X_k} + G_k \overline{N_k}\bigr)\\
        &= (G_k - 1)^2 |X_k|^2 + G_k^2 |N_k|^2\\
        &\quad+ G_k (G_k - 1) X_k \overline{N_k}
           + G_k (G_k - 1) N_k \overline{X_k}.
    \end{align*}
    Taking expectations therefore yields
    \begin{align*}
        \E[|G_k Y_k - X_k|^2]
        &= (G_k - 1)^2 \E[|X_k|^2]
           +G_k^2 \E[|N_k|^2]\\
        &\quad+ G_k (G_k - 1) \E[X_k \overline{N_k}]\\
        &\quad+ G_k (G_k - 1) \E[N_k \overline{X_k}].
    \end{align*}
    The signal and noise being uncorrelated means $\E[X_k \overline{N_k}] = 0$.
    The other cross moment is its complex conjugate,
    \begin{equation*}
        \E[N_k \overline{X_k}]
        = \overline{\E[X_k \overline{N_k}]} = 0,
    \end{equation*}
    so both cross terms vanish.
    Substituting $\E[|X_k|^2] = S_k$ and $\E[|N_k|^2] = \sigma_k^2$ gives \Cref{eq:scalar-risk}.
    Differentiation then yields
    \begin{equation*}
        \mathfrak{R}_k'(G_k)
        = -2(1 - G_k) S_k + 2 G_k \sigma_k^2.
    \end{equation*}
    Setting this derivative to zero gives \Cref{eq:scalar-wiener-gain}.
    The second derivative is $2(S_k + \sigma_k^2) \geq 0$, with a unique minimizer unless both powers vanish.
\end{proof}

The two terms in \Cref{eq:scalar-risk} expose its bias--variance tradeoff.
Increasing $G_k$ reduces the signal attenuation $(1 - G_k)^2 S_k$ but transmits more noise through $G_k^2 \sigma_k^2$.
Consequently, the oracle gain tends to one as $S_k /\sigma_k^2 \to \infty$ and to zero as $S_k /\sigma_k^2 \to 0$.

\begin{definition}[Practical Wiener gain]
    \label{def:practical-wiener-gain}
    Given a pilot signal-power estimate $\widehat{S}_k$ and a variance scale $s_v \geq 0$ such that $\widehat{S}_k + s_v \sigma_k^2 > 0$, define the variance-scaled gain by
    \begin{equation}
        \widehat{G}_k
        = \frac{\widehat{S}_k}
               {\widehat{S}_k + s_v \sigma_k^2}.
        \label{eq:wiener-gain}
    \end{equation}
    The gain lies in $[0, 1]$.
    The unscaled oracle is recovered when $s_v = 1$ and $\widehat{S}_k = S_k$.
\end{definition}

\nomenclature[EAJ]{\(\widehat{S}_k\)}{pilot signal power estimate}
\nomenclature[EAK]{\(\widehat{G}_k\)}{variance-scaled Wiener gain}
\nomenclature[EAL]{\(s_v\)}{variance scale}

\begin{definition}[Diagonal Wiener group filter]
    \label{def:diagonal-wiener-group-filter}
    For one group, define the complete coefficientwise operation by
    \begin{equation}
        \CC_g^{\WIE}
        = D_g^{\WIE} \YY_g,
        \qquad
        D_g^{\WIE} = \operatorname{diag}(\widehat{G}_{g, 1}, \ldots, \widehat{G}_{g, d}).
        \label{eq:wiener-diagonal-filter}
    \end{equation}
    Notice that the pilot $\boldsymbol{\Gamma}_g$ determines $D_g^{\WIE}$, whereas $D_g^{\WIE}$ multiplies the observation $\YY_g$.
\end{definition}

\nomenclature[EAM]{\(D_g^{\WIE}\)}{diagonal Wiener group filter}

\subsection{Estimating Signal Power from the Pilot}
The pilot is not noise-free, so its raw squared magnitude includes a noise floor.

\begin{lemma}[Pilot-power correction]
    \label{lem:pilot-power-correction}
    If $\Gamma_k = X_k + E_k$, where $E_k$ is zero mean, uncorrelated with $X_k$, and has variance $\nu_k^2$, then
    \begin{equation}
        \E[|\Gamma_k|^2] = S_k + \nu_k^2.
    \end{equation}
\end{lemma}

\nomenclature[EAN]{\(E_k\)}{residual pilot error}
\nomenclature[EAO]{\(\nu_k^2\)}{residual pilot-error variance of coefficient $k$}

\begin{proof}
    Since $\Gamma_k = X_k + E_k$, its squared magnitude is
    \begin{align*}
        |\Gamma_k|^2
        &= (X_k + E_k) (\overline{X_k} + \overline{E_k})\\
        &= |X_k|^2 + |E_k|^2
           + X_k \overline{E_k} + E_k \overline{X_k}.
    \end{align*}
    Taking expectations gives
    \begin{align*}
        \E[|\Gamma_k|^2]
        &= \E[|X_k|^2] + \E[|E_k|^2]\\
        &\quad + \E[X_k \overline{E_k}]
           + \E[E_k \overline{X_k}].
    \end{align*}
    Because $E_k$ is zero mean and uncorrelated with $X_k$,
    \begin{equation*}
        \E[X_k \overline{E_k}] = 0.
    \end{equation*}
    The other cross moment is its complex conjugate,
    \begin{equation*}
        \E[E_k \overline{X_k}]
        = \overline{\E[X_k \overline{E_k}]} = 0.
    \end{equation*}
    Finally, $\E[|X_k|^2] = S_k$ and, because $E_k$ is zero mean with variance
    $\nu_k^2$, $\E[|E_k|^2] = \nu_k^2$.
    Substitution yields $\E[|\Gamma_k|^2] = S_k + \nu_k^2$.
\end{proof}

\begin{proposition}[Bias before nonnegative truncation]
    \label{prop:pilot-power-bias}
    Under \Cref{lem:pilot-power-correction}, let $b_k \geq 0$ be deterministic and define the untruncated power estimate
    \begin{equation}
        \widetilde{S}_k(b_k) = |\Gamma_k|^2 - b_k.
    \end{equation}
    Its bias is
    \begin{equation}
        \E[\widetilde{S}_k(b_k)] - S_k = \nu_k^2 - b_k.
        \label{eq:untruncated-pilot-power-bias}
    \end{equation}
    Consequently, $|\Gamma_k|^2$ has upward bias $\nu_k^2$,
    $|\Gamma_k|^2 - \sigma_k^2$ is unbiased when $\nu_k^2 = \sigma_k^2$, and $|\Gamma_k|^2 - s_v \sigma_k^2$ is unbiased when $\nu_k^2 = s_v \sigma_k^2$.
\end{proposition}

\nomenclature[EAP]{\(\widetilde{S}_k\)}{untruncated power estimate}
\nomenclature[EAQ]{\(b_k\)}{signal-power bias correction for coefficient $k$}

\begin{proof}
    \Cref{lem:pilot-power-correction} gives
    \begin{equation*}
        \E[\widetilde{S}_k(b_k)]
        = \E[|\Gamma_k|^2] - b_k
        = S_k + \nu_k^2 - b_k.
    \end{equation*}
    Subtracting $S_k$ proves \Cref{eq:untruncated-pilot-power-bias}.
    The three cases follow by setting $b_k = 0$, $b_k = \sigma_k^2$, and
    $b_k = s_v \sigma_k^2$, respectively.
\end{proof}

\begin{definition}[Pilot-power modes]
    \label{def:pilot-power-modes}
    The classic, noise-floor-corrected, and variance-scaled pilot-power estimates are
    \begin{align}
        \widehat{S}_k^{\mathrm{classic}}
            &= |\Gamma_k|^2
            && \text{(classic)},\label{eq:power-classic}\\
        \widehat{S}_k^{\mathrm{nf}}
            &= [|\Gamma_k|^2 - \sigma_k^2]_+
            && \text{(noise-floor correction)},\label{eq:power-noise-floor}\\
        \widehat{S}_k^{\mathrm{vs}}
            &= [|\Gamma_k|^2 - s_v \sigma_k^2]_+
            && \text{(variance-scaled correction)}.\label{eq:power-variance-scaled}
    \end{align}
    Writing $a = 0$, $1$, or $s_v$ for the classic, noise-floor-corrected, or variance-scaled mode, respectively, all three estimates can be summarized as
    \begin{equation}
        \widehat{S}_k = [|\Gamma_k|^2 - a \sigma_k^2]_+.
        \label{eq:pilot-power-family}
    \end{equation}
\end{definition}

\begin{remark}[Variance proxy and truncation]
The true pilot-error variance $\nu_k^2$ is generally unavailable, so \Cref{def:pilot-power-modes} uses the second-stage coefficient variance $\sigma_k^2$ as a proxy.
The identities $\nu_k^2 = \sigma_k^2$ and $\nu_k^2 = s_v \sigma_k^2$ are therefore models for residual pilot error, not consequences of the Wiener derivation.
The positive part enforces $\widehat{S}_k \geq 0$, but the resulting truncation generally changes the bias stated in \Cref{prop:pilot-power-bias}.
\end{remark}

\begin{proposition}[Corrected pilot-power Wiener gain]
    \label{prop:corrected-wiener-dead-zone}
    For a corrected pilot-power mode with $a > 0$ and $s_v \sigma_k^2 > 0$, substituting \Cref{eq:pilot-power-family} into the practical gain in \Cref{eq:wiener-gain} gives
    \begin{equation}
        \widehat{G}_k =
        \begin{cases}
            0,
                & |\Gamma_k|^2 \leq a \sigma_k^2,\\[3pt]
            \displaystyle
            \frac{|\Gamma_k|^2 - a \sigma_k^2}
                 {|\Gamma_k|^2 - a \sigma_k^2 + s_v \sigma_k^2},
                & |\Gamma_k|^2 > a \sigma_k^2.
        \end{cases}
        \label{eq:corrected-wiener-gain}
    \end{equation}
\end{proposition}

\begin{proof}
    If $|\Gamma_k|^2 \leq a \sigma_k^2$, then \Cref{eq:pilot-power-family} gives $\widehat{S}_k = 0$.
    Since $s_v \sigma_k^2 > 0$, \Cref{eq:wiener-gain} then gives $\widehat{G}_k = 0$.
    If $|\Gamma_k|^2 > a \sigma_k^2$, then $\widehat{S}_k = |\Gamma_k|^2 - a \sigma_k^2$.
    Substitution into \Cref{eq:wiener-gain} gives the second branch of \Cref{eq:corrected-wiener-gain}.
\end{proof}

The corrected modes introduce a pilot-dependent dead zone followed by continuous shrinkage.
The classic mode has no positive noise-floor dead zone and can assign a nontrivial gain to residual pilot noise.
Conversely, an inaccurate pilot can place a weak true coefficient inside the corrected dead zone, after which the second stage cannot recover it from $Y_k$.
The power mode, therefore, controls a tradeoff between residual-noise transmission and pilot-induced attenuation.
\section{Aggregation Weights}
\label{sec:aggregation-weights}
Collaborative filtering can produce several estimates for the same voxel, since patches from neighboring reference groups often overlap.
During aggregation, these estimates should not all contribute equally.
Predictions from a less reliable group should generally have less influence than those from a more reliable one.

The reliability weight can be defined at either the group or patch level.
With group-level aggregation, every patch in a matched group receives the same weight.
With patch-level aggregation, individual patches can receive different weights based on their predicted marginal risk.

The same basic idea is used in both filtering stages, although the quantities used to estimate reliability differ.
For hard thresholding, reliability is derived from the hard mask or soft attenuation and, when the variance model is used, from the coefficient noise variances as well.
The Wiener stage uses its gains and may additionally use coefficient noise variances and pilot-based signal power estimates.

These weights should therefore be interpreted as local approximations to estimation precision.
They do not assume that errors from overlapping groups are statistically independent.

Group-level weighted overlap-add and the classic coefficient-domain retained-count and gain-energy weights are inherited from \Ac{bm3d}/\Ac{bm4d}~\cite{dabov2007image,maggioni2012nonlocal}.
Coefficient-domain patch weighting based on propagated variance follows generalized collaborative filtering~\cite{makinen2020collaborative}.
Bias-aware risk and overlap-aware covariance further extend these weighting constructions in \Ac{bmnd}.

\begin{definition}[Weighted overlap-add estimate]
    \label{def:weighted-overlap-add}
    Let $p_{g, j}(u)$ denote the value assigned to voxel $u$ by synthesized patch $j$ of group $g$, and let $\omega_{g, j}(u) \geq 0$ be its aggregation-window value.
    Let $w_{g, j} > 0$ be the reliability weight assigned to that patch, where each inner sum below includes only synthesized patches whose support contains $u$.
    Assume that every reconstructed voxel has positive total aggregation weight:
    \begin{equation}
        \sum_g \sum_j w_{g,j} \omega_{g,j}(u) > 0.
    \end{equation}
    The aggregated estimate is
    \begin{equation}
        \widehat{x}(u)
        = \frac{
            \displaystyle \sum_g \sum_j w_{g, j}
            \omega_{g, j}(u) p_{g, j}(u)
          }{
            \displaystyle \sum_g \sum_j w_{g, j} \omega_{g, j}(u)
          }.
        \label{eq:weighted-overlap-add}
    \end{equation}
\end{definition}

\nomenclature[FAA]{\(p_{g,j}(u)\)}{value assigned to $u$ by synthesized patch $j$ of group $g$}
\nomenclature[FAB]{\(\omega_{g,j}(u)\)}{aggregation-window value of patch $j$ of group $g$ at $u$}
\nomenclature[FAC]{\(w_{g,j}\)}{reliability weight of patch $j$ of group $g$}

\begin{definition}[Aggregation-weight scope]
    \label{def:aggregation-weight-scope}
    In group scope, one positive predicted risk $\widehat{R}_g$ is shared by every synthesized patch in group $g$, so
    \begin{equation}
        w_{g, j} = w_g
        = \frac{1}{\widehat{R}_g}.
        \label{eq:group-scope-weight}
    \end{equation}
    In patch scope, each synthesized patch has a positive marginal predicted risk $\widehat{R}_{g, j}$ and weight
    \begin{equation}
        w_{g, j}
        = \frac{1}{\widehat{R}_{g, j}}.
        \label{eq:patch-scope-weight}
    \end{equation}
\end{definition}

\nomenclature[FAD]{\(\widehat{R}_g\)}{predicted risk of group $g$}

For independent unbiased estimates of one scalar, the standard minimum-variance rule assigns weights proportional to inverse risk~\cite{graybill1959combining}.
Foundational \Ac{bm3d} and \Ac{bm4d} adopt this inverse-risk principle at group level, using the retained coefficient count in the hard-thresholding stage and Wiener gain energy in the second stage as coefficient-domain residual-noise surrogates in place of exact group risks~\cite{dabov2007image,maggioni2012nonlocal}.
Generalized collaborative filtering instead propagates coefficient variance through inverse nonlocal synthesis and assigns the resulting marginal risk to each synthesized patch~\cite{makinen2020collaborative}.
\Ac{bmnd} supports both scopes and extends their risk surrogates as described below.

\begin{remark}[Risk domain]
The predicted group risk $\widehat{R}_g$ or marginal patch risk $\widehat{R}_{g, j}$ may be evaluated either in the coefficient domain or in the windowed synthesized domain.
The corresponding risk constructions are introduced in the following subsections.
Within one aggregation mode, the same risk domain is used consistently for every group.
\end{remark}

This inverse-reliability rule is a local principle rather than an exact global optimality result.
Different groups cover different supports, overlapping groups share noisy voxels, and transform shrinkage generally introduces bias.
The fuller models below make $\widehat{R}_g$ or $\widehat{R}_{g, j}$ a more informative surrogate but do not remove these dependencies.

\subsection{Hard-Thresholding Weighting Models}
The first filtering stage uses either a hard-thresholding mask or, optionally, soft-thresholding attenuation.
These coefficient weights define the classic filter-energy surrogate and, together with the coefficient noise variances, the variance surrogate.

\begin{definition}[Coefficient-domain hard-thresholding weighting models]
    \label{def:hard-thresholding-weighting-models}
    Suppressing the group index $g$, let $a_k$ denote the coefficient mask from \Cref{def:hard-threshold-group-filter} for hard thresholding or the selected attenuation for optional soft thresholding.
    For group scope, the coefficient-domain models are
    \begin{align}
        \widehat{R}_g^{\mathrm{classic}}
        &= \sum_k |a_k|^2,\\
        \widehat{R}_g^{\mathrm{variance}}
        &= \sum_k |a_k|^2 \sigma_k^2.
    \end{align}
    For a hard mask, \Cref{eq:number-retained-coefficients} gives
    \begin{equation}
        \widehat{R}_g^{\mathrm{classic}}
        = \sum_k a_k
        = N_g^{\HT}.
    \end{equation}
    The variance model reduces to the retained variance $\sum_k a_k \sigma_k^2$.
    Under white noise, the latter is $\sigma^2 \sum_k a_k$, and the common factor has no effect after overlap-add normalization.\\
    For patch scope, write the coefficient index as $k = (\ell, q)$, where $\ell$ indexes the nonlocal transform plane and $q$ indexes the spatial-transform coefficient.
    Let $V_g$ denote the inverse nonlocal transform for group $g$, and define
    \begin{align}
        r_{g, \ell}^{\HT, \mathrm{classic}}
        &= \sum_q |a_{\ell, q}|^2,\\
        r_{g, \ell}^{\HT, \mathrm{variance}}
        &= \sum_q |a_{\ell, q}|^2 \sigma_{\ell, q}^2.
    \end{align}
    The marginal surrogate and weight of synthesized patch $j$ are
    \begin{align}
        \widehat{R}_{g, j}^{\HT, \mathrm{model}}
        &= \sum_{\ell} |(V_g)_{j, \ell}|^2
           r_{g, \ell}^{\HT, \mathrm{model}},\\
        w_{g, j}^{\HT, \mathrm{model}}
        &= \frac{1}{\widehat{R}_{g, j}^{\HT, \mathrm{model}}},
        \qquad
        \mathrm{model} \in \{\mathrm{classic}, \mathrm{variance}\}.
    \end{align}
    The weight is defined whenever the selected marginal surrogate is strictly positive.
\end{definition}

\nomenclature[FAE]{\(a_k\)}{hard-thresholding mask or soft-thresholding attenuation of coefficient $k$}
\nomenclature[FAF]{\(V_g\)}{inverse nonlocal transform}

\begin{remark}[No hard-thresholding risk mode]
\Ac{bmnd} does not use a hard-thresholding risk model because no sufficiently reliable independent pilot is available before the first stage to estimate rejected signal power.
\end{remark}

\subsection{Wiener Weighting Models}
\label{subsec:wiener-weighting-models}
The Wiener gains already describe how strongly each noisy coefficient contributes to the filtered group.
Together with the propagated noise variances and selected signal-power estimates, these gains determine the predicted error of the filtered group.

\begin{proposition}[Coefficient-domain Wiener risk]
    \label{prop:diagonal-wiener-risk}
    Let $G_k \in \mathbb{R}$ denote the fixed diagonal Wiener gain for coefficient $k$, and assume that signal and noise are uncorrelated.
    Then the total coefficient-domain mean-squared error is
    \begin{equation}
        R_g^\WIE
        = \sum_k \mathfrak{R}_k(G_k)
        = \sum_k \left(G_k^2 \sigma_k^2 + (1 - G_k)^2 S_k\right).
        \label{eq:diagonal-wiener-risk}
    \end{equation}
    Mutual correlation among the signal coefficients or among the noise coefficients does not change this coefficient-domain diagonal-filter risk.
\end{proposition}

\nomenclature[FAG]{\(R_g^\WIE\)}{coefficient-domain Wiener risk of group $g$}

\begin{proof}
    For coefficient $k$, \Cref{eq:scalar-risk} gives
    \begin{equation*}
        \E[|G_k Y_k - X_k|^2]
        = G_k^2 \sigma_k^2 + (1 - G_k)^2 S_k.
    \end{equation*}
    The squared Euclidean norm is the sum of the coefficientwise squared magnitudes.
    Taking expectations and summing over $k$ proves \Cref{eq:diagonal-wiener-risk}.
    Thus correlations between distinct coefficients do not affect the risk, which depends only on the marginal signal and noise powers.
\end{proof}

Taking $\widehat{R}_g$ in \Cref{eq:group-scope-weight} to be the plug-in coefficient-domain risk from \Cref{prop:diagonal-wiener-risk} yields the full Wiener risk weight below.
Starting from the classic gain-energy weight, the variance form incorporates coefficient-dependent noise variances, and the full risk form additionally accounts for rejected signal energy.

\begin{definition}[Coefficient-domain Wiener weighting models]
    \label{def:wiener-weighting-models}
    For group scope, the coefficient-domain Wiener models are
    \begin{align}
        w_g^{\WIE, \mathrm{classic}}
        &= \frac{1}{\sum_k \widehat{G}_k^2},
        \label{eq:weight-classic}\\
        w_g^{\WIE, \mathrm{variance}}
        &= \frac{1}{\sum_k \widehat{G}_k^2 \sigma_k^2},
        \label{eq:weight-variance}\\
        w_g^{\WIE, \mathrm{risk}}
        &= \frac{1}{
            \sum_k \left(\widehat{G}_k^2 \sigma_k^2 + (1 - \widehat{G}_k)^2 \widehat{S}_k\right)
          }.
        \label{eq:weight-risk}
    \end{align}
    For patch scope, write the coefficient index as $k = (\ell, q)$, where $\ell$ indexes the nonlocal transform plane and $q$ indexes the spatial-transform coefficient.
    Let $V_g$ denote the inverse nonlocal transform for group $g$, and define
    \begin{align}
        r_{g, \ell}^{\WIE, \mathrm{classic}}
        &= \sum_q \widehat{G}_{\ell, q}^2,\\
        r_{g, \ell}^{\WIE, \mathrm{variance}}
        &= \sum_q \widehat{G}_{\ell, q}^2 \sigma_{\ell, q}^2,\\
        r_{g, \ell}^{\WIE, \mathrm{risk}}
        &= \sum_q \left(
            \widehat{G}_{\ell, q}^2 \sigma_{\ell, q}^2
            + (1 - \widehat{G}_{\ell, q})^2 \widehat{S}_{\ell, q}
          \right).
    \end{align}
    The marginal surrogate and weight of synthesized patch $j$ are
    \begin{align}
        \widehat{R}_{g, j}^{\WIE, \mathrm{model}}
        &= \sum_{\ell} |(V_g)_{j, \ell}|^2
           r_{g, \ell}^{\WIE, \mathrm{model}},\\
        w_{g, j}^{\WIE, \mathrm{model}}
        &= \frac{1}{\widehat{R}_{g, j}^{\WIE, \mathrm{model}}},
        \qquad
        \mathrm{model} \in \{\mathrm{classic}, \mathrm{variance}, \mathrm{risk}\}.
    \end{align}
    Each weight is defined whenever its selected marginal surrogate is strictly positive.
\end{definition}

\subsection{Windowed Synthesized-Domain Weighting Models}
\label{subsec:windowed-weighting-models}
The coefficient-domain modes measure error before inverse spatial synthesis and aggregation windowing.
When complete inverse synthesis is unitary and the aggregation window is constant, this metric agrees with synthesized-domain error up to a common factor.
Otherwise, inverse synthesis or a nonconstant window changes the metric in which the reconstructed error is measured.

\begin{definition}[Windowed synthesis metrics]
    \label{def:windowed-synthesis-metrics}
    Let $U_g$ be the complete inverse transform of group $g$.
    Let $W_p$ be the diagonal operator whose entries are the spatial aggregation-window values of one synthesized patch.
    The corresponding group-window operator is
    \begin{equation}
        W_g = \operatorname{blockdiag}(W_p, \ldots, W_p),
    \end{equation}
    with one block for each patch in group $g$.
    Thus, $W_g$ applies the spatial window independently to every synthesized patch.
    Define the linear operator $L_g$ that maps group-transform coefficients to the windowed synthesized group as
    \begin{equation}
        L_g = W_g U_g,
        \qquad
        H_g = L_g^{\H} L_g
            = U_g^{\H} W_g^{\H} W_g U_g.
        \label{eq:group-synthesis-metric}
    \end{equation}
    Let $Q_{g,j}$ extract synthesized patch $j$ from the stacked group.
    The corresponding patch operators are
    \begin{equation}
        L_{g, j} = W_p Q_{g, j} U_g,
        \qquad
        H_{g, j}
        = L_{g, j}^{\H} L_{g, j}
        = U_g^{\H} Q_{g, j}^{\H} W_p^{\H} W_p Q_{g, j} U_g.
        \label{eq:patch-synthesis-metric}
    \end{equation}
    Because the patch extractors select the disjoint blocks of the stacked group, $H_g = \sum_j H_{g, j}$.
    Write $\mathcal{L} = L_g$ and $\mathcal{H} = H_g$ in group scope, or $\mathcal{L} = L_{g, j}$ and $\mathcal{H} = H_{g, j}$ in patch scope.
\end{definition}

\nomenclature[FAH]{\(U_g\)}{complete inverse transform of group $g$}
\nomenclature[FAI]{\(W_g\)}{aggregation-window operator for group $g$}
\nomenclature[FAJ]{\(L_g\)}{windowed synthesis operator for group $g$}
\nomenclature[FAK]{\(H_g\)}{windowed synthesis metric of group $g$}
\nomenclature[FAL]{\(\mathcal{L}\)}{windowed synthesis operator for the selected group or patch scope}
\nomenclature[FAM]{\(\mathcal{H}\)}{windowed synthesis metric for the selected group or patch scope}

\begin{proposition}[Windowed transmitted-noise energy]
    \label{prop:windowed-transmitted-noise}
    Let $F_g$ be a diagonal coefficient filter, and suppose that $\E[\NN_g \NN_g^{\H}] = \Sigma_{N, g}$.
    The transmitted-noise energy in the selected synthesized-domain metric is
    \begin{equation}
        \E\!\left[\|\mathcal{L} F_g \NN_g\|_2^2\right]
        = \operatorname{tr}\!\left(
            \mathcal{H} F_g \Sigma_{N, g} F_g^{\H}
          \right).
        \label{eq:windowed-transmitted-noise}
    \end{equation}
\end{proposition}

\nomenclature[FAN]{\(F_g\)}{diagonal coefficient filter}

\begin{proof}
    The synthesized noise contribution is $\mathcal{L} F_g \NN_g$.
    Writing its squared norm as a trace gives
    \begin{align*}
        \|\mathcal{L} F_g \NN_g\|_2^2
        &= (\mathcal{L} F_g \NN_g)^{\H}
           (\mathcal{L} F_g \NN_g)\\
        &= \operatorname{tr}\!\left(
            \mathcal{L} F_g \NN_g \NN_g^{\H}
            F_g^{\H} \mathcal{L}^{\H}
          \right).
    \end{align*}
    Taking expectations and using $\E[\NN_g \NN_g^{\H}] = \Sigma_{N, g}$ yields
    \begin{equation*}
        \E\!\left[\|\mathcal{L} F_g \NN_g\|_2^2\right]
        = \operatorname{tr}\!\left(
            \mathcal{L} F_g \Sigma_{N, g}
            F_g^{\H} \mathcal{L}^{\H}
          \right).
    \end{equation*}
    By cyclic invariance of the trace,
    \begin{align*}
        \operatorname{tr}\!\left(
            \mathcal{L} F_g \Sigma_{N, g}
            F_g^{\H} \mathcal{L}^{\H}
          \right)
        &= \operatorname{tr}\!\left(
            \mathcal{L}^{\H} \mathcal{L}
            F_g \Sigma_{N, g} F_g^{\H}
          \right)\\
        &= \operatorname{tr}\!\left(
            \mathcal{H} F_g \Sigma_{N, g} F_g^{\H}
          \right),
    \end{align*}
    where $\mathcal{H} = \mathcal{L}^{\H} \mathcal{L}$ by \Cref{def:windowed-synthesis-metrics}.
    This is \Cref{eq:windowed-transmitted-noise}.
\end{proof}

\begin{proposition}[Windowed synthesized-domain Wiener risk]
    \label{prop:windowed-wiener-risk}
    Let $D_g = \operatorname{diag}(G_1, \ldots, G_d)$ be the Wiener filtering operator.
    If
    \begin{equation}
        \E[\XX_g \XX_g^{\H}] = \Sigma_{X, g},
        \qquad
        \E[\NN_g \NN_g^{\H}] = \Sigma_{N, g},
        \qquad
        \E[\XX_g \NN_g^{\H}] = 0,
    \end{equation}
    then the Wiener risk in either selected synthesized-domain metric is
    \begin{equation}
        R^{\mathrm{win}}(\mathcal{H})
        = \operatorname{tr}(\mathcal{H} D_g \Sigma_{N, g} D_g^{\H})
          + \operatorname{tr}\!\left(
            \mathcal{H} (D_g - I) \Sigma_{X, g} (D_g - I)^{\H}
          \right).
        \label{eq:windowed-wiener-risk}
    \end{equation}
\end{proposition}

\nomenclature[FAO]{\(R^{\mathrm{win}}(\mathcal{H})\)}{windowed synthesized-domain Wiener risk in metric $\mathcal{H}$}

\begin{proof}
    The coefficient error is
    \begin{equation*}
        \mathbf{e}_g
        = D_g \YY_g - \XX_g
        = (D_g - I) \XX_g + D_g \NN_g.
    \end{equation*}
    Its outer product expands as
    \begin{align*}
        \mathbf{e}_g \mathbf{e}_g^{\H}
        &= (D_g - I) \XX_g \XX_g^{\H} (D_g - I)^{\H}
           + D_g \NN_g \NN_g^{\H} D_g^{\H}\\
        &\quad
           + (D_g - I) \XX_g \NN_g^{\H} D_g^{\H}
           + D_g \NN_g \XX_g^{\H} (D_g - I)^{\H}.
    \end{align*}
    The assumption $\E[\XX_g \NN_g^{\H}] = 0$ also implies $\E[\NN_g \XX_g^{\H}] = 0$ by conjugate transposition.
    Consequently, both cross terms vanish after taking expectations, and
    \begin{equation*}
        \E[\mathbf{e}_g \mathbf{e}_g^{\H}]
        = (D_g - I) \Sigma_{X, g} (D_g - I)^{\H}
          + D_g \Sigma_{N, g} D_g^{\H}.
    \end{equation*}
    The synthesized-domain risk is
    \begin{align*}
        \E[\|\mathcal{L} \mathbf{e}_g\|_2^2]
        &= \operatorname{tr}\!\left(
            \mathcal{L}
            \E[\mathbf{e}_g \mathbf{e}_g^{\H}]
            \mathcal{L}^{\H}
          \right)\\
        &= \operatorname{tr}\!\left(
            \mathcal{H}
            \E[\mathbf{e}_g \mathbf{e}_g^{\H}]
          \right),
    \end{align*}
    where the second equality uses cyclic invariance of the trace and $\mathcal{H} = \mathcal{L}^{\H} \mathcal{L}$.
    Substituting $\E[\mathbf{e}_g \mathbf{e}_g^{\H}]$ and using linearity of the trace gives
    \begin{align*}
        \E[\|\mathcal{L} \mathbf{e}_g\|_2^2]
        &= \operatorname{tr}\!\left(
            \mathcal{H} (D_g - I)
            \Sigma_{X, g} (D_g - I)^{\H}
          \right)\\
        &\quad
           + \operatorname{tr}\!\left(
            \mathcal{H} D_g \Sigma_{N, g} D_g^{\H}
          \right),
    \end{align*}
    which is \Cref{eq:windowed-wiener-risk}.
\end{proof}

\begin{definition}[Windowed synthesized-domain weighting models]
    \label{def:windowed-weighting-models}
    Let $F_g = A_g$ for hard thresholding and $F_g = D_g^{\WIE}$ for Wiener filtering.
    The classic and variance models in either stage are
    \begin{align}
        \widehat{R}^{\mathrm{classic}}(\mathcal{H})
        &= \operatorname{tr}\!\left(
            \mathcal{H} F_g F_g^{\H}
          \right),\\
        \widehat{R}^{\mathrm{variance}}(\mathcal{H})
        &= \operatorname{tr}\!\left(
            \mathcal{H} F_g
            \widehat{\Sigma}_{N, g} F_g^{\H}
          \right).
        \label{eq:windowed-shared-weight-models}
    \end{align}
    The Wiener stage additionally admits the risk model
    \begin{align}
        \widehat{R}^{\mathrm{risk}}(\mathcal{H})
        &= \operatorname{tr}\!\left(
            \mathcal{H} D_g^{\WIE} \widehat{\Sigma}_{N, g} (D_g^{\WIE})^{\H}
          \right)
          + \operatorname{tr}\!\left(
            \mathcal{H} (D_g^{\WIE} - I)
            \operatorname{diag}(\widehat{S}_k) (D_g^{\WIE} - I)^{\H}
          \right).
        \label{eq:windowed-weight-models}
    \end{align}
    Here $\widehat{\Sigma}_{N, g}$ denotes the selected diagonal or covariance-aware noise model, and $\widehat{S}_k$ denotes the selected Wiener signal-power estimate.
    Each selected surrogate is converted to a weight by \Cref{eq:group-scope-weight,eq:patch-scope-weight}.
\end{definition}

\begin{remark}[Diagonal signal model]
\label{rem:diagonal-signal-model}
The practical Wiener risk approximates the transformed clean-group second moment by $\operatorname{diag}(\widehat{S}_k)$.
\Ac{dct} and Haar transforms often concentrate signal energy and can reduce cross-coefficient dependence over an ensemble of clean groups, which motivates this approximation.
Transform orthogonality alone does not imply vanishing cross-moments, however, so the approximation is generally not exact.
Unlike the coefficient-domain risk in \Cref{prop:diagonal-wiener-risk}, off-diagonal signal moments can affect the risk when the selected metric or filtering operator is non-diagonal.
\end{remark}

When $U_g$ is unitary and $W_g = I$, then $H_g = I$, so the first two models reduce to their coefficient-domain forms.
Under the same conditions, \Cref{eq:windowed-wiener-risk} reduces to \Cref{eq:diagonal-wiener-risk} when the diagonals of $\Sigma_{X,g}$ and $\Sigma_{N,g}$ are $S_k$ and $\sigma_k^2$.

\section{Aggregation-Aware Mass Conservation}
\label{sec:mass-conservation}
The aggregation weights in \Cref{sec:aggregation-weights} control how overlapping estimates are combined, but they do not ensure that the reconstructed array preserves the observed total intensity.
Under the scaled-Poisson model, this total is proportional to the number of detected events and can be altered by shrinkage.
The contribution of an individual group to the reconstructed total depends on its reliability weights, aggregation window, and pointwise overlap denominator.
\Ac{bmnd} introduces an optional aggregation-aware constraint under which every filtered group reproduces the normalized overlap-add contribution of its noisy source group.
The constraint is enforced by a transform-domain correction that adds the same offset to every entry of a synthesized group.

The stage index is suppressed below, with $\CC_g$ denoting either $\CC_g^{\HT}$ or $\CC_g^{\WIE}$.
\begin{definition}[Normalized aggregation functional]
    \label{def:normalized-aggregation-functional}
    Following \Cref{def:weighted-overlap-add}, denote the overlap-add denominator by
    \begin{equation}
        \eta(u)
        = \sum_g \sum_j w_{g,j} \omega_{g,j}(u),
        \label{eq:dc-aggregation-denominator}
    \end{equation}
    and define the normalized patch influences at locations with $\eta(u) > 0$ by
    \begin{equation}
        \zeta_{g,j}(u)
        = \frac{w_{g,j} \omega_{g,j}(u)}{\eta(u)}.
        \label{eq:dc-normalized-influence}
    \end{equation}
    Let $\boldsymbol{\zeta}_g$ contain these influences in the stacking order of the synthesized group.
    Define the aggregation-aware coefficient-space mass functional by
    \begin{equation}
        \hh_g
        = U_g^\top \boldsymbol{\zeta}_g.
        \label{eq:dc-aggregation-functional}
    \end{equation}
\end{definition}

The contribution of group $g$ to the reconstructed total is therefore
\begin{equation}
    \boldsymbol{\zeta}_g^\top U_g \CC_g
    = \hh_g^\top \CC_g.
    \label{eq:dc-group-contribution}
\end{equation}

\nomenclature[GAA]{\(\eta(u)\)}{total aggregation weight at sampling location $u$}
\nomenclature[GAB]{\(\zeta_{g,j}(u)\)}{normalized influence of patch $j$ of group $g$ at $u$}
\nomenclature[GAC]{\(\boldsymbol{\zeta}_g\)}{stacked normalized aggregation influences for group $g$}
\nomenclature[GAD]{\(\hh_g\)}{aggregation-aware coefficient-space mass functional}

\begin{definition}[Aggregation-aware constant-synthesis correction]
    \label{def:aggregation-aware-dc-correction}
    Using the observation group $\yy_g$, define the residual between its normalized contribution and that of the filtered group by
    \begin{equation}
        d_g
        = \boldsymbol{\zeta}_g^\top \yy_g
          - \hh_g^\top \CC_g.
        \label{eq:dc-mass-residual}
    \end{equation}
    Define the coefficient direction that synthesizes an all-ones group by
    \begin{equation}
        \mathbf{m}_g
        = U_g^{-1} \mathbf{1},
        \qquad
        U_g \mathbf{m}_g = \mathbf{1}.
        \label{eq:dc-constant-synthesis-direction}
    \end{equation}
    For a group with nonzero normalized influence, define its corrected coefficients by
    \begin{equation}
        \CC_g^{\mathrm{mc}}
        = \CC_g
          + \frac{d_g}{
              \hh_g^\top \mathbf{m}_g
            }
            \mathbf{m}_g
        = \CC_g
          + \frac{d_g}{
              \boldsymbol{\zeta}_g^\top \mathbf{1}
            }
            \mathbf{m}_g.
        \label{eq:aggregation-aware-dc-correction}
    \end{equation}
\end{definition}

\begin{remark}[Relation to DC coefficients]
    The correction is constant along both the patch axes and the nonlocal group axis.
    For a separable orthonormal \ac{dct} applied along all these axes, it modifies only the all-axis DC coefficient.
    For other invertible transforms, including biorthogonal transforms, a constant group may require several coefficients.
    The direction $\mathbf{m}_g = U_g^{-1} \mathbf{1}$ specifies the same synthesized correction independently of this coefficient representation.
\end{remark}

\nomenclature[GAE]{\(d_g\)}{missing normalized mass contribution of group $g$}
\nomenclature[GAF]{\(\mathbf{m}_g\)}{coefficient direction that synthesizes a constant group}

\begin{proposition}[Exactness of the constant-synthesis correction]
    \label{prop:aggregation-aware-dc-correction}
    Suppose group $g$ has nonzero normalized influence.
    Then the denominator in \Cref{eq:aggregation-aware-dc-correction} is strictly positive, and $\CC_g^{\mathrm{mc}}$ is the unique coefficient vector of the form $\CC_g + \alpha \mathbf{m}_g$, with scalar $\alpha$, satisfying
    \begin{equation}
        \boldsymbol{\zeta}_g^\top
        U_g \CC_g^{\mathrm{mc}}
        = \boldsymbol{\zeta}_g^\top \yy_g.
        \label{eq:dc-group-constraint}
    \end{equation}
\end{proposition}

\begin{proof}
    A coefficient correction of the form $\alpha \mathbf{m}_g$ adds a constant offset to the synthesized group because
    \begin{equation*}
        U_g (\CC_g + \alpha \mathbf{m}_g)
        = U_g \CC_g + \alpha \mathbf{1}.
    \end{equation*}
    Its effect on the normalized mass contribution is determined by
    \begin{equation*}
        \hh_g^\top \mathbf{m}_g
        = \boldsymbol{\zeta}_g^\top U_g \mathbf{m}_g
        = \boldsymbol{\zeta}_g^\top \mathbf{1},
    \end{equation*}
    using \Cref{eq:dc-aggregation-functional,eq:dc-constant-synthesis-direction}.
    The normalized influences are nonnegative, and at least one is positive by assumption.
    Hence $\hh_g^\top \mathbf{m}_g > 0$, so the denominator in \Cref{eq:aggregation-aware-dc-correction} is nonzero.

    Substituting $\CC_g + \alpha \mathbf{m}_g$ into the group constraint gives
    \begin{align*}
        \boldsymbol{\zeta}_g^\top U_g
        (\CC_g + \alpha \mathbf{m}_g)
        &= \boldsymbol{\zeta}_g^\top \yy_g, \\
        \hh_g^\top \CC_g
        + \alpha \hh_g^\top \mathbf{m}_g
        &= \boldsymbol{\zeta}_g^\top \yy_g.
    \end{align*}
    With the residual $d_g$ from \Cref{eq:dc-mass-residual}, this is equivalent to
    \begin{equation*}
        \alpha \hh_g^\top \mathbf{m}_g = d_g.
    \end{equation*}

    Substituting $\CC_g + \alpha \mathbf{m}_g$ into \Cref{eq:dc-group-constraint} gives
    \begin{equation*}
        \boldsymbol{\zeta}_g^\top U_g
        (\CC_g + \alpha \mathbf{m}_g)
        = \boldsymbol{\zeta}_g^\top \yy_g.
    \end{equation*}
    By linearity, the left-hand side separates into the contribution of the filtered group and that of the correction:
    \begin{equation*}
        \boldsymbol{\zeta}_g^\top U_g \CC_g
        + \alpha \boldsymbol{\zeta}_g^\top U_g \mathbf{m}_g
        = \boldsymbol{\zeta}_g^\top \yy_g.
    \end{equation*}
    Using $\hh_g^\top = \boldsymbol{\zeta}_g^\top U_g$ from \Cref{eq:dc-aggregation-functional}, this becomes
    \begin{equation*}
        \hh_g^\top \CC_g
        + \alpha \hh_g^\top \mathbf{m}_g
        = \boldsymbol{\zeta}_g^\top \yy_g.
    \end{equation*}
    Subtracting the current contribution $\hh_g^\top \CC_g$ from both sides yields
    \begin{equation*}
        \alpha \hh_g^\top \mathbf{m}_g
        = \boldsymbol{\zeta}_g^\top \yy_g - \hh_g^\top \CC_g
        = d_g,
    \end{equation*}
    where the last equality follows from \Cref{eq:dc-mass-residual}.
    Thus, the correction must supply exactly the difference between the observed and filtered group contributions.

    Since $\hh_g^\top \mathbf{m}_g > 0$, this equation has the unique solution
    \begin{equation*}
        \alpha = \frac{d_g}{\hh_g^\top \mathbf{m}_g}.
    \end{equation*}
    The resulting coefficients are exactly those in \Cref{eq:aggregation-aware-dc-correction}.
\end{proof}

Although the correction is constant within each group, it can vary spatially after aggregation because the normalized weights differ between patch occurrences.
For complex-valued transforms, let $\mathcal{S}_g$ be the coefficient subspace corresponding to real-valued groups, as defined in \Cref{rem:complex-valued-transforms}.
If $\CC_g \in \mathcal{S}_g$, the correction amplitude is real and $\mathbf{m}_g = U_g^{-1} \mathbf{1} \in \mathcal{S}_g$, so the corrected coefficients also synthesize a real-valued group in exact arithmetic.

\begin{theorem}[Mass conservation after normalized overlap-add]
    \label{thm:aggregation-aware-dc-global-conservation}
    Suppose every sampling location is covered, the reliability weights and denominator in \Cref{eq:dc-aggregation-denominator} are fixed, and every contributing group satisfies \Cref{eq:dc-group-constraint}.
    Then the weighted overlap-add estimate in \Cref{def:weighted-overlap-add}, formed from the corrected groups, satisfies
    \begin{equation}
        \sum_{u \in \Omega} \widehat{x}(u)
        = \sum_{u \in \Omega} y(u).
        \label{eq:aggregation-aware-dc-global-conservation}
    \end{equation}
\end{theorem}

\begin{proof}
    Summing the corrected group contributions and using that every entry of $\yy_g$ is the observation at the corresponding sampling location gives
    \begin{align*}
        \sum_{u \in \Omega} \widehat{x}(u)
        &= \sum_g
           \boldsymbol{\zeta}_g^\top
           U_g \CC_g^{\mathrm{mc}} \\
        &= \sum_g \boldsymbol{\zeta}_g^\top \yy_g \\
        &= \sum_{u \in \Omega} y(u)
           \sum_g \sum_j
           \frac{w_{g,j} \omega_{g,j}(u)}{\eta(u)} \\
        &= \sum_{u \in \Omega} y(u)
           \frac{\eta(u)}{\eta(u)} \\
        &= \sum_{u \in \Omega} y(u).
    \end{align*}
\end{proof}

The theorem preserves the sum of the observed realization, not the unknown clean mass, and it does not impose separate constraints on arbitrary subregions.

\begin{remark}[Fixed aggregation weights]
\label{rmk:fixed-aggregation-weights}
The normalized aggregation functional depends on the denominator $\eta(u)$ and therefore on the weights of all contributing groups.
These weights are computed from the uncorrected shrinkage estimates and held fixed during mass correction.
Recomputing them after correction would change the aggregation functional and could invalidate the imposed constraints.

Each constrained filtering stage is therefore evaluated in two sweeps.
The first computes the filtered coefficients, their reliability weights, and the denominator $\eta(u)$.
The second applies \Cref{eq:aggregation-aware-dc-correction} and aggregates the corrected patches using the same weights and denominator.
The risk models in \Cref{sec:aggregation-weights} thus describe the uncorrected shrinkage estimates, not the corrected estimator.
The normalized patch weights sum to one at each covered sampling location, ensuring that each observation contributes exactly once to the total in \Cref{thm:aggregation-aware-dc-global-conservation}.
\end{remark}

\section{Experiments}
\label{sec:experiments}
We evaluate whether \ac{bmnd} combines effective noise reduction with intensity and structural preservation across data dimensions.
We begin with controlled experiments on the reference-patch shift schedule and finite-count structural allowance to examine the balance between processing workload and denoising quality, as well as the sensitivity of noise-aware matching.
We then investigate how mass conservation affects global intensity, regional bias, and reconstruction quality, first in direct image denoising and then in reconstruction from denoised sinograms.
A factorial ablation study extends this analysis across datasets and noise levels to assess the contributions of individual algorithmic choices.
Finally, we evaluate performance on measured fluorescence microscopy noise and demonstrate the applicability of the $n$-dimensional formulation beyond images and volumes on one-dimensional \ac{ecg} signals with added Gaussian noise.
These applications test the method beyond controlled image experiments, with particular attention to biological image structures and physiological waveform morphology.

The source code for our experiments can be found at \url{https://zivgitlab.uni-muenster.de/ag-pria/poisson-denoising}.

\subsection{Datasets}
We evaluate \ac{bmnd} on one-dimensional physiological signals, two-dimensional images, and three-dimensional volumes.
The datasets combine standard benchmarks and application data with generated diagnostic examples that isolate intensity levels and spatial structures.
Dataset-specific evaluation protocols are described in the corresponding experimental subsections.

The generated diagnostic datasets provide controlled two- and three-dimensional test cases.
The two-dimensional dataset contains three constant images with intensities $0.05$, $0.25$, and $0.80$, a smooth diagonal gradient, sharp intensity transitions with a circular region, spatially varying sinusoidal texture, as well as sparse Gaussian spots and ridge structures adapted from \authorcite{makitalo2010optimal}.
These examples allow us to examine intensity dependence and the preservation of distinct spatial structures under controlled conditions.
The three-dimensional dataset contains a constant volume, spheres of different sizes and intensities, crossing or oblique tubes, and sinusoidal texture within a smooth spatial envelope.
We refer to these datasets as \emph{generated 2D} and \emph{generated 3D}, respectively.

The natural-image benchmarks comprise \ac{bsd68}~\cite{martin2001database}, \ac{set12} as distributed with FFDNet~\cite{zhang2018ffdnet}, \ac{kodak24}~\cite{kodak1999lossless}, and selected reference images from scikit-image~\cite{vanderwalt2014scikit}.
We refer to these four sources collectively as \emph{natural} in the following experiments.
Together, they provide images containing smooth regions, edges, repeated structures, and textures.

For fluorescence microscopy, we use \ac{fmd}~\cite{zhang2018poisson} to extend the evaluation to biological image structures (examples in \Cref{fig:fmd:fov19}).
We additionally use three-dimensional anatomical brain phantom data generated with \ac{xcat}~\cite{segars2010xcat}.
For this, we randomly sample activity distributions and affine transformations.
The phantom then gets forward projected into sinogram space.
Examples from the \emph{generated 2D}, \emph{generated 3D}, \emph{natural}, \ac{fmd}, and \ac{xcat} datasets are shown in \Cref{fig:ablation:examples}.

For the mass-conservation experiments, we use the two-dimensional Shepp--Logan phantom and its forward-projected sinogram (see \Cref{fig:mc:shepp-logan}).
Its piecewise-constant regions provide controlled reference intensities for assessing global and regional intensity bias under Poisson noise, both in direct image denoising and in reconstruction from denoised sinograms.

For the one-dimensional evaluation, we use electrocardiogram recordings from the MIT--BIH Arrhythmia Database~\cite{moody2001impact} (example waveforms in \Cref{fig:ecg:qrs-morphology}).
We use channel zero from all records and assess reconstruction quality and preservation of beat morphology under added Gaussian noise.
The subject-level development and confirmation split, segment selection, and evaluation protocol are described in \Cref{sec:experiments:ecg}.

\subsection{Noise Models}
For Gaussian observations, let $\sigma_{255}$ denote the noise standard deviation in $[0, 255]$ units.
When the image is stored in $[0, 1]$ units, the corresponding standard deviation is
\begin{equation}
    \sigma = \frac{\sigma_{255}}{255}.
\end{equation}

For Poisson observations, let $I_i$ denote the reference intensity in stored units and let $d > 0$ be the normalization divisor.
The corresponding normalized intensity is
\begin{equation}
    X_i = \frac{I_i}{d}.
\end{equation}
Let $\lambda_{\mathrm{pk}} > 0$ denote the peak expected raw count, namely the Poisson rate at unit normalized intensity.
We define
\begin{equation}
    s_P = \frac{1}{\lambda_{\mathrm{pk}}},
    \qquad
    K_i \sim \Poi\left(\frac{X_i}{s_P}\right)
    = \Poi\left(\lambda_{\mathrm{pk}} X_i\right)
\end{equation}
and represent the observation in normalized processing units as
\begin{equation}
    Y_i = s_P K_i.
\end{equation}
The corresponding observation in the original stored units is
\begin{equation}
    \widetilde{Y}_i = d Y_i = \frac{d}{\lambda_{\mathrm{pk}}} K_i.
\end{equation}
Thus, $d = 1$ for images in $[0, 1]$ and $d = 255$ for images in $[0, 255]$.
The implementation uses $s_P = 1 / \lambda_{\mathrm{pk}}$ as the processing-unit Poisson count scale.

\subsection{Shift Schedule Analysis}
First, we measure the reference-patch count and coverage uniformity of each schedule.
Second, we measure the resulting denoising quality under additive Gaussian noise at several noise levels.
We use reference-patch count as a proxy for computational cost, since each reference patch initiates block matching and collaborative filtering.
This measures the scheduled processing workload rather than execution time.
All other algorithmic parameters are fixed to the standard configuration in \Cref{app:standard-configuration}.

The sparse schedule is omitted from the two-dimensional quality
comparison.
It differs from the generated schedule only by removing singleton axes from the axis hierarchy.
Since the evaluated two-dimensional images contain no singleton axes, the two schedules generate identical reference origins in this setting.

For each image and schedule, define a coverage map $C(x)$ by counting how many scheduled $8 \times 8$ reference patches contain pixel $x$:
\begin{equation}
    C(x) = \sum_{r \in \mathcal{R}} \mathbbm{1}\!\left[x \in r + \mathcal{B}\right],
\end{equation}
where $\mathcal{R}$ is the set of reference-patch origins and $\mathcal{B}$ is the patch support.
The spatial mean and standard deviation are then
\begin{equation}
    \mu_C = \frac{1}{|\Omega|} \sum_{x \in \Omega} C(x),
    \qquad
    \sigma_C = \sqrt{\frac{1}{|\Omega|} \sum_{x \in \Omega} \left(C(x) - \mu_C\right)^2},
\end{equation}
and the \ac{cv} is
\begin{equation}
    \mathrm{CV}(C) = \frac{\sigma_C}{\mu_C}.
\end{equation}
We calculate the \ac{cv} separately for each image and schedule and report its median across images in \Cref{fig:shift-schedule:coverage}.
HT and Wiener values are identical here because both stages use the same patch size and step.
$\mathrm{CV} = 0$ means every pixel is covered by exactly the same number of reference patches.
A higher \ac{cv} means reference-patch processing is concentrated more heavily in some spatial regions than others.
This measures only reference-schedule geometry, not matched-patch coverage or final aggregation weights.

\begin{figure}[htbp]
    \centering
    \includegraphics[width=0.8\linewidth]{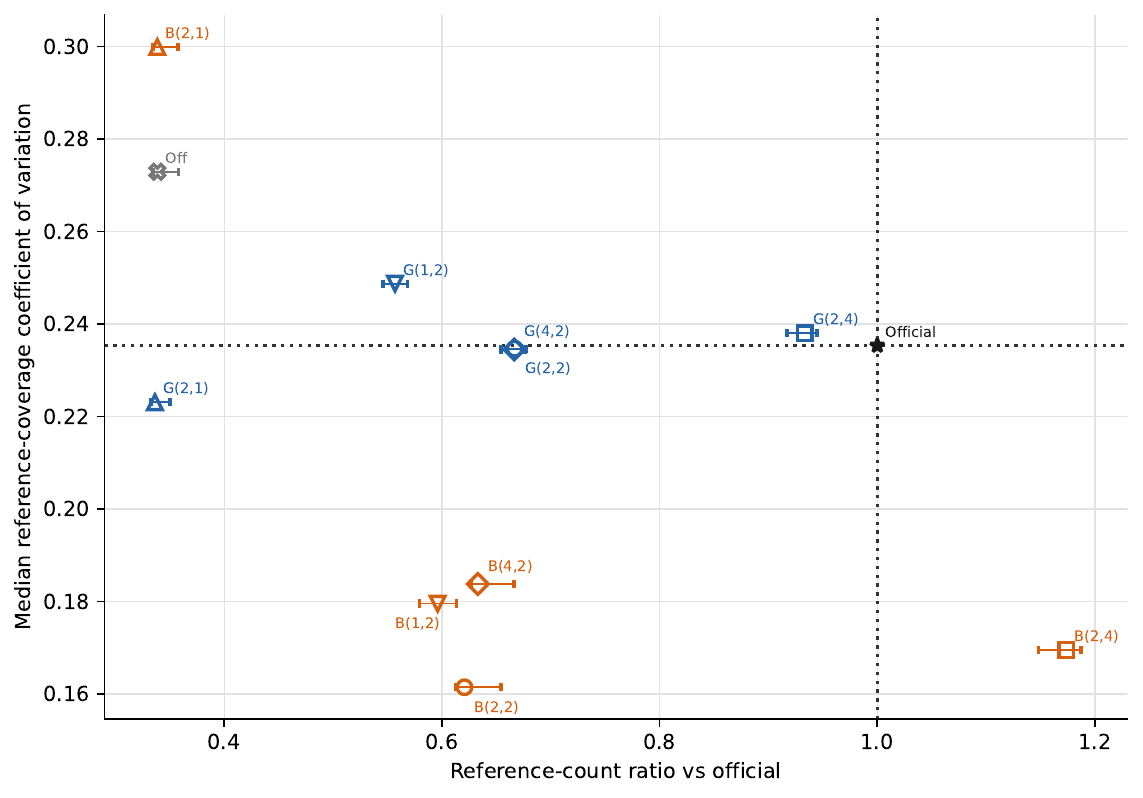}
    \caption{Reference-patch count and coverage uniformity.
    Counts are normalized by those of the official schedule and averaged across images.
    Horizontal error bars show their range across image shapes.
    Uniformity is measured by the median \ac{cv} of reference-patch coverage across images where lower values indicate more uniform coverage.
    Dotted lines mark the official schedule.
    G and B denote generated and balanced schedules, with tuples specifying (shift density, schedule density).}
    \label{fig:shift-schedule:coverage}
\end{figure}

\begin{figure}[htbp]
    \centering
    \includegraphics[width=\linewidth]{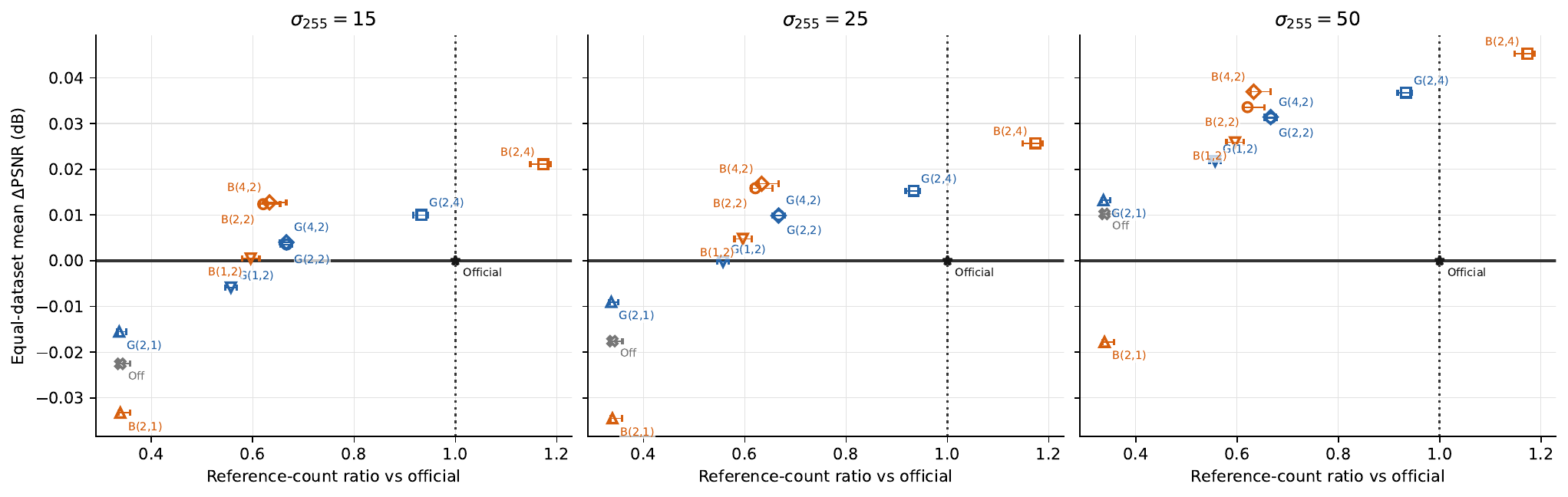}
    \caption{Denoising quality versus reference-patch count under Gaussian noise.
    \ac{psnr} changes are relative to the official schedule and averaged equally across BSD68, Kodak24, Set12, and scikit-image.
    Counts are normalized by those of the official schedule and averaged across images.
    Horizontal error bars show their range across image shapes.
    Panels correspond to $\sigma_{255} \in \{ 15, 25, 50 \}$.
    G and B denote generated and balanced schedules, with tuples specifying (shift density, schedule density).}
    \label{fig:shift-schedule:quality-cost}
\end{figure}

\Cref{fig:shift-schedule:coverage} compares the number and spatial distribution of the reference patches.
The balanced schedules generally give more uniform coverage than the official schedule, with balanced $(2, 2)$ giving the lowest \ac{cv} among configurations using fewer reference patches than the official schedule.
The generated $(2, 2)$ schedule has nearly the same coverage variation as the official schedule while using approximately $67\%$ of its reference patches.
Removing shifts or using only one schedule pass reduces the reference count further, but also makes the coverage less uniform.

The corresponding denoising results are shown in \Cref{fig:shift-schedule:quality-cost}.
All differences from the official schedule are small and remain below $0.05\,\mathrm{dB}$.
Schedules using approximately $60$-$70\%$ of the official reference count preserve or slightly improve the mean \ac{psnr}.
The balanced $(4, 2)$ schedule gives the highest \ac{psnr} in this reference-count range, although its advantage over balanced $(2, 2)$ is at most $0.003\,\mathrm{dB}$.
Reducing the reference count to approximately one third results in small losses at $\sigma_{255} = 15$ and $25$, whereas most schedules improve on the official schedule at $\sigma_{255} = 50$.

For the following experiments, we use the generated $(2, 2)$ schedule.
Among the dimension-independent schedules considered here, it most closely reproduces the official schedule while reducing the reference count by approximately one third.
Its mean \ac{psnr} differs from the official schedule by $+0.004$, $+0.010$, and $+0.031\,\mathrm{dB}$ for $\sigma_{255} = 15$, $25$, and $50$, respectively.

\subsection{Finite-count Structural Allowance Analysis}
We examine the sensitivity of denoising quality to the structure factor $\beta_s$ and intensity exponent $\kappa_s$.

We vary
\begin{equation}
    \beta_s \in \{ 0, 0.02, 0.05, 0.1, 0.2, 0.4, 0.8, 1.6 \},
    \qquad
    \kappa_s \in \{ 0.25, 0.5, 1, 2, 3, 4 \}.
\end{equation}
Finite-count matching is used only in the hard-thresholding stage, while Wiener matching remains fixed and uses the pilot estimate.
The remaining parameters are unchanged.
For simulated observations, we use peak expected counts between $0.5$ and $50$ and two noise realizations.

\begin{figure}[htbp]
    \centering
    \includegraphics[width=0.8\linewidth]{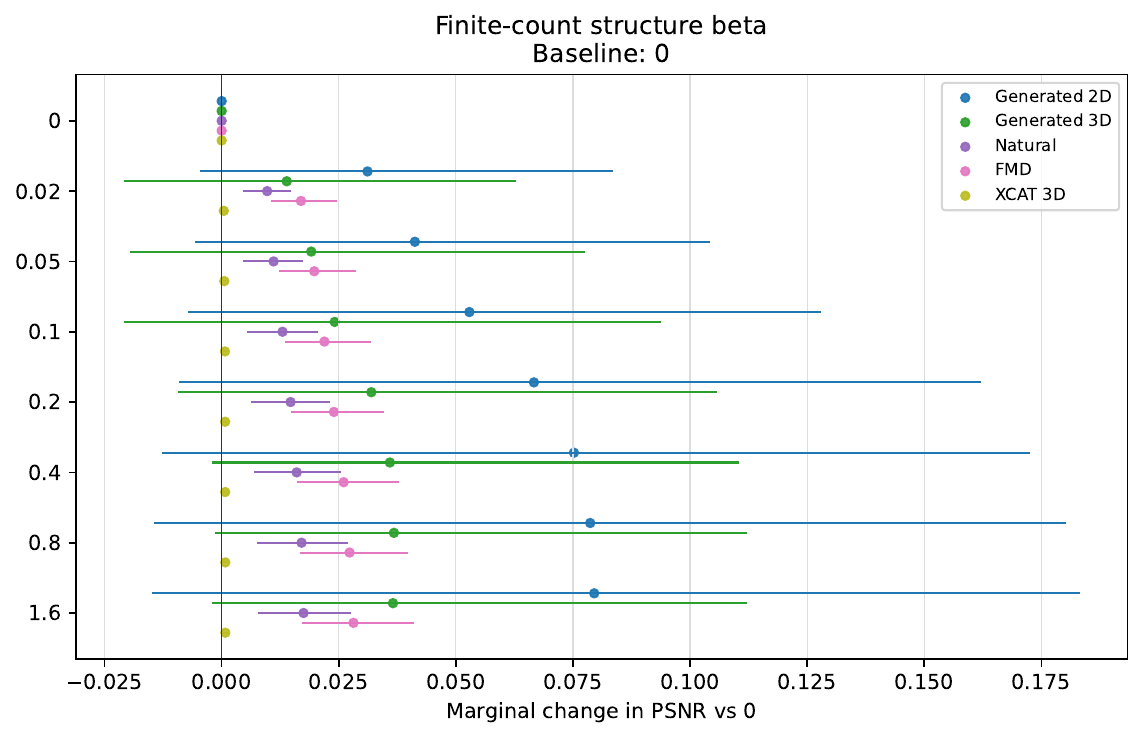}
    \medskip
    \includegraphics[width=0.8\linewidth]{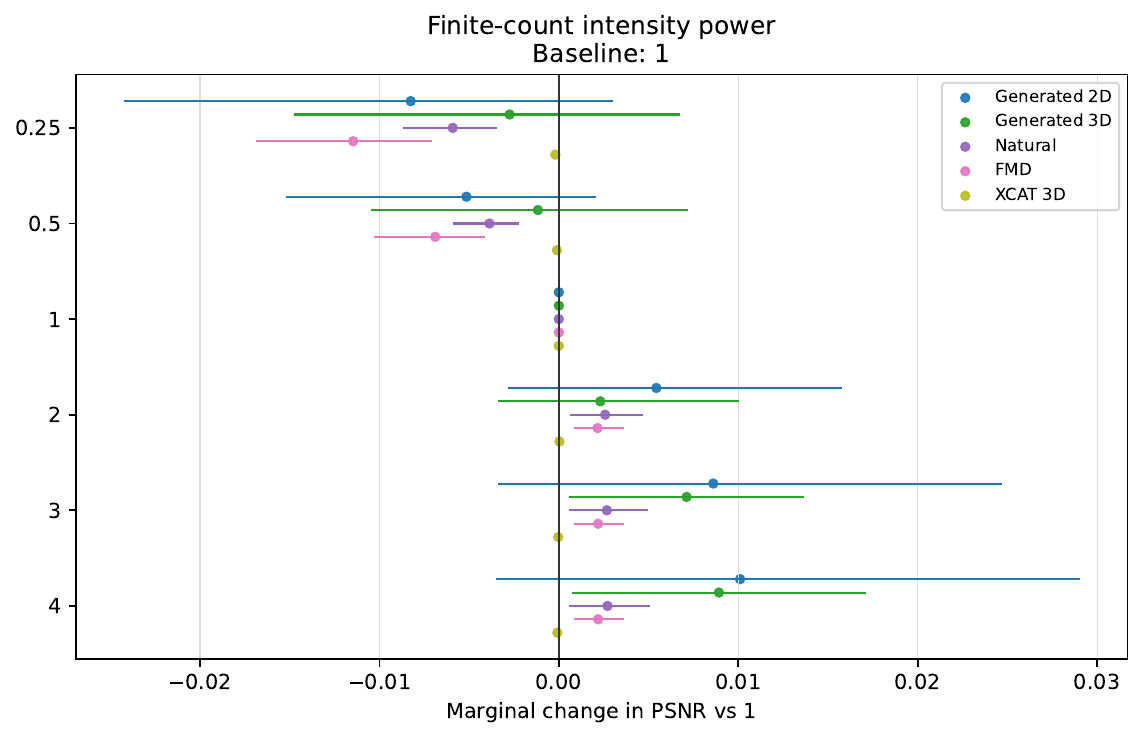}
    \caption{Source-balanced marginal effects of the finite-count structural allowance on \ac{psnr}.
    Effects are measured relative to $\beta_s = 0$ for the structure factor (top) and $\kappa_s = 1$ for the intensity exponent (bottom), averaging uniformly over the levels of the other factor.
    Positive values indicate improved reconstruction quality.
    Error bars denote 95\% nonparametric source-unit bootstrap intervals.}
    \label{fig:finite-count:marginal-effects}
\end{figure}

\Cref{fig:finite-count:marginal-effects} shows that a positive $\beta_s$ improves the mean \ac{psnr} on the generated, natural-image, and \ac{fmd} datasets.
The improvement is strongest on the generated two-dimensional data, where it reaches approximately $0.08\,\mathrm{dB}$.
Most of the other gains remain below $0.04\,\mathrm{dB}$.
The effect on the \ac{xcat} dataset is close to zero.

The choice of $\kappa_s$ has a smaller effect.
Values below $1$ tend to reduce the mean \ac{psnr}, while values above $1$ give small improvements on several datasets.
These differences are generally below $0.01\,\mathrm{dB}$.

For the following experiments, we use $\beta_s = 0.1$ and $\kappa_s = 1$.
At $\beta_s = 0.1$, the structural allowance gives a positive mean change in \ac{psnr}.
The influence of $\kappa_s$ is small, but $\kappa_s = 1$ generally performs at least as well as the lower exponents and gives the allowance a simple linear dependence on the mean reference count.

\subsection{Mass Conservation}
We evaluate the effect of mass conservation on intensity bias and reconstruction quality using Poisson-corrupted sinograms of the Shepp--Logan phantom.
We compare conservation in the hard-thresholding stage, the Wiener stage, and both stages with denoising without mass conservation.
We measure total sinogram intensity, regional reconstruction bias, and \ac{psnr} in sinogram and reconstructed-image space.

We evaluate the effect of mass conservation on intensity bias and reconstruction quality using the Shepp--Logan phantom.
We first denoise Poisson-corrupted phantom images directly, then denoise Poisson-corrupted sinograms and evaluate the resulting reconstructions.
This separates the effects of denoising from those introduced by the subsequent reconstruction step.
In both experiments, we compare conservation in the hard-thresholding stage, the Wiener stage, and both stages with denoising without mass conservation.

\begin{figure}[htbp]
    \centering
    \includegraphics[width=\linewidth]{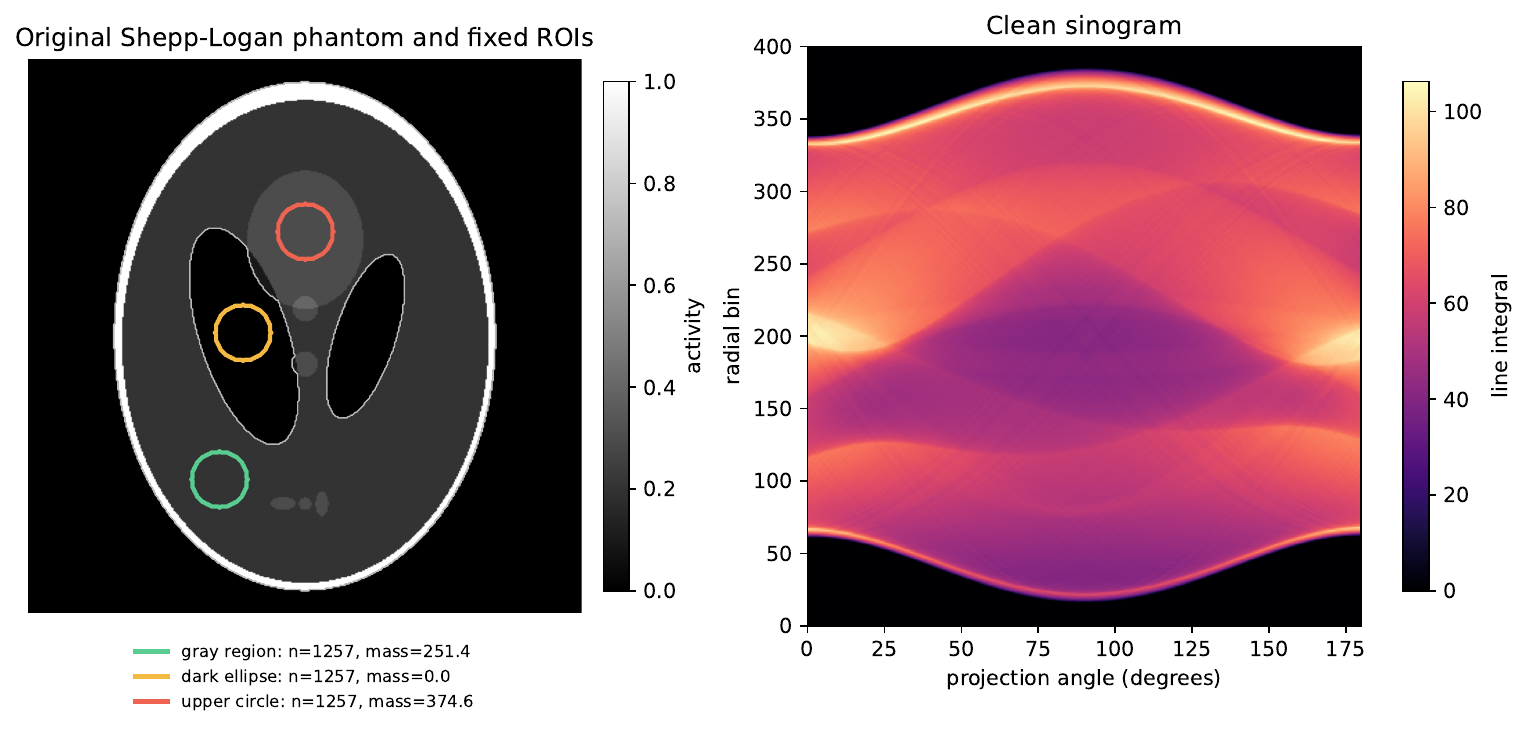}
    \caption{Shepp--Logan phantom with three fixed regions of interest (left) and its clean sinogram (right). The regions sample the gray background within the object, a dark ellipse with zero phantom activity, and the upper circular structure.}
    \label{fig:mc:shepp-logan}
\end{figure}

The phantom, evaluation regions, and clean sinogram are shown in \Cref{fig:mc:shepp-logan}.
We simulate Poisson observations at peak expected bin counts $\lambda_{\mathrm{pk}} \in \{ 0.1, 0.5, 1, 2, 10, 50 \}$, using the same 100 noise realizations for all configurations at each count level.

For image-space evaluation, we use the reconstruction of the clean sinogram as the reference.
This separates deviations caused by noise and denoising from the baseline reconstruction error relative to the original phantom.
We measure signed mass differences within the three regions shown in \Cref{fig:mc:shepp-logan} and over the full object support, expressing them as percentages of the corresponding reference mass.
For the dark ellipse, whose activity is zero in the original phantom, we instead report the signed mean-intensity difference to avoid normalization by a near-zero reference mass.
The whole-sinogram measurement uses the clean sinogram directly as its reference.
Bias curves report means over the 100 realizations, with $95\%$ bootstrap confidence intervals.
We additionally report paired \ac{psnr} differences relative to denoising without mass conservation in both sinogram and reconstructed-image space.

\begin{figure}[htbp!]
    \centering
    \includegraphics[width=\linewidth]{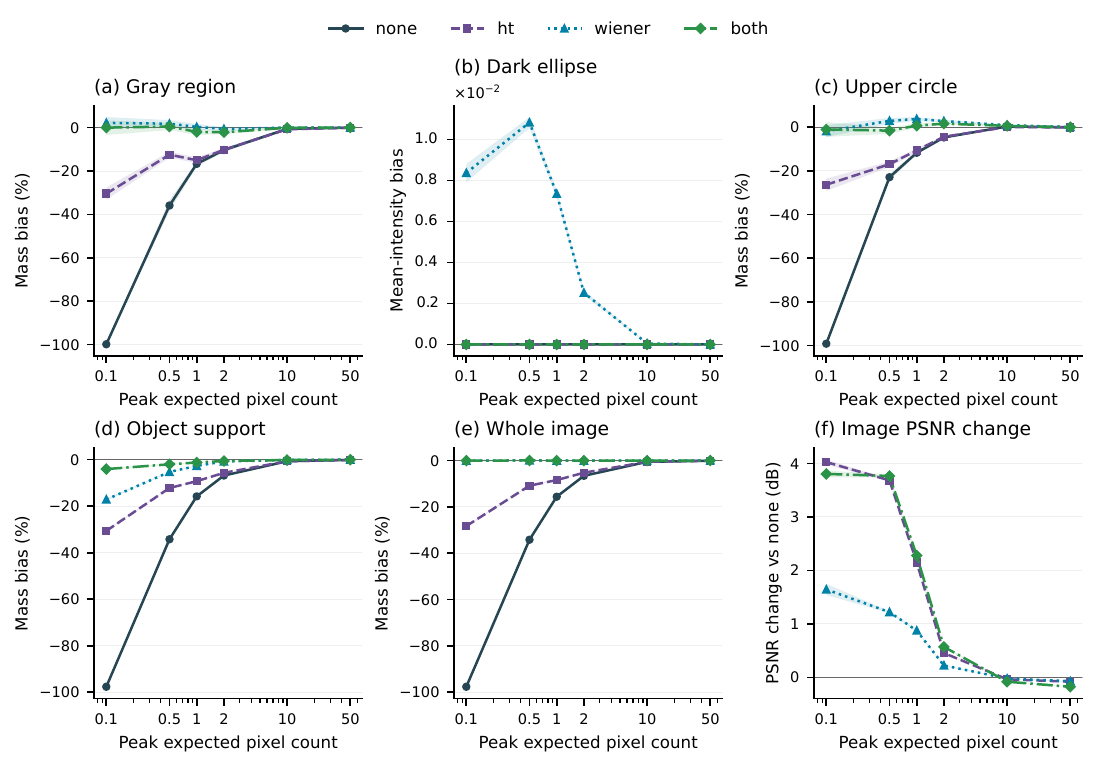}
    \caption{Mass-conservation effects when denoising Poisson-corrupted Shepp--Logan images directly.
    Panels show bias in the gray region (a), dark ellipse (b), upper circle (c), object support (d), and whole image (e), measured against the clean phantom.
    Bias is expressed as a percentage of reference mass, except in (b), which shows the signed mean-intensity difference.
    Panel (f) shows the \ac{psnr} change relative to denoising without mass conservation.
    Curves show means over 100 Poisson realizations, with shaded 95\% bootstrap confidence intervals.}
    \label{fig:mc:image-bias}
\end{figure}

For direct image denoising, we use the clean phantom as the reference for both intensity bias and \ac{psnr}.
\Cref{fig:mc:image-bias} shows substantial intensity loss without mass conservation at the lowest count levels.
At $\lambda_{\mathrm{pk}} = 0.1$, the mean bias is approximately $-97\%$ over the whole image and approaches $-100\%$ in the gray region and upper circle.
Applying conservation only in the hard-thresholding stage reduces the whole-image bias to approximately $-28\%$.
When conservation is applied in the Wiener stage, either alone or together with hard thresholding, the mean whole-image bias remains close to zero across the evaluated count levels.

Preserving the whole-image total does not ensure unbiased regional intensities.
At $\lambda_{\mathrm{pk}} = 0.1$, Wiener-only conservation retains an object-support bias of approximately $-17\%$, compared with approximately $-5\%$ when conservation is applied in both stages.
Applying conservation in Wiener or both stages keeps the dark-ellipse bias close to zero.

All three constrained configurations improve image \ac{psnr} at the lowest count levels.
At $\lambda_{\mathrm{pk}} = 0.1$, hard-thresholding-only and two-stage conservation give gains of approximately $4.0$ and $3.8\,\mathrm{dB}$, respectively, compared with $1.6\,\mathrm{dB}$ for Wiener-only conservation.
The gains diminish as the expected counts increase, with small losses at the highest evaluated count levels.

We next examine whether these benefits persist when denoising is performed in sinogram space before reconstruction.
For whole-sinogram bias and sinogram \ac{psnr}, the reference is the clean sinogram.
For regional bias and reconstructed-image \ac{psnr}, we instead use the reconstruction of the clean sinogram.
This separates deviations caused by noise and denoising from the baseline reconstruction error relative to the original phantom.

\begin{figure}[htbp!]
    \centering
    \includegraphics[width=\linewidth]{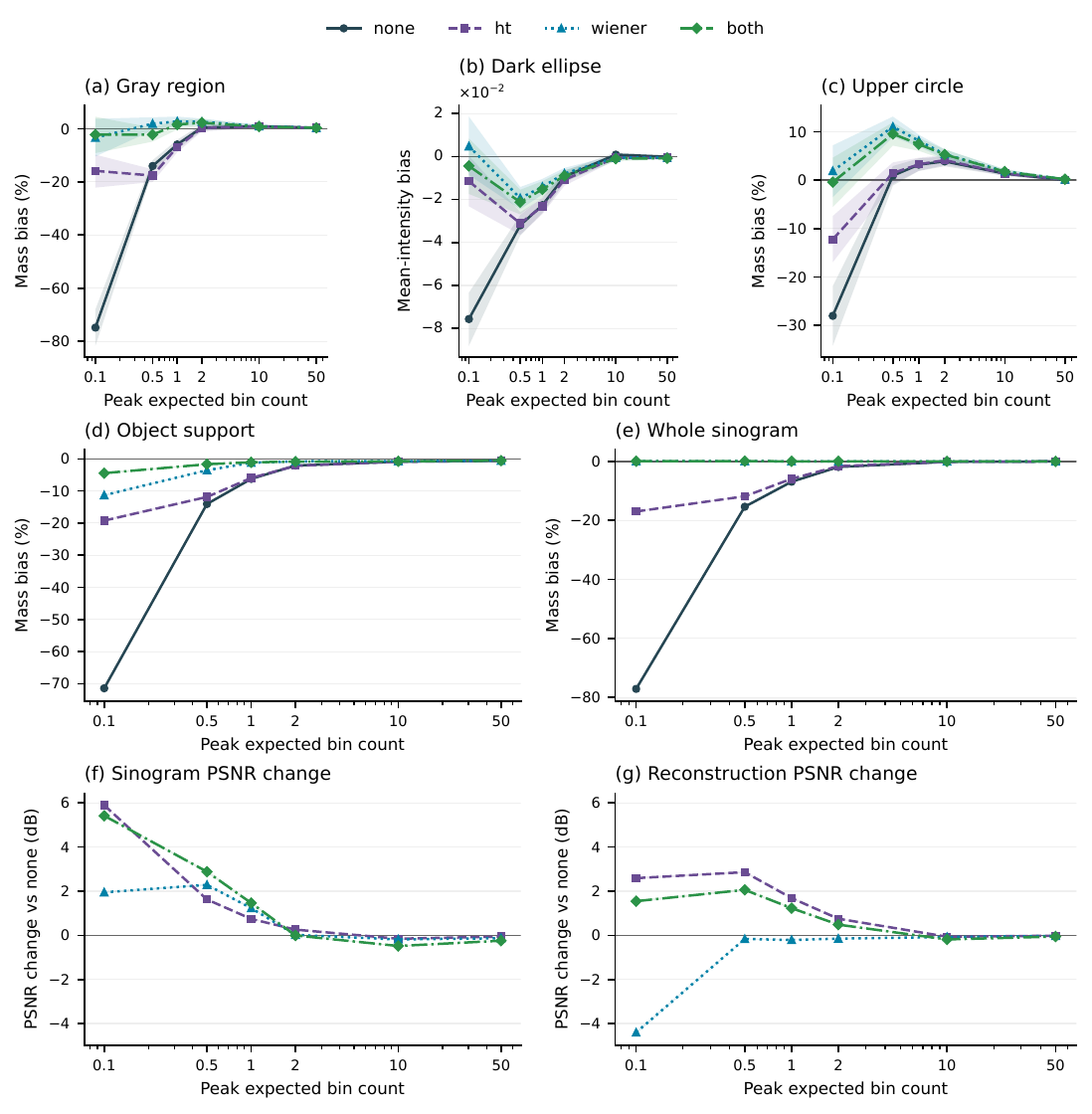}
    \caption{Mass-conservation effects when denoising Poisson-corrupted Shepp--Logan sinograms.
    Panels (a)--(d) show bias in the reconstructed gray region, dark ellipse, upper circle, and object support, measured against the reconstruction of the clean sinogram.
    Panel (e) shows whole-sinogram bias relative to the clean sinogram.
    Bias is expressed as a percentage of reference mass, except in (b), which shows the signed mean-intensity difference.
    Panels (f) and (g) show \ac{psnr} changes relative to denoising without mass conservation in sinogram and reconstructed-image space, respectively.
    Curves show means over 100 Poisson realizations, with shaded 95\% bootstrap confidence intervals.}
    \label{fig:mc:sinogram-bias}
\end{figure}

\Cref{fig:mc:sinogram-bias} shows pronounced intensity loss without mass conservation at the lowest count level.
At $\lambda_{\mathrm{pk}} = 0.1$, the mean bias is approximately $-77\%$ over the whole sinogram and $-71\%$ over the reconstructed object support.
Applying conservation only in the hard-thresholding stage reduces these losses, but does not eliminate the final sinogram bias.
When conservation is applied in the Wiener stage, either alone or together with hard thresholding, the mean whole-sinogram bias remains close to zero across the evaluated count levels.

The regional results show that preserving the sinogram total does not ensure unbiased reconstructed intensities.
Wiener-stage conservation substantially reduces the strong underestimation in the gray region at the lowest count level.
Applying conservation in both stages also reduces the object-support bias to approximately $-5\%$ at this level.
However, regional errors remain, and their sign depends on the region and count level.
In the upper-circle ROI, Wiener-stage conservation produces a positive bias of approximately $10\%$ at $\lambda_{\mathrm{pk}} = 0.5$.
The dark ellipse also retains a negative mean-intensity bias at several count levels.

At the lowest count levels, all three constrained configurations improve sinogram \ac{psnr}, but their effects on reconstructed-image \ac{psnr} differ (\Cref{fig:mc:sinogram-bias}, panels (f) and (g)).
HT-only conservation gives the largest reconstruction gains, approximately $2.6$ and $2.9\,\mathrm{dB}$ at $\lambda_{\mathrm{pk}} = 0.1$ and $0.5$, respectively.
Applying conservation in both stages also improves reconstruction \ac{psnr} at these levels, by approximately $1.5$ and $2.1\,\mathrm{dB}$.
In contrast, Wiener-only conservation reduces reconstruction \ac{psnr} by approximately $4.4\,\mathrm{dB}$ at $\lambda_{\mathrm{pk}} = 0.1$, despite improving sinogram \ac{psnr}.
At higher count levels, the benefits diminish, and conservation can cause small \ac{psnr} losses.

Across both experiments, Wiener-only and two-stage conservation retain the same observed sinogram mass, yet constraining the pilot as well gives substantially better reconstruction \ac{psnr} at the lowest count levels.
In this experiment, HT-only conservation provides the largest low-count reconstruction-PSNR gains, whereas two-stage conservation combines preservation of the observed total with improved reconstruction quality.

\subsection{Ablation Study}
We evaluate how the individual algorithmic choices affect reconstruction quality under Poisson noise and on fluorescence microscopy acquisitions.
The analysis examines the effects of noise-aware matching, Wiener gains, and aggregation weights, together with the influence of thresholding and covariance modeling.
We report marginal changes in \ac{psnr} relative to the reference level of each factor, averaging over the other evaluated settings.

We evaluate generated 2D and 3D data, natural images, three-dimensional \ac{xcat} brain sinograms, and \ac{fmd} acquisitions.
The simulated-noise experiments comprise eight generated images, four generated volumes, 39 natural images, and 12 \ac{xcat} sinograms, each evaluated at peak expected counts $\lambda_{\mathrm{pk}} \in \{ 1, 5, 20, 100 \}$.
For \ac{fmd}, we use 36 image-type and field-of-view combinations at averaging levels $1$, $4$, and $16$.

We jointly vary 12 factors in a complete factorial design comprising $276{,}480$ configurations.
The factors cover first-stage matching, covariance modeling, thresholding, Wiener gains, aggregation weights in both stages, and group mass conservation.
The complete factor levels and fixed settings are listed in \Cref{tab:ablation-profile-space}.

We summarize the results separately for each dataset group using source-balanced marginal effects.
A source unit is an individual image or volume for the generated and natural-image data, an individual sinogram for \ac{xcat}, and an image-type and field-of-view combination for \ac{fmd}.
Within each source unit, we first average \ac{psnr} over the evaluated noise conditions and then uniformly over all combinations of the other factors.
For each factor level, we subtract the corresponding mean at its reference level and average these differences equally across source units.
Positive values therefore indicate improved reconstruction quality relative to the reference level of that factor, rather than relative to a single common baseline configuration.
We obtain $95\%$ confidence intervals from $1000$ bootstrap resamples of the source units.

\begin{figure}[htbp]
    \centering
    \includegraphics[width=0.8\linewidth]{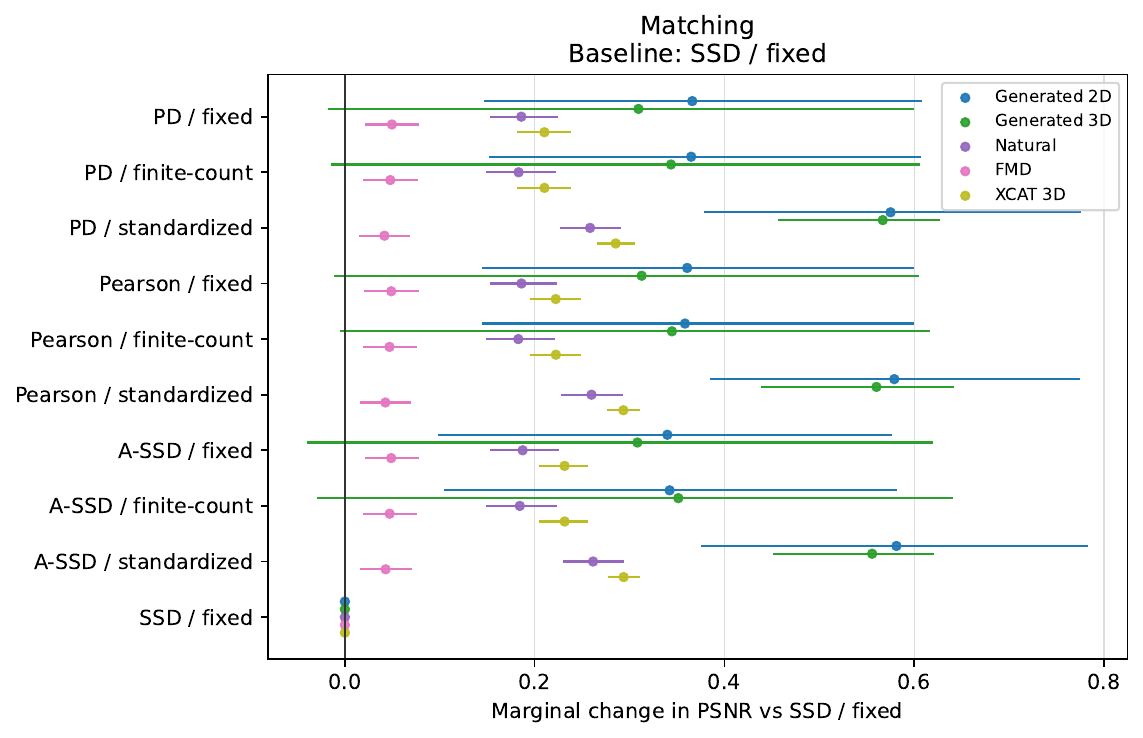}
    \caption{Marginal effects of the matching strategy on \ac{psnr} relative to SSD with fixed calibration.
    PD denotes Poisson deviance; A-SSD denotes Anscombe-transformed SSD.
    Finite-count denotes reference-based finite-count calibration; standardized denotes candidate standardization.
    Error bars denote 95\% source-unit bootstrap intervals.}
    \label{fig:ablation:matching}
\end{figure}

\Cref{fig:ablation:matching} shows positive mean changes in \ac{psnr} for all evaluated noise-aware matching strategies relative to SSD with fixed calibration.
Poisson deviance, Pearson, and Anscombe-transformed SSD produce similar marginal effects when used with the same calibration.
Candidate standardization gives the largest mean improvements on the generated, natural-image, and \ac{xcat} datasets, whereas its effect on \ac{fmd} is similar to that of the other calibration strategies.
Fixed and reference-based finite-count calibration give nearly identical marginal results.

\begin{figure}[htbp!]
    \centering
    \includegraphics[width=0.8\linewidth]{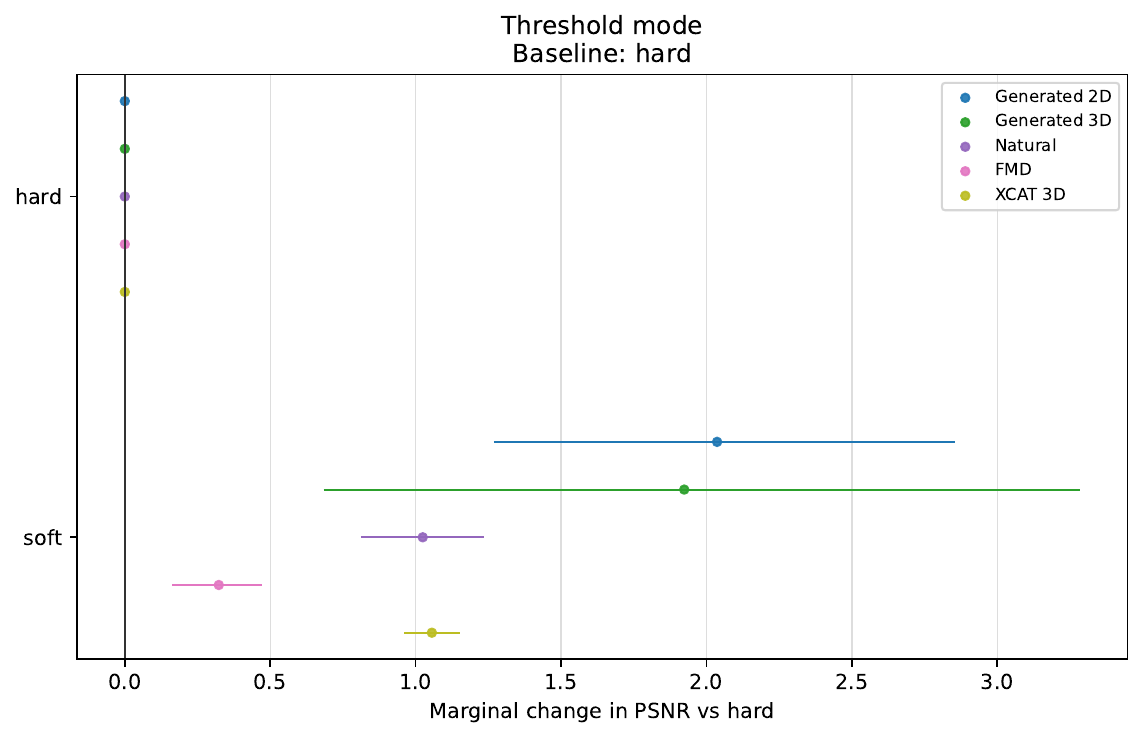}
    \medskip
    \includegraphics[width=0.8\linewidth]{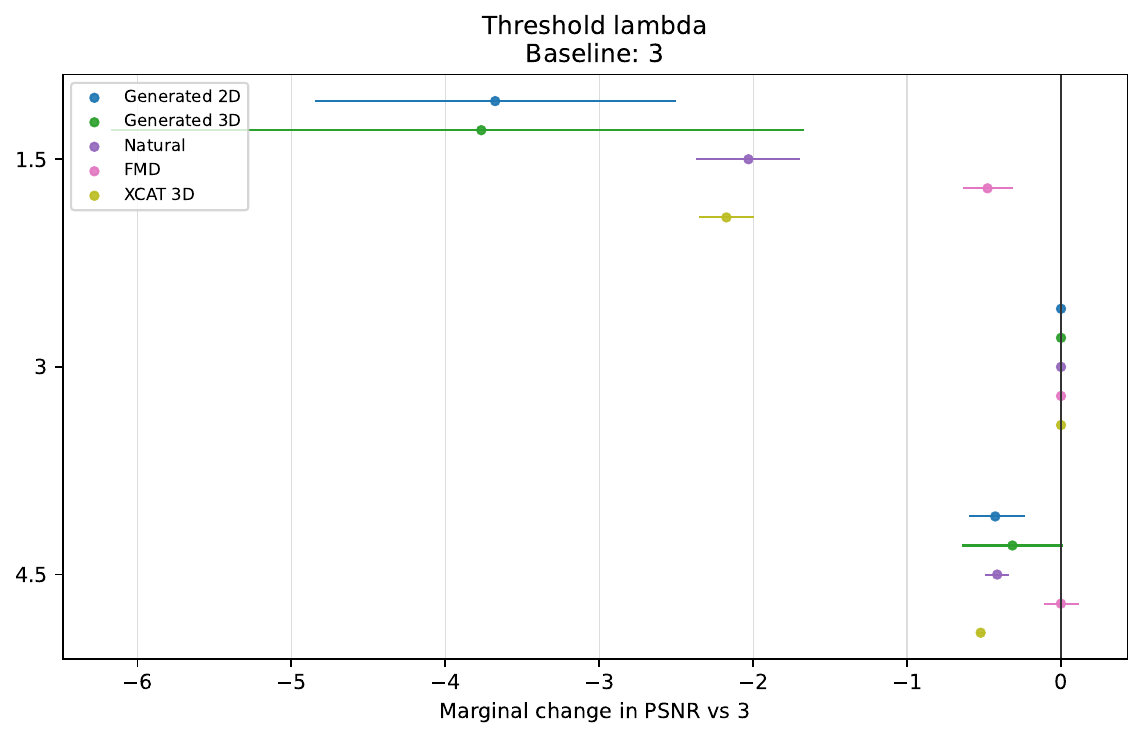}
    \caption{Marginal \ac{psnr} effects of the thresholding rule relative to hard thresholding (top) and the threshold multiplier relative to $\lambda_\mathrm{HT} = 3$ (bottom).
    Error bars denote 95\% source-unit bootstrap intervals.
    Note the different horizontal scales.}
    \label{fig:ablation:thresholding}
\end{figure}

The first filtering stage is sensitive to both the thresholding rule and the threshold multiplier.
\Cref{fig:ablation:thresholding} shows that soft thresholding gives higher marginal mean \ac{psnr} than hard thresholding across all evaluated dataset groups.
The improvement is largest on the generated data and smallest on \ac{fmd}.
Reducing the threshold multiplier from $\lambda_\mathrm{HT} = 3$ to $\lambda_\mathrm{HT} = 1.5$ substantially decreases the marginal mean \ac{psnr} in every group.
Increasing it to $\lambda_\mathrm{HT} = 4.5$ produces smaller losses on the generated, natural-image, and \ac{xcat} datasets, and little change on \ac{fmd}.

\begin{figure}[htbp!]
    \centering
    \includegraphics[width=0.8\linewidth]{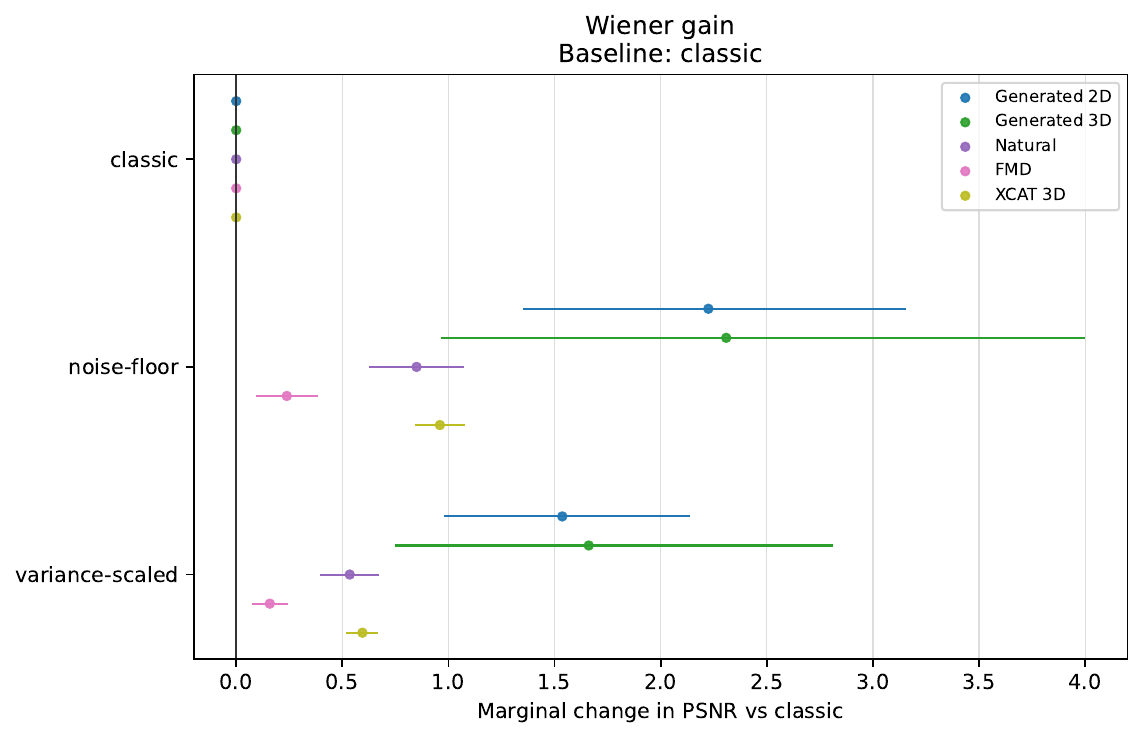}
    \caption{Marginal effects of the Wiener gain formulation on \ac{psnr} relative to classic gains.
    Error bars denote 95\% source-unit bootstrap intervals.}
    \label{fig:ablation:wiener-gain}
\end{figure}

For the second filtering stage, \Cref{fig:ablation:wiener-gain} shows positive marginal \ac{psnr} changes for both the noise-floor and variance-scaled formulations relative to classic gains across all evaluated dataset groups.
The noise-floor formulation gives the larger mean improvement in every group.
The gains are strongest on the generated data, but remain positive on natural images, \ac{xcat}, and \ac{fmd}.

\begin{figure}[htbp!]
    \centering
    \includegraphics[width=0.8\linewidth]{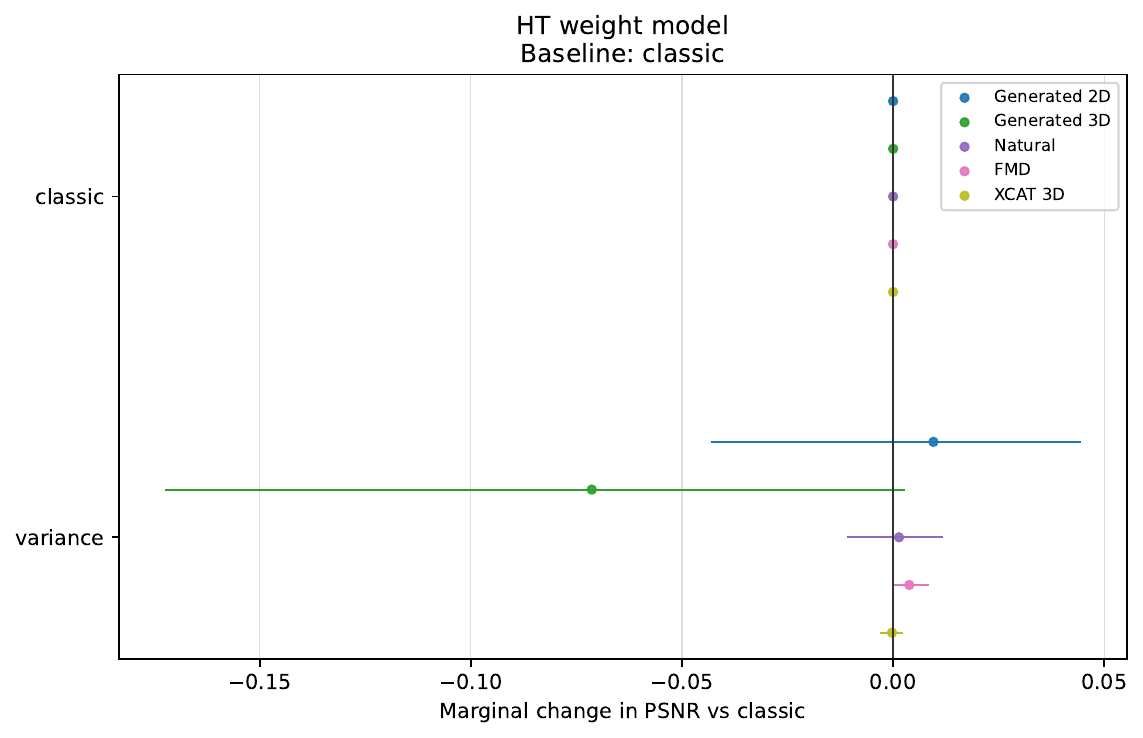}
    \medskip
    \includegraphics[width=0.8\linewidth]{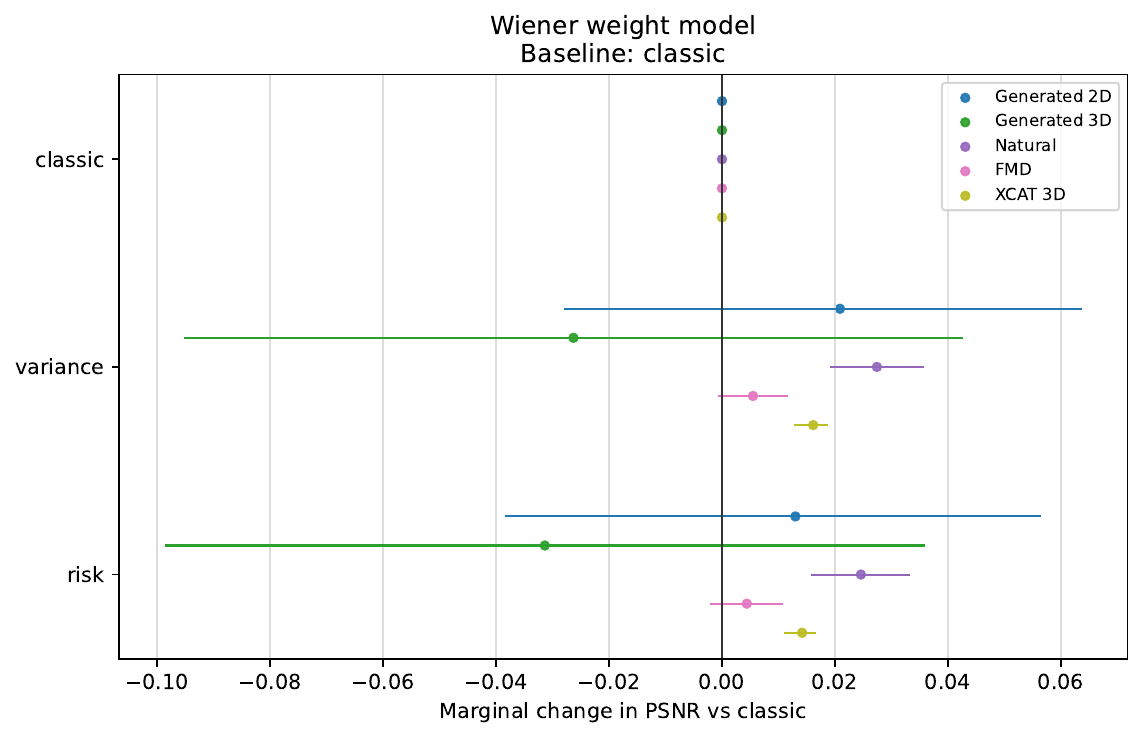}
    \caption{Marginal \ac{psnr} effects of aggregation weights relative to classic weighting in the first filtering stage (top) and Wiener stage (bottom).
    Error bars denote 95\% source-unit bootstrap intervals.
    Note the different horizontal scales.}
    \label{fig:ablation:weight-models}
\end{figure}

\Cref{fig:ablation:weight-models} compares variance-aware and bias-aware aggregation weights with classic weighting.
In the first filtering stage, variance weighting has little effect on natural images and \ac{xcat}, while giving a small positive mean change on \ac{fmd}.
The benefit is more pronounced in the Wiener stage, where both variance and risk weighting improve the marginal mean \ac{psnr} on natural images and \ac{xcat}.
The corresponding mean improvements on \ac{fmd} are smaller.
For the generated 2D and 3D datasets, the effects remain inconclusive in both stages.
The mean changes are positive in two dimensions and negative in three dimensions, but all corresponding confidence intervals include zero.
The bias-aware risk model gives similar marginal results to variance weighting, without a clear additional benefit.

The remaining factor effects are reported in \Cref{app:ablation}.
Using exact covariance planes improves the marginal mean \ac{psnr} on \ac{fmd}, but decreases it on natural images and \ac{xcat}, while the generated-data effects remain uncertain (\Cref{fig:ablation:covariance-planes}).
Changing the aggregation-weight domain has only small effects in both stages (\Cref{fig:ablation:ht-weight-domain,fig:ablation:wiener-weight-domain}).
Patch-level weighting gives more consistently positive mean changes in the first stage than in the Wiener stage, where the direction depends on the dataset (\Cref{fig:ablation:ht-weight-scope,fig:ablation:wiener-weight-scope}).
Group mass conservation improves the marginal mean \ac{psnr} on natural images and \ac{xcat}, while its effects on generated data depend on the stage of application and remain close to zero on \ac{fmd} (\Cref{fig:ablation:group-mass}).

\begin{figure}[htbp!]
    \centering
    \includegraphics[width=\linewidth]{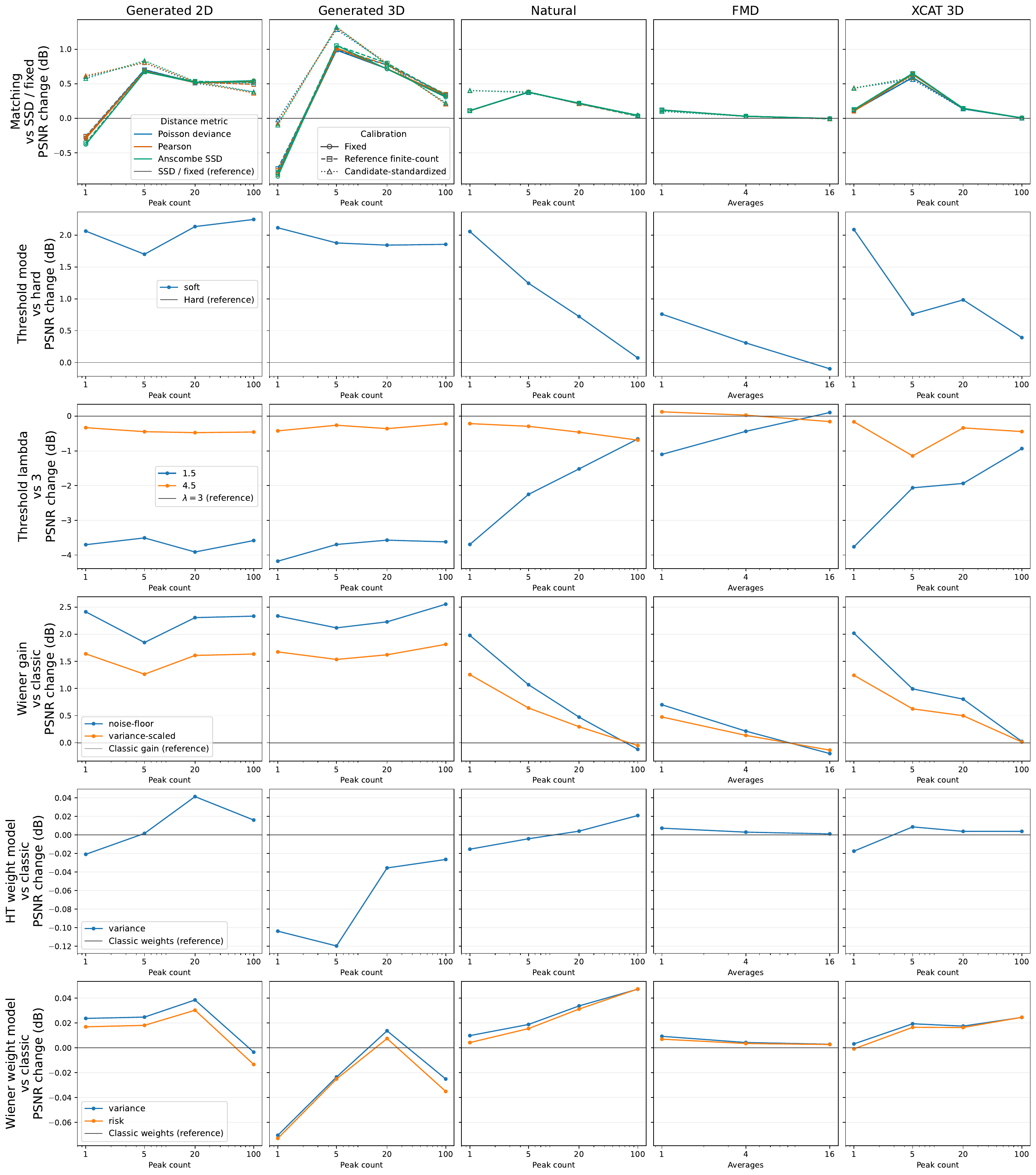}
    \caption{Source-balanced marginal \ac{psnr} effects by noise condition, with dataset groups in columns and algorithmic factors in rows.
    Effects are relative to the indicated reference settings, averaging over the other factors.
    Logarithmic horizontal axes show peak expected counts or acquisition averages for \ac{fmd}.
    In the matching row, colors identify distances, and line styles and markers identify calibrations.
    Vertical scales are shared within rows.}
    \label{fig:ablation:noise-trends}
\end{figure}

\Cref{fig:ablation:noise-trends} shows that the benefits of the proposed Wiener gain formulations become more pronounced as noise increases.
On natural images, \ac{xcat}, and \ac{fmd}, their largest mean improvements occur at the lowest peak count or with single acquisitions.
Soft thresholding likewise gives its largest improvements on these datasets under the strongest evaluated noise.
The generated datasets differ, retaining substantial benefits from both modifications throughout the evaluated count range.

The choice of matching calibration has a larger effect on mean \ac{psnr} at low counts.
At $\lambda_{\mathrm{pk}} = 1$, candidate standardization improves upon fixed and reference-based finite-count calibration on the generated, natural-image, and \ac{xcat} datasets.
The calibration strategies give more similar results at higher counts.
The aggregation-weight models have smaller effects with a different dependence on count level.
In particular, the benefits of Wiener variance and risk weighting on natural images increase rather than decrease with count level.

\begin{figure}[htbp]
    \centering
    \includegraphics[width=0.95\linewidth]{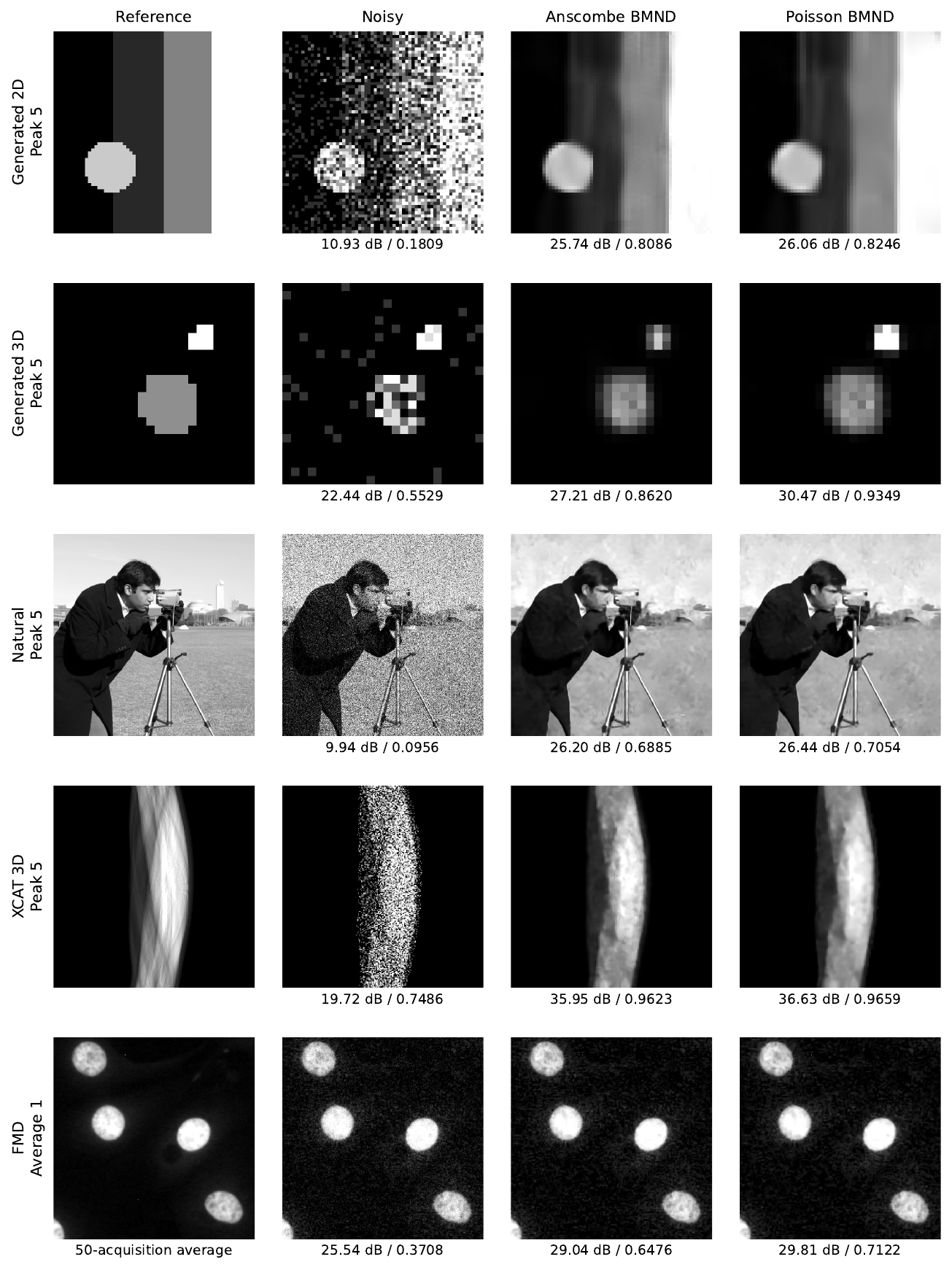}
    \caption{Denoising examples from the five ablation dataset groups.
    Columns show the reference, noisy observation, Anscombe \ac{bmnd}, and direct Poisson \ac{bmnd} with the globally top-ranked ablation configuration.
    Simulated observations use $\lambda_{\mathrm{pk}} = 5$; the \ac{fmd} example uses a single acquisition with a 50-acquisition average as reference.
    Labels report \ac{psnr}/\ac{ssim}.}
    \label{fig:ablation:examples}
\end{figure}

To complement the marginal-effect analysis, \Cref{fig:ablation:examples} compares example reconstructions obtained with Anscombe \ac{bmnd} and direct Poisson \ac{bmnd} across all five dataset groups.
For direct Poisson denoising, we use the single configuration ranked first globally across all datasets in the ablation study.
The examples are randomly selected.
Direct Poisson denoising gives higher \ac{psnr} and \ac{ssim} in each displayed example.

Overall, noise-aware matching and the modified Wiener gains have positive marginal mean effects on \ac{psnr} across the evaluated dataset groups.
The results by noise level show larger mean improvements of the modified Wiener gains under strong noise and highlight the importance of matching calibration at low counts.
Noise-aware Wiener aggregation yields smaller improvements, most clearly on natural images and anatomical phantom data.

\subsection{FMD Denoising}
We evaluate \ac{bmnd} on fluorescence microscopy images from \ac{fmd}, covering confocal, two-photon, and widefield microscopy.
Unlike the simulated-noise experiments, this evaluation uses noisy acquisitions, with noise levels controlled by averaging $1$, $2$, $4$, $8$, or $16$ raw images.
At each averaging level, we evaluate 48 images of size $512 \times 512$ from the held-out FOV 19, comprising 20 confocal, 16 two-photon, and 12 widefield images.
We calculate \ac{psnr} and \ac{ssim} against the dataset-provided high-SNR references and report arithmetic means of the per-image metrics.

We compare direct Poisson \ac{bmnd} with an Anscombe-transform approach using Gaussian \ac{bmnd}.
Following \authorcite{zhang2018poisson}, both approaches use the same calibrated affine mean-variance model,
\begin{equation}
    \operatorname{Var}(Y_i) = a \mathbb{E}[Y_i] + b,
\end{equation}
fitted separately for each image type and averaging level from repeated acquisitions in normalized intensity units.
Calibration uses five folds defined over fields of view, with the fold containing the evaluated field excluded from its calibration fit.
For direct Poisson denoising, we shift the observation by $b / a$, use the Poisson count scale $s_P = a$, and subtract the shift after denoising.
This matches the affine variance model but does not imply that the shifted observations follow an exact Poisson distribution.
The Anscombe baseline uses the same shift and count scale, applies Gaussian \ac{bmnd} with unit noise standard deviation in the transformed domain, and then applies the inverse transform.

For direct Poisson denoising, we select either one profile across the \ac{fmd} tuning data or a separate profile for each microscopy modality.
Selection maximizes mean \ac{psnr}, first averaging within each image-type and field-of-view combination and then equally across these source units.
The tuning experiment uses central $64 \times 64$ regions from six image types and six fields of view at averaging levels $1$, $4$, and $16$.
Test FOV 19 is excluded from profile selection.
The selected profiles remain fixed across all five test averaging levels.

\begingroup
    \setlength{\tabcolsep}{2pt}
    \setlength{\LTcapwidth}{\linewidth}
    \renewcommand{\arraystretch}{1.04}
    \begin{longtable}{@{}>{\raggedright\arraybackslash}p{0.30\linewidth}lrrrrr@{}}
    \caption{FMD denoising performance by averaging level.
    Published results are from \authorcite{zhang2018poisson}.
    For the full test set, bold marks the best non-deep-learning score and underlining the best overall.
    Within each modality subset, bold marks the best score.}
    \label{tab:fmd-denoising} \\
        \toprule
        \rowcolor{white}
        Method & Metric & $1$ & $2$ & $4$ & $8$ & $16$ \\
        \midrule
        \endfirsthead
        \toprule
        \rowcolor{white}
        Method & Metric & $1$ & $2$ & $4$ & $8$ & $16$ \\
        \midrule
        \endhead
        \midrule
        \multicolumn{7}{r}{Continued on next page} \\
        \endfoot
        \bottomrule
        \endlastfoot
        \rowcolor{white}
        \multicolumn{7}{@{}l}{\textit{Reported by \authorcite{zhang2018poisson}}} \\*
        \rowcolor{white}
        Raw & PSNR (dB) & $27.22$ & $30.08$ & $32.86$ & $36.03$ & $39.70$ \\*
        \rowcolor{white}
         & SSIM & $0.5442$ & $0.6800$ & $0.7981$ & $0.8892$ & $0.9487$ \\
        \rowcolor{black!6}
        VST + BM3D & PSNR (dB) & $32.71$ & $34.09$ & $36.05$ & $38.01$ & $40.61$ \\*
        \rowcolor{black!6}
         & SSIM & $0.7922$ & $0.8430$ & $0.8970$ & $0.9336$ & $0.9598$ \\
        \rowcolor{white}
        PURE-LET & PSNR (dB) & $31.95$ & $33.49$ & $35.29$ & $37.25$ & $39.59$ \\*
        \rowcolor{white}
         & SSIM & $0.7664$ & $0.8270$ & $0.8814$ & $0.9212$ & $0.9450$ \\
        \rowcolor{black!6}
        DnCNN & PSNR (dB) & $34.88$ & $36.02$ & $37.57$ & $39.28$ & $\underline{41.57}$ \\*
        \rowcolor{black!6}
         & SSIM & $0.9063$ & $\underline{0.9257}$ & $0.9460$ & $0.9588$ & $0.9721$ \\
        \rowcolor{white}
        Noise2Noise & PSNR (dB) & $\underline{35.40}$ & $\underline{36.40}$ & $\underline{37.59}$ & $\underline{39.43}$ & $41.45$ \\*
        \rowcolor{white}
         & SSIM & $\underline{0.9187}$ & $0.9230$ & $\underline{0.9481}$ & $\underline{0.9601}$ & $\underline{0.9724}$ \\
        \rowcolor{white}
        \multicolumn{7}{@{}l}{\textit{Our evaluation}} \\*
        \rowcolor{black!6}
        Anscombe BMND & PSNR (dB) & $33.59$ & $34.83$ & $36.64$ & $38.48$ & $41.11$ \\*
        \rowcolor{black!6}
         & SSIM & $0.8453$ & $0.8853$ & $0.9237$ & $0.9459$ & $0.9661$ \\
        \rowcolor{white}
        BMND (global) & PSNR (dB) & $33.83$ & $34.96$ & $36.64$ & $38.29$ & $40.69$ \\*
        \rowcolor{white}
         & SSIM & $0.8687$ & $0.8986$ & $0.9285$ & $0.9445$ & $0.9627$ \\
        \rowcolor{black!6}
        BMND (by modality) & PSNR (dB) & $\mathbf{34.59}$ & $\mathbf{35.62}$ & $\mathbf{37.19}$ & $\mathbf{38.86}$ & $\mathbf{41.24}$ \\*
        \rowcolor{black!6}
         & SSIM & $\mathbf{0.9027}$ & $\mathbf{0.9212}$ & $\mathbf{0.9418}$ & $\mathbf{0.9531}$ & $\mathbf{0.9671}$ \\
        \hline\hline
        \rowcolor{white}
        \multicolumn{7}{@{}l}{\textit{Confocal subset}} \\*
        \rowcolor{white}
        Noisy & PSNR (dB) & $29.13$ & $32.11$ & $34.86$ & $37.97$ & $41.08$ \\*
        \rowcolor{white}
         & SSIM & $0.7100$ & $0.8131$ & $0.8883$ & $0.9404$ & $0.9721$ \\
        \rowcolor{black!6}
        Anscombe BMND & PSNR (dB) & $35.88$ & $37.47$ & $38.88$ & $\mathbf{40.63}$ & $\mathbf{42.34}$ \\*
        \rowcolor{black!6}
         & SSIM & $0.9353$ & $0.9523$ & $0.9639$ & $0.9740$ & $\mathbf{0.9814}$ \\
        \rowcolor{white}
        BMND (rank 1) & PSNR (dB) & $\mathbf{35.95}$ & $\mathbf{37.55}$ & $\mathbf{38.90}$ & $\mathbf{40.63}$ & $42.32$ \\*
        \rowcolor{white}
         & SSIM & $\mathbf{0.9390}$ & $\mathbf{0.9552}$ & $\mathbf{0.9650}$ & $\mathbf{0.9743}$ & $0.9813$ \\
        \hline
        \rowcolor{white}
        \multicolumn{7}{@{}l}{\textit{Two-photon subset}} \\*
        \rowcolor{white}
        Noisy & PSNR (dB) & $26.32$ & $28.95$ & $31.58$ & $34.53$ & $38.70$ \\*
        \rowcolor{white}
         & SSIM & $0.4781$ & $0.6303$ & $0.7665$ & $0.8727$ & $0.9417$ \\
        \rowcolor{black!6}
        Anscombe BMND & PSNR (dB) & $\mathbf{34.01}$ & $33.83$ & $35.55$ & $36.56$ & $40.00$ \\*
        \rowcolor{black!6}
         & SSIM & $0.8953$ & $0.8996$ & $0.9308$ & $0.9388$ & $0.9604$ \\
        \rowcolor{white}
        BMND (rank 1) & PSNR (dB) & $34.00$ & $\mathbf{33.92}$ & $\mathbf{35.67}$ & $\mathbf{36.70}$ & $\mathbf{40.18}$ \\*
        \rowcolor{white}
         & SSIM & $\mathbf{0.8999}$ & $\mathbf{0.9061}$ & $\mathbf{0.9349}$ & $\mathbf{0.9420}$ & $\mathbf{0.9631}$ \\
        \hline
        \rowcolor{white}
        \multicolumn{7}{@{}l}{\textit{Widefield subset}} \\*
        \rowcolor{white}
        Noisy & PSNR (dB) & $25.23$ & $28.19$ & $31.23$ & $34.79$ & $38.74$ \\*
        \rowcolor{white}
         & SSIM & $0.3674$ & $0.5322$ & $0.6957$ & $0.8293$ & $0.9209$ \\
        \rowcolor{black!6}
        Anscombe BMND & PSNR (dB) & $29.20$ & $31.76$ & $34.33$ & $37.47$ & $40.54$ \\*
        \rowcolor{black!6}
         & SSIM & $0.6285$ & $0.7544$ & $0.8471$ & $0.9086$ & $0.9483$ \\
        \rowcolor{white}
        BMND (rank 1) & PSNR (dB) & $\mathbf{33.11}$ & $\mathbf{34.65}$ & $\mathbf{36.34}$ & $\mathbf{38.78}$ & $\mathbf{40.84}$ \\*
        \rowcolor{white}
         & SSIM & $\mathbf{0.8458}$ & $\mathbf{0.8849}$ & $\mathbf{0.9122}$ & $\mathbf{0.9325}$ & $\mathbf{0.9489}$ \\
    \end{longtable}
\endgroup

\Cref{tab:fmd-denoising} reports our results alongside the benchmark values published by \authorcite{zhang2018poisson}.
Modality-specific profiles improve both \ac{psnr} and \ac{ssim} over the global profile at every averaging level, with \ac{psnr} gains between $0.55$ and $0.76\,\mathrm{dB}$.
They also outperform Anscombe \ac{bmnd} at every level, although the \ac{psnr} advantage decreases from $1.00\,\mathrm{dB}$ for single acquisitions to $0.13\,\mathrm{dB}$ when averaging 16 images.

The modality-specific comparisons show that the advantage over Anscombe \ac{bmnd} is concentrated in widefield microscopy.
For single acquisitions, direct Poisson \ac{bmnd} improves \ac{psnr} over the noisy input by $6.82$, $7.68$, and $7.88\,\mathrm{dB}$ for confocal, two-photon, and widefield images, respectively.
The corresponding gains with Anscombe \ac{bmnd} are $6.75$, $7.69$, and $3.97\,\mathrm{dB}$.
Thus, the two approaches give similar improvements for confocal and two-photon images, whereas the Anscombe approach is substantially less effective on the widefield subset.

For widefield images, the advantage of direct Poisson \ac{bmnd} over Anscombe \ac{bmnd} decreases from $3.91\,\mathrm{dB}$ for single acquisitions to $0.30\,\mathrm{dB}$ when averaging 16 images.
The corresponding \ac{ssim} advantage decreases from $0.2173$ to $0.0006$.
This trend is consistent with limitations of the Gaussian approximation after variance stabilization in low-count regions.

\Cref{fig:fmd:fov19} shows single-acquisition denoising examples from held-out FOV 19 for all three microscopy modalities.
The enlarged regions allow comparison of residual noise and preservation of biological structures between Anscombe \ac{bmnd} and direct Poisson \ac{bmnd}.

The modality-specific full-set results exceed the published \ac{vst} + \ac{bm3d} and PURE-LET scores in both metrics at every averaging level, but remain below the best published learned-method results.
These scores are taken from the literature.
We did not reevaluate the corresponding methods.

\begin{figure}[htbp]
    \centering
    \includegraphics[width=\linewidth]{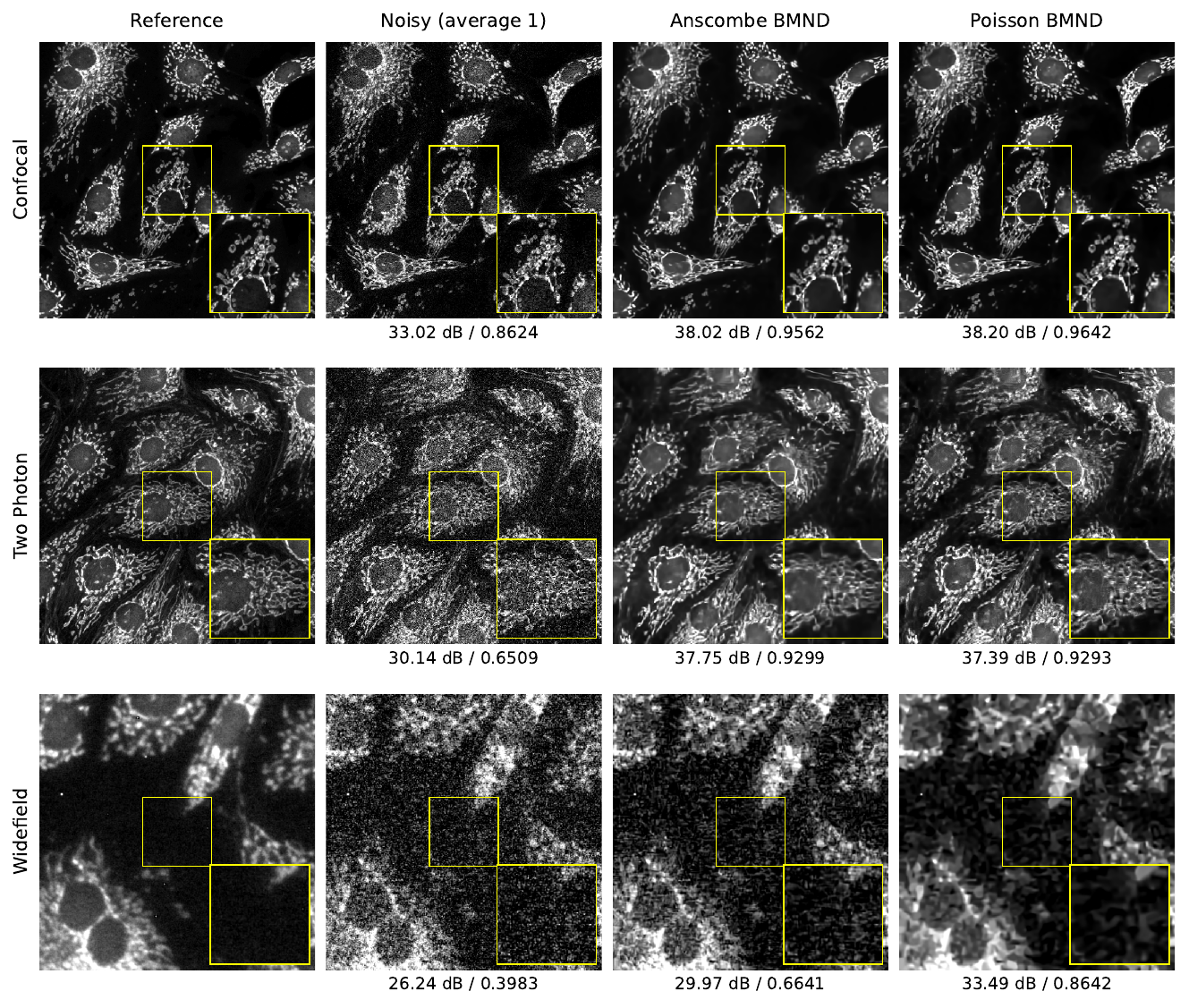}
    \caption{Single-acquisition FMD denoising examples from held-out FOV 19.
    Rows show confocal, two-photon, and widefield microscopy.
    Columns show the reference, noisy input, Anscombe \ac{bmnd}, and direct Poisson \ac{bmnd} with modality-specific profiles.
    Yellow boxes indicate enlarged regions.
    Labels report full-image \ac{psnr}/\ac{ssim}.}
    \label{fig:fmd:fov19}
\end{figure}

\subsection{1-Dimensional ECG Data Denoising}
\label{sec:experiments:ecg}
To demonstrate the applicability of the dimension-independent formulation beyond images and volumes, we evaluate one-dimensional \ac{bmnd} in a controlled experiment with additive white Gaussian noise using channel zero from all records in the MIT--BIH Arrhythmia Database.
The dataset is split into 16 development subjects and 31 held-out test subjects.
For each record, we evaluate five fixed 20-second windows and retain three seconds of unscored context on either side.

We add Gaussian noise at input \ac{snr} levels of $0$, $6$, $12$, and $18\,\mathrm{dB}$ without clipping the resulting observations.
One noise realization is used during development and two during confirmation.
We select the parameters for \ac{bmnd} and each comparator on the development subjects by maximizing the mean \ac{snr} improvement across subjects, with equal weighting across the four input \ac{snr} levels.
\Cref{app:ecg} shows the bounds of this parameter search.
All configurations are then fixed before evaluation on the confirmation subjects.

The comparison includes tuned soft thresholding in the \ac{dwt} domain~\cite{donoho1994ideal,donoho1995denoising,addison2005wavelet}, \ac{nlm}~\cite{tracey2012nlm}, \ac{sscf}~\cite{liu2021cooperative}, as well as fixed and adaptive Savitzky--Golay filtering~\cite{huang2019adaptive}.

The primary reconstruction metric is \ac{snr} improvement, calculated as the difference between the output \ac{snr} of the denoised signal and the measured input \ac{snr} of the noisy observation.
Positive values therefore correspond to a reduction in squared reconstruction error.
Output \ac{snr} is calculated from the energy of the mean-centered reference \ac{ecg} and the squared error between the reconstruction and reference.

Global reconstruction error does not separately quantify distortion of the QRS complex, so we also evaluate QRS-specific waveform and amplitude errors.
The QRS complex reflects ventricular electrical activation and contains the sharp deflections characteristic of each heartbeat.
We measure the \ac{rmse} within $80\,\mathrm{ms}$ of each expert-annotated beat, together with the relative error and signed bias of beat-centered peak-to-peak QRS amplitude.

Metrics are calculated first for each segment and noise realization, then aggregated within subjects and summarized across subjects.
We obtain $95\%$ confidence intervals by bootstrapping subjects, thereby preserving the dependence among segments and repeated noise realizations from the same subject.

\begin{figure}[htbp]
    \centering
    \includegraphics[width=0.7\linewidth]{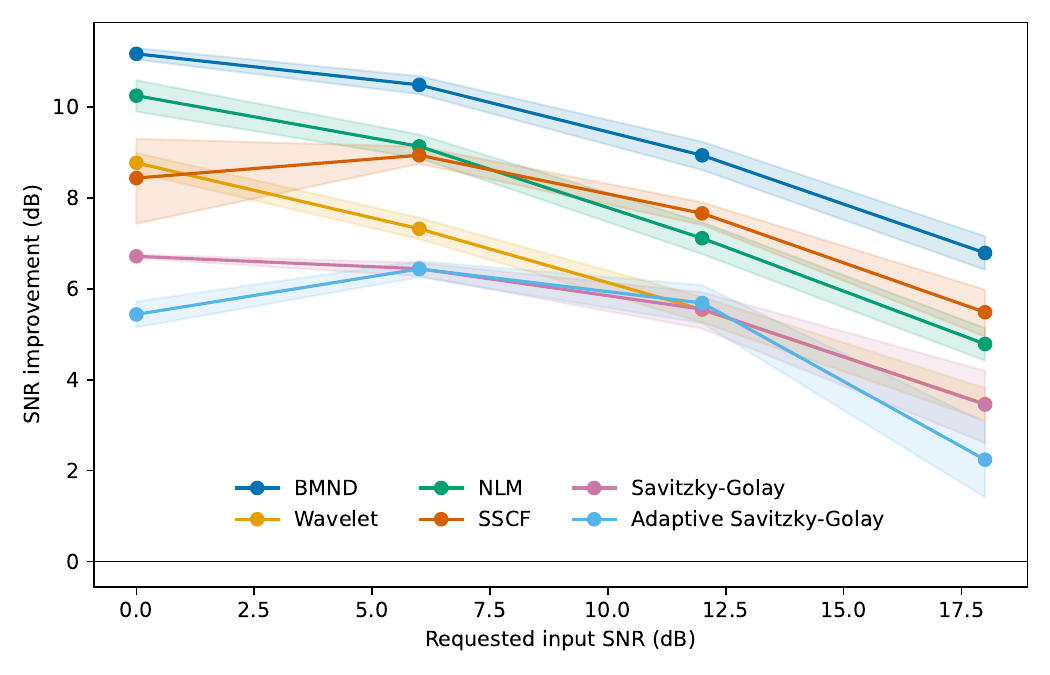}
    \caption{Mean \ac{snr} improvement across confirmation subjects at each input \ac{snr} level. Shaded areas denote 95\% subject-bootstrap confidence intervals.}
    \label{fig:ecg:snr-improvement}
\end{figure}

\begin{figure}[htbp]
    \centering
    \includegraphics[width=0.7\linewidth]{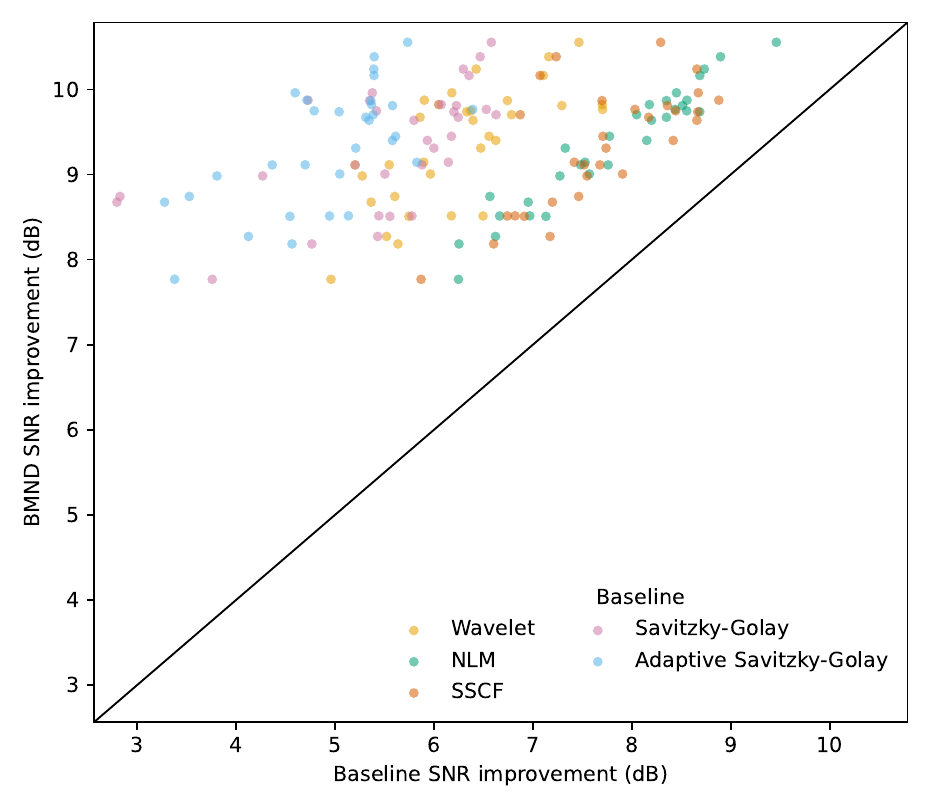}
    \caption{Paired subject-level comparison of \ac{bmnd} with each baseline, averaged across input \ac{snr} levels. Points above the diagonal indicate greater \ac{snr} improvement with \ac{bmnd}. Colors identify the corresponding baseline.}
    \label{fig:ecg:paired-comparison}
\end{figure}

As shown in \Cref{fig:ecg:snr-improvement}, \ac{bmnd} gives the largest mean \ac{snr} improvement at every input noise level.
Its improvement decreases from $11.17\,\mathrm{dB}$ at an input \ac{snr} of $0\,\mathrm{dB}$ to $6.79\,\mathrm{dB}$ at $18\,\mathrm{dB}$.

The subject-level results in \Cref{fig:ecg:paired-comparison} follow the same pattern.
Every point lies above the identity line, so \ac{bmnd} outperforms each baseline for every confirmation subject after averaging over the four input \ac{snr} levels.
The differences are largest for the two Savitzky--Golay methods, while \ac{nlm} and \ac{sscf} are the closest competitors.

\begin{table}[htbp]
    \centering
    \caption{QRS morphology preservation across input \ac{snr} levels, measured by waveform \ac{rmse}, relative amplitude error, and signed relative amplitude bias. Brackets report 95\% subject-bootstrap confidence intervals. Best results are bold and second-best results are underlined.}
    \label{tab:ecg:morphology}
    \vspace{3mm}
    \setlength{\tabcolsep}{3pt}
    \begin{tabular}{@{}lcccc@{}}
\toprule
Method & 0 dB & 6 dB & 12 dB & 18 dB \\
\midrule
\multicolumn{5}{@{}l}{\textit{QRS RMSE $\downarrow$}} \\
BMND & \shortstack{$\mathbf{0.1105}$\\{\scriptsize $[0.0965,0.1254]$}} & \shortstack{$\mathbf{0.0638}$\\{\scriptsize $[0.0557,0.0724]$}} & \shortstack{$\mathbf{0.0388}$\\{\scriptsize $[0.0341,0.0437]$}} & \shortstack{$\mathbf{0.0238}$\\{\scriptsize $[0.0210,0.0267]$}} \\
Wavelet & \shortstack{$0.1901$\\{\scriptsize $[0.1699,0.2113]$}} & \shortstack{$0.1188$\\{\scriptsize $[0.1055,0.1322]$}} & \shortstack{$0.0726$\\{\scriptsize $[0.0644,0.0808]$}} & \shortstack{$0.0430$\\{\scriptsize $[0.0383,0.0478]$}} \\
NLM & \shortstack{$\underline{0.1395}$\\{\scriptsize $[0.1184,0.1627]$}} & \shortstack{$0.0768$\\{\scriptsize $[0.0682,0.0858]$}} & \shortstack{$0.0532$\\{\scriptsize $[0.0475,0.0592]$}} & \shortstack{$0.0350$\\{\scriptsize $[0.0308,0.0393]$}} \\
SSCF & \shortstack{$0.2234$\\{\scriptsize $[0.1505,0.3079]$}} & \shortstack{$\underline{0.0700}$\\{\scriptsize $[0.0615,0.0791]$}} & \shortstack{$\underline{0.0415}$\\{\scriptsize $[0.0366,0.0465]$}} & \shortstack{$\underline{0.0247}$\\{\scriptsize $[0.0220,0.0276]$}} \\
Savitzky--Golay & \shortstack{$0.1583$\\{\scriptsize $[0.1384,0.1787]$}} & \shortstack{$0.0860$\\{\scriptsize $[0.0761,0.0962]$}} & \shortstack{$0.0533$\\{\scriptsize $[0.0477,0.0588]$}} & \shortstack{$0.0401$\\{\scriptsize $[0.0353,0.0453]$}} \\
Adaptive Savitzky--Golay & \shortstack{$0.1935$\\{\scriptsize $[0.1645,0.2241]$}} & \shortstack{$0.0947$\\{\scriptsize $[0.0835,0.1065]$}} & \shortstack{$0.0617$\\{\scriptsize $[0.0568,0.0664]$}} & \shortstack{$0.0534$\\{\scriptsize $[0.0487,0.0580]$}} \\
\addlinespace
\multicolumn{5}{@{}l}{\textit{QRS amplitude error (\%) $\downarrow$}} \\
BMND & \shortstack{$\mathbf{7.97}$\\{\scriptsize $[6.85,9.18]$}} & \shortstack{$\mathbf{4.43}$\\{\scriptsize $[3.99,4.90]$}} & \shortstack{$\mathbf{2.65}$\\{\scriptsize $[2.43,2.89]$}} & \shortstack{$\mathbf{1.61}$\\{\scriptsize $[1.47,1.76]$}} \\
Wavelet & \shortstack{$17.67$\\{\scriptsize $[15.42,20.00]$}} & \shortstack{$11.69$\\{\scriptsize $[10.14,13.32]$}} & \shortstack{$7.46$\\{\scriptsize $[6.51,8.45]$}} & \shortstack{$4.59$\\{\scriptsize $[4.02,5.19]$}} \\
NLM & \shortstack{$15.20$\\{\scriptsize $[11.25,19.63]$}} & \shortstack{$6.09$\\{\scriptsize $[5.13,7.15]$}} & \shortstack{$4.10$\\{\scriptsize $[3.66,4.57]$}} & \shortstack{$2.63$\\{\scriptsize $[2.34,2.95]$}} \\
SSCF & \shortstack{$21.37$\\{\scriptsize $[14.45,29.67]$}} & \shortstack{$\underline{4.95}$\\{\scriptsize $[4.37,5.59]$}} & \shortstack{$\underline{2.89}$\\{\scriptsize $[2.65,3.15]$}} & \shortstack{$\underline{1.74}$\\{\scriptsize $[1.58,1.90]$}} \\
Savitzky--Golay & \shortstack{$\underline{12.40}$\\{\scriptsize $[10.69,14.17]$}} & \shortstack{$6.91$\\{\scriptsize $[5.76,8.33]$}} & \shortstack{$5.62$\\{\scriptsize $[4.21,7.34]$}} & \shortstack{$5.37$\\{\scriptsize $[3.83,7.23]$}} \\
Adaptive Savitzky--Golay & \shortstack{$19.15$\\{\scriptsize $[15.79,22.71]$}} & \shortstack{$7.29$\\{\scriptsize $[6.50,8.08]$}} & \shortstack{$5.33$\\{\scriptsize $[4.43,6.38]$}} & \shortstack{$5.68$\\{\scriptsize $[4.58,6.86]$}} \\
\addlinespace
\multicolumn{5}{@{}l}{\textit{QRS amplitude bias (\%) $\to 0$}} \\
BMND & \shortstack{$\mathbf{-1.04}$\\{\scriptsize $[-2.47,0.26]$}} & \shortstack{$\mathbf{-0.01}$\\{\scriptsize $[-0.41,0.38]$}} & \shortstack{$\mathbf{0.06}$\\{\scriptsize $[-0.11,0.23]$}} & \shortstack{$\mathbf{0.16}$\\{\scriptsize $[0.05,0.27]$}} \\
Wavelet & \shortstack{$-15.46$\\{\scriptsize $[-18.47,-12.53]$}} & \shortstack{$-10.85$\\{\scriptsize $[-12.71,-9.05]$}} & \shortstack{$-6.93$\\{\scriptsize $[-8.05,-5.83]$}} & \shortstack{$-4.27$\\{\scriptsize $[-4.95,-3.61]$}} \\
NLM & \shortstack{$-12.00$\\{\scriptsize $[-16.82,-7.67]$}} & \shortstack{$\underline{0.47}$\\{\scriptsize $[-0.63,1.47]$}} & \shortstack{$2.80$\\{\scriptsize $[2.37,3.26]$}} & \shortstack{$1.95$\\{\scriptsize $[1.60,2.30]$}} \\
SSCF & \shortstack{$-19.34$\\{\scriptsize $[-28.08,-12.14]$}} & \shortstack{$-1.17$\\{\scriptsize $[-1.81,-0.57]$}} & \shortstack{$\underline{0.16}$\\{\scriptsize $[-0.07,0.39]$}} & \shortstack{$\underline{-0.18}$\\{\scriptsize $[-0.43,0.07]$}} \\
Savitzky--Golay & \shortstack{$\underline{6.44}$\\{\scriptsize $[3.07,9.68]$}} & \shortstack{$-1.19$\\{\scriptsize $[-3.50,0.96]$}} & \shortstack{$-3.99$\\{\scriptsize $[-6.02,-2.21]$}} & \shortstack{$-4.87$\\{\scriptsize $[-6.83,-3.20]$}} \\
Adaptive Savitzky--Golay & \shortstack{$15.71$\\{\scriptsize $[11.20,20.35]$}} & \shortstack{$1.04$\\{\scriptsize $[-0.98,3.08]$}} & \shortstack{$-3.77$\\{\scriptsize $[-5.21,-2.47]$}} & \shortstack{$-5.41$\\{\scriptsize $[-6.68,-4.20]$}} \\
\bottomrule
\end{tabular}

\end{table}

\begin{figure}[htbp]
    \centering
    \includegraphics[width=0.8\linewidth]{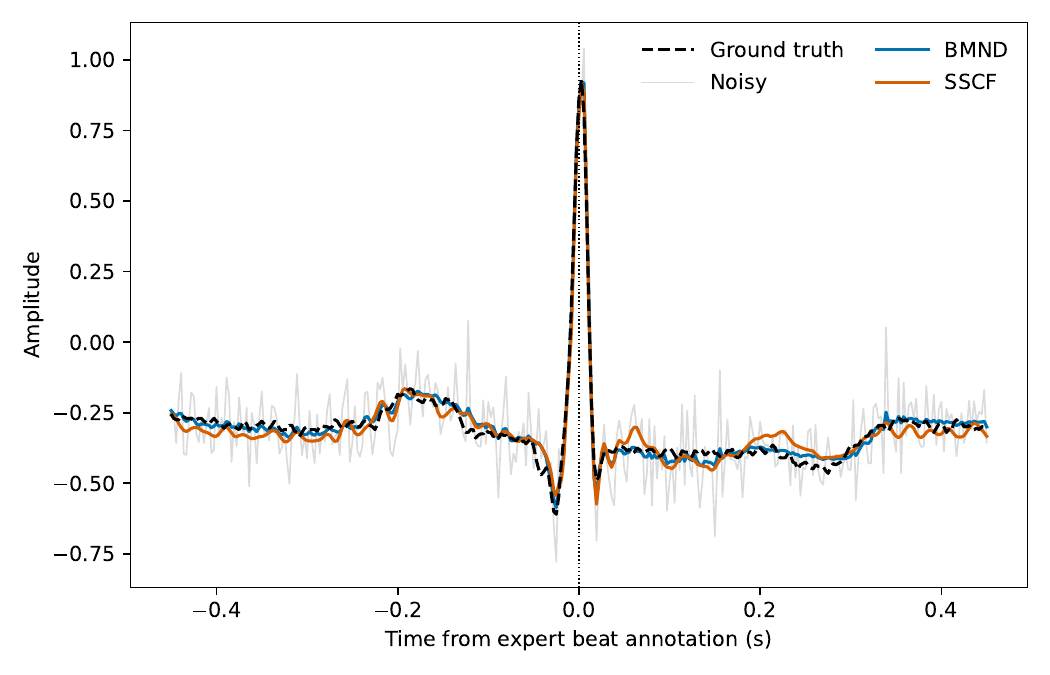}
    \caption{Example beat-centered \ac{ecg} waveforms comparing the reference, noisy observation, \ac{bmnd}, and \ac{sscf}.
    The vertical dotted line marks the expert beat annotation.}
    \label{fig:ecg:qrs-morphology}
\end{figure}

\Cref{tab:ecg:morphology} shows that the gain in \ac{snr} is accompanied by lower errors in the evaluated QRS morphology measures.
\Ac{bmnd} has the lowest QRS waveform \ac{rmse} and relative amplitude error at every input \ac{snr} level.
Its QRS \ac{rmse} decreases from $0.1105$ to $0.0238$, while its relative amplitude error decreases from $7.97\%$ to $1.61\%$.
The signed amplitude bias remains between $-1.04\%$ and $0.16\%$.
In contrast, the wavelet method consistently underestimates QRS amplitude, and both Savitzky--Golay methods introduce larger biases at the higher input \ac{snr} levels.
\Cref{fig:ecg:qrs-morphology} provides a qualitative comparison with \ac{sscf} for an example beat.
Both methods preserve the dominant QRS peak, while \ac{bmnd} shows smaller deviations from the reference immediately after the peak in this example.
These results indicate that \ac{bmnd} reduces noise with less distortion of QRS amplitude and waveform shape than the evaluated baselines.

\section{Discussion and Conclusion}
\label{sec:conclusion}
In this work, we presented \ac{bmnd}, a dimension-independent formulation of block-matching collaborative filtering for Gaussian and Poisson noise.
The formulation incorporates noise-aware patch matching, collaborative filtering, and aggregation, with optional mass conservation.
We evaluated the method on one-dimensional physiological signals, two-dimensional images, and three-dimensional volumes.

The ablation study shows that the benefits of noise-aware modeling depend on the algorithmic component.
Adapting patch matching and Wiener gains to the observation statistics improves marginal mean reconstruction quality across all evaluated dataset groups.
Noise-aware aggregation gives smaller improvements, with positive marginal mean effects of Wiener-stage variance and risk weighting on natural images, fluorescence microscopy acquisitions, and anatomical phantom data.
Bias-aware risk weighting gives no clear additional benefit over variance weighting, and exact covariance treatment does not consistently improve reconstruction quality across datasets.

The Shepp--Logan experiment shows that mass conservation can reduce denoising-induced intensity loss at low Poisson counts.
Preserving the observed total does not, however, guarantee unbiased regional intensities or improved reconstructed-image \ac{psnr}.
Conservation only in the hard-thresholding stage gives the largest low-count reconstruction \ac{psnr} gains in this experiment, whereas two-stage conservation combines preservation of the observed total with improved reconstruction quality.

On \ac{fmd}, direct Poisson \ac{bmnd} uses an intensity shift to match the calibrated affine mean--variance relationship, although this does not ensure that the shifted observations follow an exact Poisson distribution.
Its advantage over Anscombe \ac{bmnd} is largest for single-acquisition widefield images and decreases with acquisition averaging.
This trend is consistent with a more accurate Gaussian approximation after variance stabilization at higher effective counts.
At every evaluated averaging level, modality-specific \ac{bmnd} gives higher aggregate \ac{psnr} and \ac{ssim} than the evaluated Anscombe baseline and the published \ac{vst} + \ac{bm3d} and PURE-LET results, but remains below the best published learning-method scores.
The improvement over the selected global configuration also shows that a common formulation does not remove the need for application-specific parameter selection.

The \ac{ecg} evaluation demonstrates the applicability of the formulation beyond imaging.
Under added Gaussian noise, one-dimensional \ac{bmnd} gives larger mean \ac{snr} improvements and lower QRS waveform and amplitude errors than the evaluated baselines.
This experiment assesses reconstruction and QRS morphology under controlled noise rather than performance on physiological recording artifacts.

Future work will investigate \ac{bmnd} on higher-dimensional positron emission tomography projection data, including sinograms with time-of-flight, dynamic, and gating axes.
Such acquisitions can yield six-dimensional or higher-dimensional data representations, providing a natural application of the dimension-independent formulation.
Evaluation will assess the benefits of joint processing across these axes while examining sinogram consistency and the preservation of quantitative uptake and temporal fidelity after reconstruction.
An explicit mixed Poisson–Gaussian observation model would extend the framework beyond the present scaled-Poisson formulation and its moment-matching approximation for fluorescence data.
Finally, we will investigate approximate global block matching under Gaussian and Poisson noise, including search representations adapted to the geometry of the matching statistics.  
This evaluation will assess computational cost and denoising quality, particularly the trade-off between access to additional self-similar patches and increased susceptibility to noise-driven matches.

\section*{Acknowledgments}
This work was supported by grants from Deutsche Forschungsgemeinschaft (CRC1450 project ID 431460824).

OpenAI's GPT-5.2, GPT-5.4, GPT-5.5, GPT-5.6, and GPT-6 were used during all development steps.
OpenAI's GPT-6 was used for language editing and revisions to the \LaTeX{} source.

\bibliographystyle{plainnat}
\bibliography{bibliography}

\clearpage
\appendix

\crefname{appendix}{appendix}{appendices}
\Crefname{appendix}{Appendix}{Appendices}
\crefalias{section}{appendix}

\section{Standard Configuration}
\label{app:standard-configuration}
\Cref{tab:standard-gaussian-profiles} shows the parameters used for experiments if not stated otherwise.
These are taken from the modern \Ac{bm3d}/\Ac{bm4d} implementations.

\begingroup
\footnotesize
\begin{longtable}{@{}>{\raggedright\arraybackslash}p{0.28\textwidth}>{\raggedright\arraybackslash}p{0.10\textwidth}>{\raggedright\arraybackslash}p{0.23\textwidth}>{\raggedright\arraybackslash}p{0.23\textwidth}@{}}
    \caption{Standard Gaussian BM3D and BM4D profile values used by BMND.}
    \label{tab:standard-gaussian-profiles} \\
        \toprule
        Parameter & Stage & BM3D & BM4D \\
        \midrule
        \endfirsthead

        \toprule
        Parameter & Stage & BM3D & BM4D \\
        \midrule
        \endhead

        \midrule
        \multicolumn{4}{r}{Continued on next page} \\
        \endfoot

        \bottomrule
        \endlastfoot

        \multicolumn{4}{@{}l}{\textit{General and scheduling}} \\
        Reference schedule & Both & \texttt{generated} & \texttt{generated} \\
        Shift density & Both & 2.0 & 2.0 \\
        Schedule density & Both & 2 & 2 \\
        \addlinespace
        \multicolumn{4}{@{}l}{\textit{Covariance and matching}} \\
        Local-variance domain $n_f$ & Both & $32 \times 32$ & $16 \times 16 \times 16$ \\
        Exact covariance planes $k$ & Both & 4 & 4 \\
        Matching distance & Both & \texttt{ssd} & \texttt{ssd} \\
        HT SSD bias factor $\gamma$ & HT & 3.0 & 3.0 \\
        \addlinespace
        \multicolumn{4}{@{}l}{\textit{Aggregation}} \\
        Weight model / domain / scope & HT & \texttt{variance} / \texttt{coefficient} / \texttt{patch} & \texttt{variance} / \texttt{coefficient} / \texttt{patch} \\
        Weight model / domain / scope & Wiener & \texttt{variance} / \texttt{coefficient} / \texttt{patch} & \texttt{variance} / \texttt{coefficient} / \texttt{patch} \\
        \addlinespace
        \multicolumn{4}{@{}l}{\textit{Patch grouping}} \\
        Block size & HT & $8 \times 8$ & $4 \times 4 \times 4$ \\
        Block size & Wiener & $8 \times 8$ & $5 \times 5 \times 5$ \\
        Reference step & HT & $3 \times 3$ & $3 \times 3 \times 3$ \\
        Reference step & Wiener & $3 \times 3$ & $3 \times 3 \times 3$ \\
        Search window & HT & $19 \times 19$ & $7 \times 7 \times 7$ \\
        Search window & Wiener & $19 \times 19$ & $7 \times 7 \times 7$ \\
        Minimum group size & HT & 2 & 2 \\
        Minimum group size & Wiener & 2 & 2 \\
        Maximum group size & HT & 16 & 16 \\
        Maximum group size & Wiener & 32 & 32 \\
        Matching threshold & HT & 2.9527 & 2.9527 \\
        Matching threshold & Wiener & 0.3937 & 0.7689 \\
        \addlinespace
        \multicolumn{4}{@{}l}{\textit{Collaborative filtering}} \\
        Thresholding & HT & hard & hard \\
        Threshold multiplier $\lambda$ & HT & 3.0 & 3.0 \\
        Variance scale & Wiener & 0.4 & 0.4 \\
        Kaiser window $\beta$ & HT & 2.0 & 2.0 \\
        Kaiser window $\beta$ & Wiener & 2.0 & 2.0 \\
        Transforms & HT & Haar (group), bior1.5 (spatial) & Haar (group), bior1.5 (spatial) \\
        Transforms & Wiener & Haar (group), DCT (spatial) & Haar (group), DCT (spatial) \\
\end{longtable}
\endgroup

\section{Ablation Study}
\label{app:ablation}

\Cref{tab:ablation-profile-space} shows the parameters used for the ablation experiment.
All profile fields not listed here retain the dimension-specific standard configuration in \Cref{app:standard-configuration}.

\begingroup
\small
\setlength{\tabcolsep}{4pt}
\renewcommand{\arraystretch}{1.12}
\begin{longtable}{@{}>{\raggedright\arraybackslash}p{0.30\textwidth}>{\raggedright\arraybackslash}p{0.55\textwidth}r@{}}
    \caption{Profile space evaluated in the Cartesian ablation. The level count is the number of independently crossed options.}
    \label{tab:ablation-profile-space} \\

    \toprule
    Profile parameter & Candidate values or fixed setting & Levels \\
    \midrule
    \endfirsthead

    \toprule
    Profile parameter & Candidate values or fixed setting & Levels \\
    \midrule
    \endhead

    \midrule
    \multicolumn{3}{r}{Continued on next page} \\
    \endfoot

    \bottomrule
    \endlastfoot

    \multicolumn{3}{@{}l}{\textit{Varied Cartesian factors}} \\
    HT matching distance / policy & \texttt{poisson\_deviance} / \texttt{fixed}, \texttt{poisson\_deviance} / \texttt{reference\_finite\_count}, \texttt{poisson\_deviance} / \texttt{candidate\_standardized}, \texttt{pearson} / \texttt{fixed}, \texttt{pearson} / \texttt{reference\_finite\_count}, \texttt{pearson} / \texttt{candidate\_standardized}, \texttt{anscombe\_ssd} / \texttt{fixed}, \texttt{anscombe\_ssd} / \texttt{reference\_finite\_count}, \texttt{anscombe\_ssd} / \texttt{candidate\_standardized}, \texttt{ssd} / \texttt{fixed} & 10 \\
    Exact covariance planes $k$ & $\{ 0, 1, 4, 32 \}$ & 4 \\
    Wiener gain & $\{\,$\texttt{classic}, \texttt{noise-floor}, \texttt{variance-scaled}$\,\}$ & 3 \\
    HT aggregation-weight model & $\{\,$\texttt{classic}, \texttt{variance}$\,\}$ & 2 \\
    HT aggregation-weight domain & $\{\,$\texttt{coefficient}, \texttt{windowed\_synthesis}$\,\}$ & 2 \\
    HT aggregation-weight scope & $\{\,$\texttt{group}, \texttt{patch}$\,\}$ & 2 \\
    Wiener aggregation-weight model & $\{\,$\texttt{classic}, \texttt{variance}, \texttt{risk}$\,\}$ & 3 \\
    Wiener aggregation-weight domain & $\{\,$\texttt{coefficient}, \texttt{windowed\_synthesis}$\,\}$ & 2 \\
    Wiener aggregation-weight scope & $\{\,$\texttt{group}, \texttt{patch}$\,\}$ & 2 \\
    Group-mass conservation stage & $\{\,$\texttt{none}, \texttt{ht}, \texttt{wiener}, \texttt{both}$\,\}$ & 4 \\
    HT threshold rule & $\{\,$\texttt{hard}, \texttt{soft}$\,\}$ & 2 \\
    HT threshold multiplier $\lambda$ & $\{ 1.5, 3, 4.5 \}$ & 3 \\
    \addlinespace
    \multicolumn{2}{r}{\textbf{Total Cartesian profiles}} & \textbf{276,480} \\
    \addlinespace
    \multicolumn{3}{@{}l}{\textit{Fixed and conditional settings}} \\
    Noise model & \texttt{poisson} & 1 \\
    Reference schedule & \texttt{generated} with shift density 2 and schedule density 2 & 1 \\
    Wiener matching & \texttt{pilot} source / \texttt{fixed} policy & 1 \\
    Wiener variance scale & 0.4 & 1 \\
    Global mass conservation & \texttt{none} & 1 \\
    Finite-count structural allowance & $\beta_s = 0.1$ and intensity power $\kappa_s = 1$ when the HT policy is \texttt{reference\_finite\_count} & 1 \\
    HT SSD bias factor $\gamma$ & $3$ when the matching distance is \texttt{ssd} & 1 \\
\end{longtable}
\endgroup

\begin{figure}[htbp]
    \centering
    \includegraphics[width=0.8\linewidth]{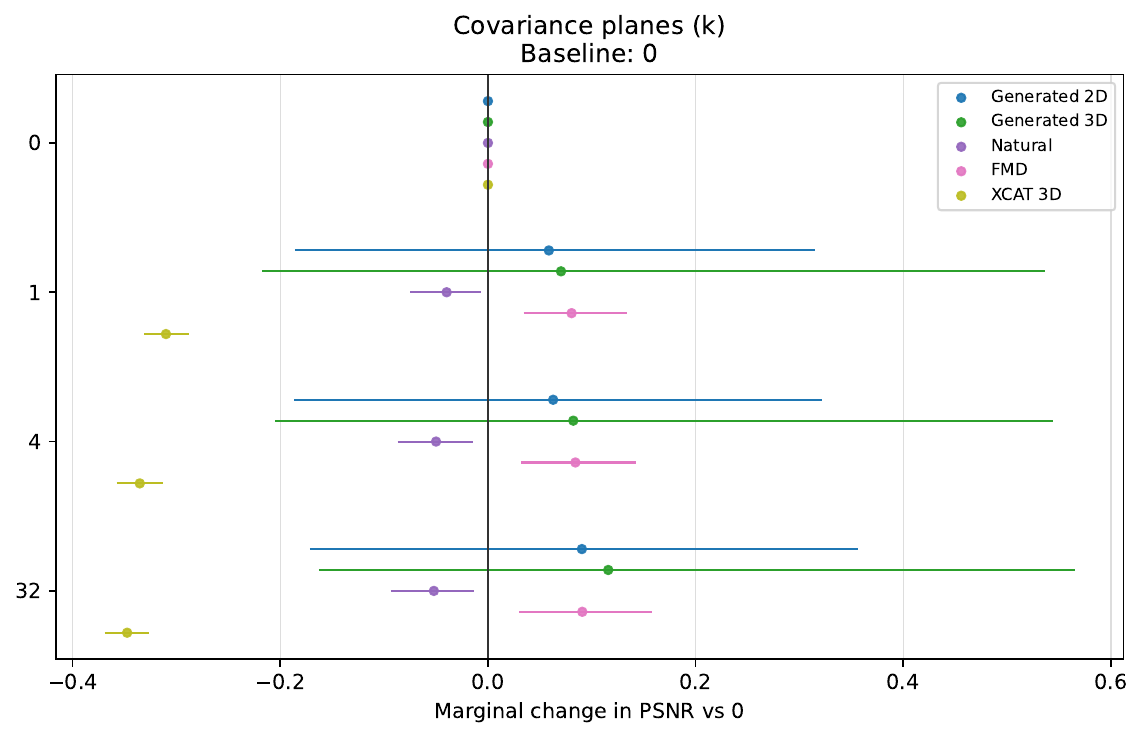}
    \caption{Marginal effects of the number of covariance planes $k$ on \protect\ac{psnr} relative to $k = 0$.
    Error bars denote 95\% source-unit bootstrap intervals.}
    \label{fig:ablation:covariance-planes}
\end{figure}

\begin{figure}[htbp]
    \centering
    \includegraphics[width=0.8\linewidth]{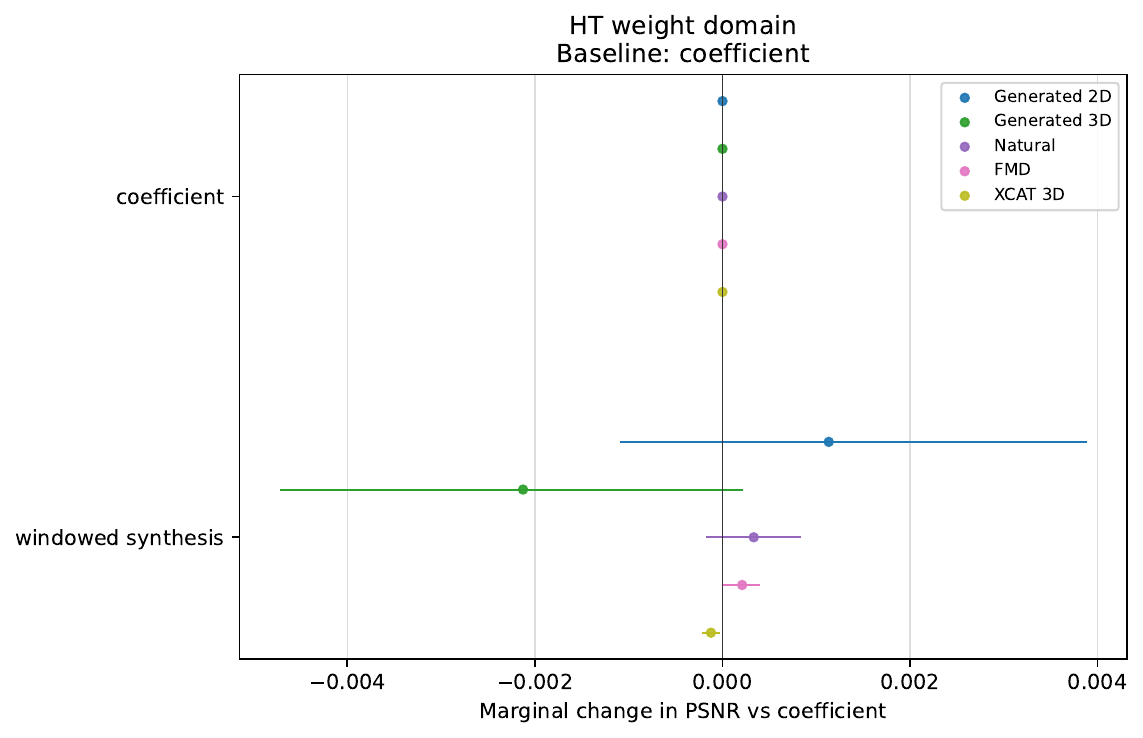}
    \caption{Marginal effects of the hard-thresholding aggregation-weight domain on \protect\ac{psnr} relative to coefficient-domain weighting.
    Error bars denote 95\% source-unit bootstrap intervals.}
    \label{fig:ablation:ht-weight-domain}
\end{figure}

\begin{figure}[htbp]
    \centering
    \includegraphics[width=0.8\linewidth]{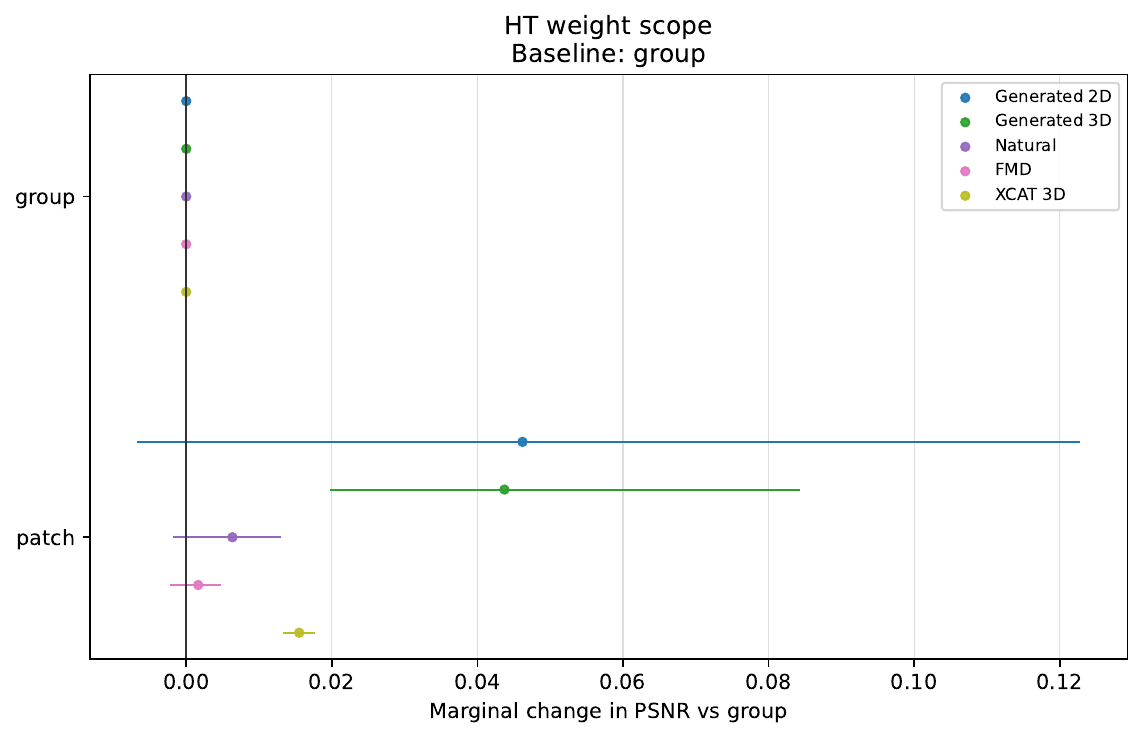}
    \caption{Marginal effects of the hard-thresholding aggregation-weight scope on \protect\ac{psnr} relative to group-level weighting.
    Error bars denote 95\% source-unit bootstrap intervals.}
    \label{fig:ablation:ht-weight-scope}
\end{figure}

\begin{figure}[htbp]
    \centering
    \includegraphics[width=0.8\linewidth]{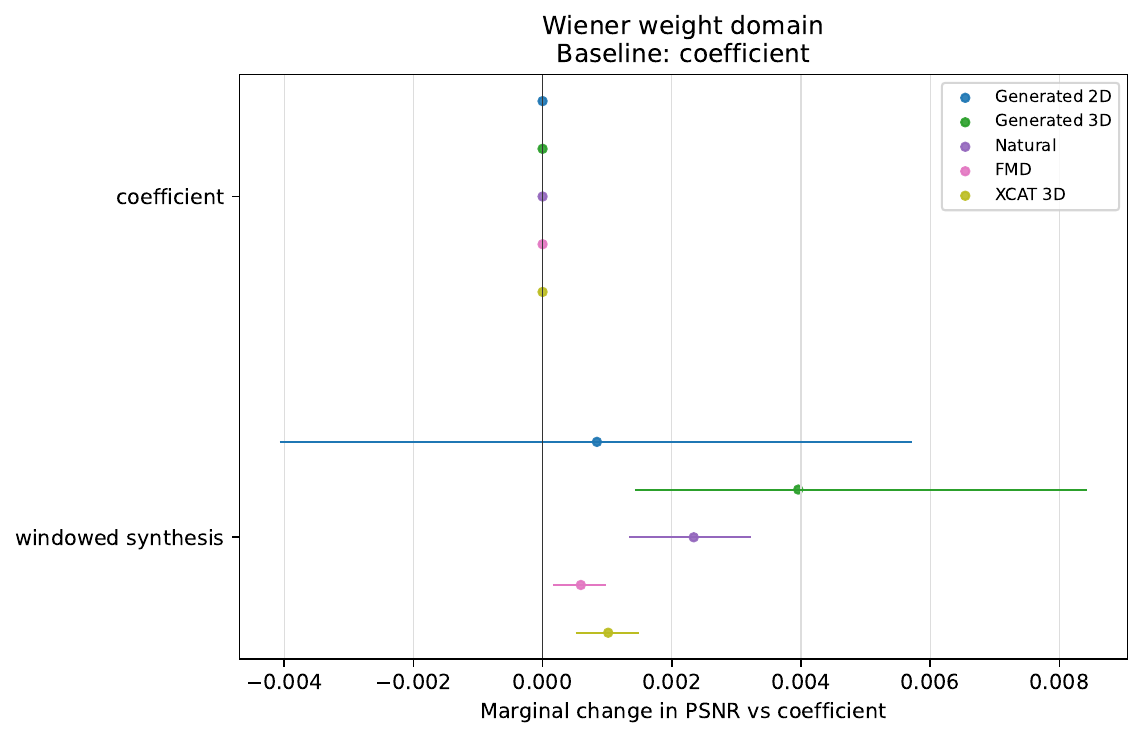}
    \caption{Marginal effects of the Wiener-stage aggregation-weight domain on \protect\ac{psnr} relative to coefficient-domain weighting.
    Error bars denote 95\% source-unit bootstrap intervals.}
    \label{fig:ablation:wiener-weight-domain}
\end{figure}

\begin{figure}[htbp]
    \centering
    \includegraphics[width=0.8\linewidth]{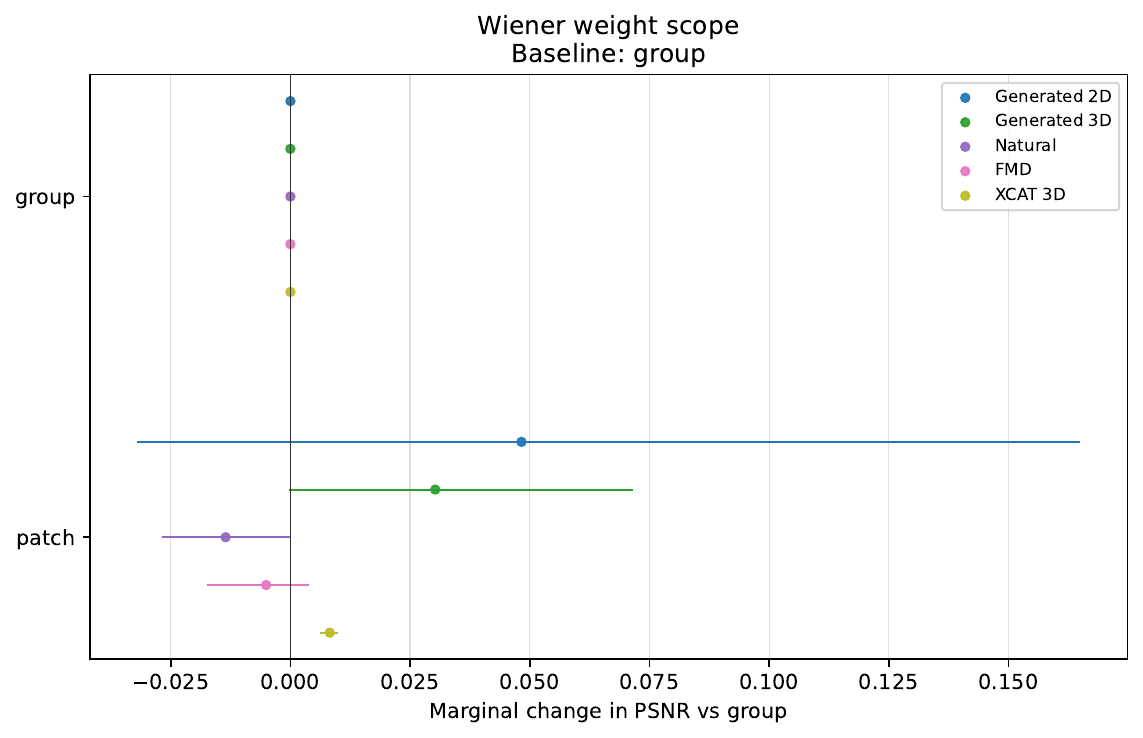}
    \caption{Marginal effects of the Wiener-stage aggregation-weight scope on \protect\ac{psnr} relative to group-level weighting.
    Error bars denote 95\% source-unit bootstrap intervals.}
    \label{fig:ablation:wiener-weight-scope}
\end{figure}

\begin{figure}[htbp]
    \centering
    \includegraphics[width=0.8\linewidth]{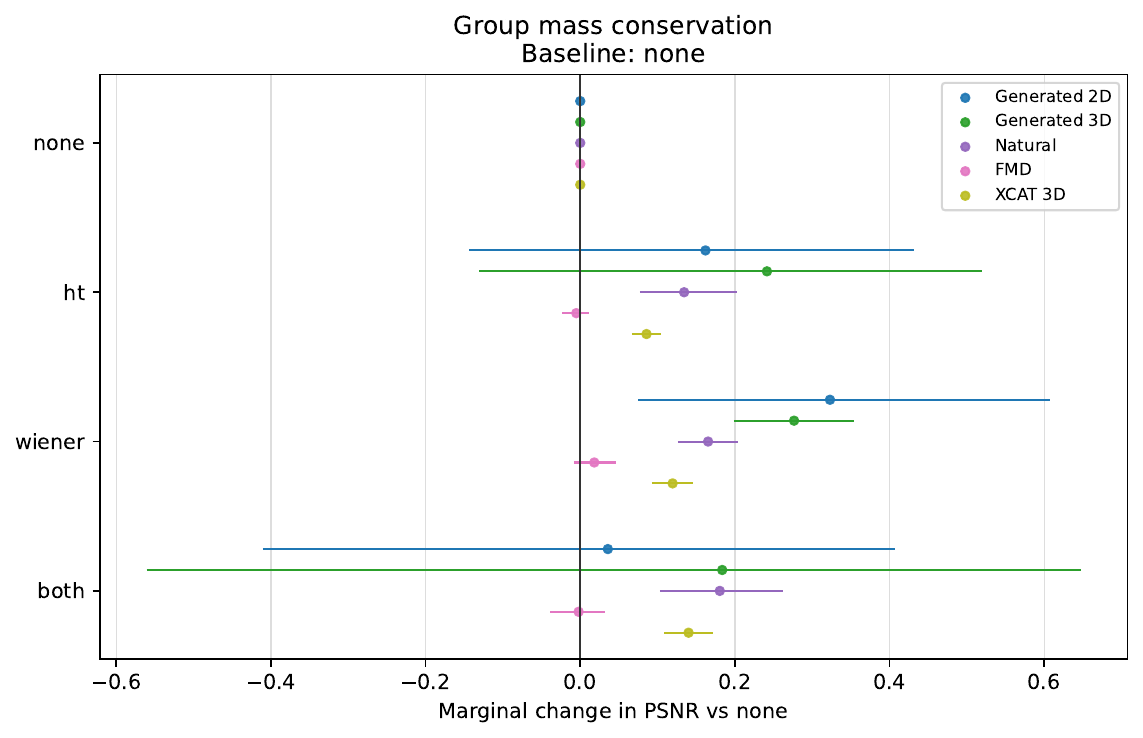}
    \caption{Marginal effects of group mass conservation on \protect\ac{psnr} relative to denoising without mass conservation.
    Conservation is applied in the hard-thresholding stage, the Wiener stage, or both stages.
    Error bars denote 95\% source-unit bootstrap intervals.}
    \label{fig:ablation:group-mass}
\end{figure}
\section{ECG Experiment}
\label{app:ecg}

\begingroup
\small
\setlength{\tabcolsep}{4pt}
\renewcommand{\arraystretch}{1.12}

\begin{longtable}{@{}>{\raggedright\arraybackslash}p{0.4\textwidth}>{\raggedright\arraybackslash}p{0.55\textwidth}r@{}}
    \caption{Parameter spaces evaluated in the controlled ECG experiment. Candidate counts include only combinations varied during development, fixed settings are reported for reproducibility.}
    \label{tab:ecg-parameter-spaces}
    \\

    \toprule
    Method and parameter
        & Candidate values or fixed setting
        & \makebox[0pt][r]{Candidates} \\
    \midrule
    \endfirsthead

    \toprule
    Method and parameter
        & Candidate values or fixed setting
        & \makebox[0pt][r]{Candidates} \\
    \midrule
    \endhead

    \midrule
    \multicolumn{3}{r}{Continued on next page} \\
    \endfoot

    \bottomrule
    \endlastfoot

    \textbf{1D BMND} & & \textbf{240} \\
    \quad Patch length $L$
        & $\{ 8, 16, 32, 64, 128 \}$ samples
        & \\
    \quad Reference step
        & $L / 4$ samples
        & \\
    \quad One-sided search radius
        & $\{ 1, 2, 3, 4 \}\,\mathrm{s}$
        & \\
    \quad Maximum group size
        & $\{ 8, 16, 32, 64 \}$
        & \\
    \quad Matching-threshold scale $c_m$
        & $\{ 0.5, 1, 2 \}$
        & \\
    \quad Hard-thresholding match threshold
        & $2.9527 ( L / 64 ) c_m$
        & \\
    \quad Wiener match threshold
        & $0.3937 ( L / 64 ) c_m$
        & \\
    \quad Transforms
        & Haar group and biorthogonal-1.5 local transforms in the
          hard-thresholding stage; Haar group and DCT local transforms in the
          Wiener stage
        & \\

    \addlinespace
    \textbf{DWT shrinkage} & & \textbf{84} \\
    \quad Wavelet
        & $\{ \mathrm{db4}, \mathrm{sym4}, \mathrm{coif3} \}$
        & \\
    \quad Decomposition level
        & $\{ 3, 4, 5, 6 \}$
        & \\
    \quad Threshold multiplier $c_{\lambda}$
        & $\{ 0.1, 0.25, 0.5, 0.75, 1, 1.25, 1.5 \}$
        & \\
    \quad Threshold rule
        & Soft thresholding with
          $\lambda = c_{\lambda} \sigma \sqrt{2 \log n}$ and periodized
          boundaries
        & \\

    \addlinespace
    \textbf{Nonlocal means} & & \textbf{120} \\
    \quad Patch length
        & $\{ 8, 16, 32, 64, 128 \}$ samples
        & \\
    \quad One-sided search radius
        & $\{ 1, 2, 3, 4 \}\,\mathrm{s}$
        & \\
    \quad Bandwidth multiplier $c_h$
        & $\{ 0.2, 0.4, 0.6, 0.8, 1, 1.2 \}$
        & \\
    \quad Bandwidth and distance correction
        & $h = c_h \sigma$; patch distances were corrected by subtracting
          $2 \sigma^2$
        & \\

    \addlinespace
    \textbf{Similar-segment cooperative filtering} & & \textbf{3} \\
    \quad Segment length
        & $17$ samples
        & \\
    \quad Matching-threshold scale $c_{\tau}$
        & $\{ 0.5, 1, 2 \}$
        & \\
    \quad Matching threshold
        & $\tau = c_{\tau} \max \left(
          1325 v^3 - 252.986 v^2 + 20.359 v + 0.0695,
          \epsilon \right)$
        & \\
    \quad Fixed Savitzky--Golay settings
        & Seven-sample prefilter; polynomial orders $2$ and $4$ for non-peak and peak groups, respectively
        & \\
    \quad Fixed grouping settings
        & Robust peak detection, distance-sorted group members, first-order
          fitting across similar segments, and uniform overlap averaging
        & \\

    \addlinespace
    \textbf{Savitzky--Golay} & & \textbf{27} \\
    \quad Window length
        & $\{ 5, 7, 9, 11, 15, 21, 31, 41, 61 \}$ samples
        & \\
    \quad Polynomial order
        & $\{ 2, 3, 4 \}$
        & \\

    \addlinespace
    \textbf{Adaptive Savitzky--Golay} & & \textbf{84} \\
    \quad Window-length and maximum-order pairs
        & $\{ (9, 7), (11, 7), (11, 9), (15, 9), (15, 13),
          (21, 9), (21, 13) \}$
        & \\
    \quad Maximum curvature-search length
        & $\{ 2, 5, 10 \}$ samples
        & \\
    \quad DSS angle threshold
        & $\{ 0.01, 0.025, 0.05, 0.1 \}\,\mathrm{rad}$
        & \\
    \quad Order assignment
        & Curvature mapped onto integer polynomial orders from $1$ to the
          selected maximum
        & \\

    \addlinespace
    \multicolumn{2}{r}{\textbf{Total development candidates}}
        & \textbf{558} \\
\end{longtable}

\noindent
Here, $n$ is the number of samples in the denoised segment,
$\sigma$ is the known standard deviation of the added Gaussian noise, $v$ is the \ac{sscf} prefilter-residual variance, and $\epsilon$ is a small positive numerical constant.

\endgroup

\end{document}